\documentclass[letter]{article}

\usepackage{iclr2022_conference,times}

\usepackage[utf8]{inputenc} 
\usepackage[T1]{fontenc}    
\usepackage{hyperref}       
\usepackage{url}            
\usepackage{booktabs}       
\usepackage{amsfonts}       
\usepackage{nicefrac}       
\usepackage{microtype}      
\usepackage{xcolor}         

\usepackage{graphicx}
\usepackage{tikz-cd}
\usepackage{booktabs}
\usepackage{colortbl}
\usepackage{wrapfig}
\usepackage{array}
\usepackage{placeins}
\usepackage{multirow}
\usepackage{diagbox}
\usepackage{caption}
\usepackage{subcaption}
\usepackage{amssymb}
\usepackage{amsthm}
\usepackage{mathtools}
\usepackage{csquotes}
\usepackage[capitalise]{cleveref}
\usepackage[shortlabels]{enumitem}
\usepackage{my_macros}

\definecolor{tbl_highlight}{gray}{0.9}

\graphicspath{{./fig/}}

\title{Steerable Partial Differential Operators for Equivariant Neural Networks}

\author{%
  Erik Jenner\thanks{Work done during an internship at QUVA Lab}\\
  University of Amsterdam\\
  \texttt{erik@ejenner.com} \\
  \And
  Maurice Weiler \\
  University of Amsterdam\\
  \texttt{m.weiler.ml@gmail.com} \\
}

\iclrfinalcopy 
\begin{document}

\maketitle

\begin{abstract}
Recent work in equivariant deep learning bears strong similarities to physics. Fields over a base space are fundamental entities in both subjects, as are equivariant maps between these fields. In deep learning, however, these maps are usually defined by convolutions with a kernel, whereas they are partial differential operators (PDOs) in physics. Developing the theory of equivariant PDOs in the context of deep learning could bring these subjects even closer together and lead to a stronger flow of ideas. In this work, we derive a \(G\)-steerability constraint that completely characterizes when a PDO between feature vector fields is equivariant, for arbitrary symmetry groups \(G\). We then fully solve this constraint for several important groups. We use our solutions as equivariant drop-in replacements for convolutional layers and benchmark them in that role. Finally, we develop a framework for equivariant maps based on Schwartz distributions that unifies classical convolutions and differential operators and gives insight about the relation between the two.
\end{abstract}

\begin{figure}[h]
  \centering
  \includegraphics[width=.9\linewidth]{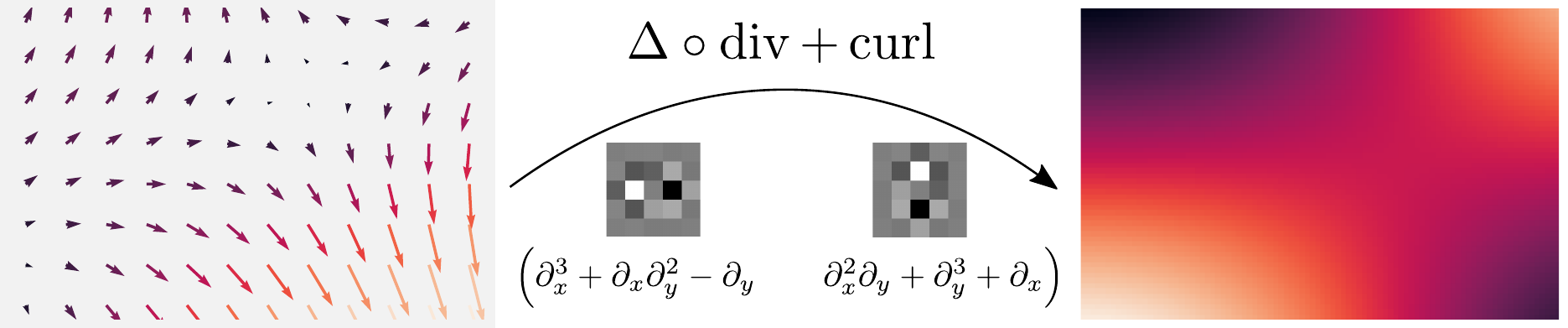}
  \caption{\small A vector field (left) can
    be mapped to a scalar field (right) by applying certain partial differential
    operators (PDOs), such as the Laplacian of the divergence and the 2D curl. Such a PDO from a 2D
    vector to a scalar field can be represented as a \(2 \times 1\) matrix, where
    each of the two entries is a one-dimensional PDO that acts on one of the two components of
    the vector field. Similarly, matrices of PDOs with different dimensions map
    between other types of fields. Our goal is to find \emph{all} PDOs for which this map becomes
    equivariant, for arbitrary types of fields. For the implementation, we will later
    discretize PDOs as stencils (middle).}
  \label{fig:feature_field_map}
\end{figure}

\section{Introduction}
In many machine learning tasks, the data exhibits certain symmetries, such as
translation- and sometimes rotation-invariance in image classification. To exploit those symmetries,
equivariant neural networks have been widely studied and successfully applied in
the past years, beginning with Group convolutional neural networks~\citep{cohen2016a,weiler2018}.
A significant generalization of Group convolutional networks is given by
steerable CNNs~\citep{cohen2016,weiler2018a,weiler2019}, which
unify many different pre-existing equivariant models~\citep{weiler2019}.
They do this by representing features as fields of feature vectors over a base space,
such as \(\R^{2}\) in the case of two-dimensional images. Layers are then linear
equivariant maps between these fields.

This is very reminiscent of physics. There, fields are used to model particles
and their interactions, with physical space or spacetime as the base space.
The maps between these fields are also equivariant,
with the symmetries being part of fundamental physical laws.

It is also noteworthy that these symmetries are largely ones that
appear the most often in deep learning, such as translation and rotation
equivariance.
These similarities have already led to ideas from physics being
applied in equivariant deep learning~\citep{lang2021}.

However, one remaining difference is that physics uses equivariant partial differential
operators (PDOs) to define maps between fields, such as the gradient or Laplacian.
Therefore, using PDOs instead of convolutions in deep
learning would complete the analogy to physics and could lead to even more
transfer of ideas between subjects.

Equivariant PDO-based networks have already been designed in prior
work~\citep{shen2020,smets2020,sharp2020}. Most relevant for our work are
PDO-eConvs~\citep{shen2020}, which can be seen as the PDO-analogon of group
convolutions. However, PDO-eConvs are only one instance of equivariant
PDOs and do not cover the most common PDOs from physics, such as the gradient,
divergence, etc. Very
similarly to how steerable CNNs~\citep{cohen2016,weiler2018a,weiler2019}
generalize group convolutions, we generalize PDO-eConvs by characterizing
the set of \emph{all} translation equivariant PDOs between feature fields over
Euclidean spaces. Because of this analogy, we dub these equivariant differential
operators \emph{steerable PDOs}.

These steerable PDOs and their similarity to steerable CNNs also raise the question
of how equivariant PDOs and kernels relate to each other, and whether they can be
unified. We present a framework for equivariant maps that contains both steerable PDOs
and convolutions with steerable kernels as special cases. We then prove that this framework
defines the most general set of translation equivariant, linear, continuous maps between feature fields,
complementing recent work~\citep{aronsson2021} that describes when equivariant maps are convolutions.
Since formally developing this framework requires the
theory of Schwartz distributions, we cover the details only in
\cref{sec:distribution_equivariance}, and the
main paper can be read without any knowledge of distributions. However, we give
an overview of the main results from this distributional framework in \cref{sec:distributions_overview}.

In order to make steerable PDOs practically applicable, we describe an
approach to find complete bases for vector spaces of equivariant PDOs and then apply
this method to the most important cases. We have also implemented steerable PDOs
for all subgroups of \(\orth(2)\) (\url{https://github.com/ejnnr/steerable_pdos}).
Our code has been merged into the E2CNN library\footnote{\url{https://quva-lab.github.io/e2cnn/}}~\citep{weiler2019},
which will allow practitioners to easily use both steerable kernels and steerable PDOs within the
same library, and even to combine both inside the same network.
Finally, we test our approach empirically by comparing steerable PDOs to
steerable CNNs. In particular, we benchmark different discretization methods for
the numerical implementation.

In summary, our main contributions are as follows:
\begin{itemize}
  \item We develop the theory of equivariant PDOs on Euclidean spaces,
  giving a practical characterization of precisely when a PDO
  is equivariant under any given symmetry.
  \item We unify equivariant PDOs and kernels into one framework that provably
  contains all translation equivariant, linear, continuous maps between feature spaces.
  \item We describe a method for finding bases of the vector spaces of equivariant PDOs, and provide
  explicit bases for many important cases.
  \item We benchmark steerable PDOs using different
  discretization procedures and provide an implementation of
  steerable PDOs as an extension of the E2CNN library.
\end{itemize}

\subsection{Related work}
\paragraph{Equivariant convolutional networks}
Our approach to equivariance follows the one taken by steerable
CNNs~\citep{cohen2016,weiler2018a,weiler2019,brandstetter2021}. They represent
each feature as a map from the base space, such as \(\R^{d}\), to a fiber
\(\R^{c}\) that is equipped with a representation \(\rho\) of the point group
\(G\). Compared to vanilla CNNs, which have fiber \(\R\), steerable CNNs thus
extend the \emph{codomain} of feature maps.

A different approach is taken by regular group convolutional
networks~\citep{cohen2016a,hoogeboom2018,weiler2018,kondor2018,bekkers2021}. They represent each
feature as a map from a group \(H\) acting on the input space to \(\R\).
Because the input to the network usually does not lie in \(H\), this requires a
lifting map from the input space to \(H\). Compared to vanilla CNNs, group
convolutional networks can thus be understood as extending the \emph{domain} of
feature maps.

When \(H = \R^d \rtimes G\) is the semidirect product of the translation group
and a pointwise group \(G\), then group convolutions
on \(H\) are equivalent to \(G\)-steerable convolutions with regular representations.
For finite \(G\), the group convolution over \(G\) simply becomes a finite sum.
LieConvs~\citep{finzi2020generalizing} describe a way of implementing group convolutions
even for infinite groups by using a Monte Carlo approximation for the convolution
integral. Steerable CNNs with regular representations would have to use similar
approximations for infinite groups, but they can instead also use (non-regular)
finite-dimensional representations.
Both the group convolutional and the steerable approach can be applied to
non-Euclidean input spaces---LieConvs define group convolutions on arbitrary Lie groups
and steerable convolutions can be defined on Riemannian manifolds~\citep{cohen2019a,weiler2021}
and homogeneous spaces~\citep{cohen2019b}.

One practical advantage of the group convolutional approach employed by LieConvs
is that it doesn't require solving any equivariance constraints, which tends to
make implementation of new groups easier. They also require somewhat less heavy
theoretical machinery. On the other hand, steerable CNNs are much more general.
This makes them interesting from a theoretical angle and also has more practical
advantages; for example, they can naturally represent the symmetries of vector field
input or output. Since our focus is developing the theory of equivariant PDOs and
the connection to physics, where vector fields are ubiquitous, we are taking the
steerable perspective in this paper.

\paragraph{Equivariant PDO-based networks}
The work most closely related to ours are PDO-eConvs~\citep{shen2020},
which apply the group convolutional perspective to PDOs. Unlike LieConvs,
they are not designed to work with infinite groups. The steerable PDOs we
introduce generalizes PDO-eConvs, which are obtained as a special case
by using regular representations.

A different approach to equivariant PDO-based networks was taken by
\citet{smets2020}. Instead of applying a differential operator to input
features, they use layers that map an initial condition for a PDE to its
solution at a fixed later time. The PDE has a fixed form but several learnable
parameters and constraints on these parameters---combined with the form of the
PDE---guarantee equivariance. \Citet{sharp2020} also use a PDE, namely the diffusion equation, as part of
their DiffusionNet model, which can learn on 3D surfaces. Interestingly, the time evolution operator for the
diffusion equation is \(\exp(t\Delta)\), which can be interpreted as an infinite
power series in the Laplacian, very reminiscent of the finite Laplacian
polynomials that naturally appear throughout this paper. Studying the equivariance
of such infinite series of PDOs might be an interesting direction for future work.
We clarify the relation between PDO-based and kernel-based networks in some more detail
in \cref{sec:kernels_vs_pdos}.

\section{Steerable PDOs}
\begin{wrapfigure}{R}{0.5\linewidth}
  \vspace{-1.5em}
  \centering
  \includegraphics[width=\linewidth]{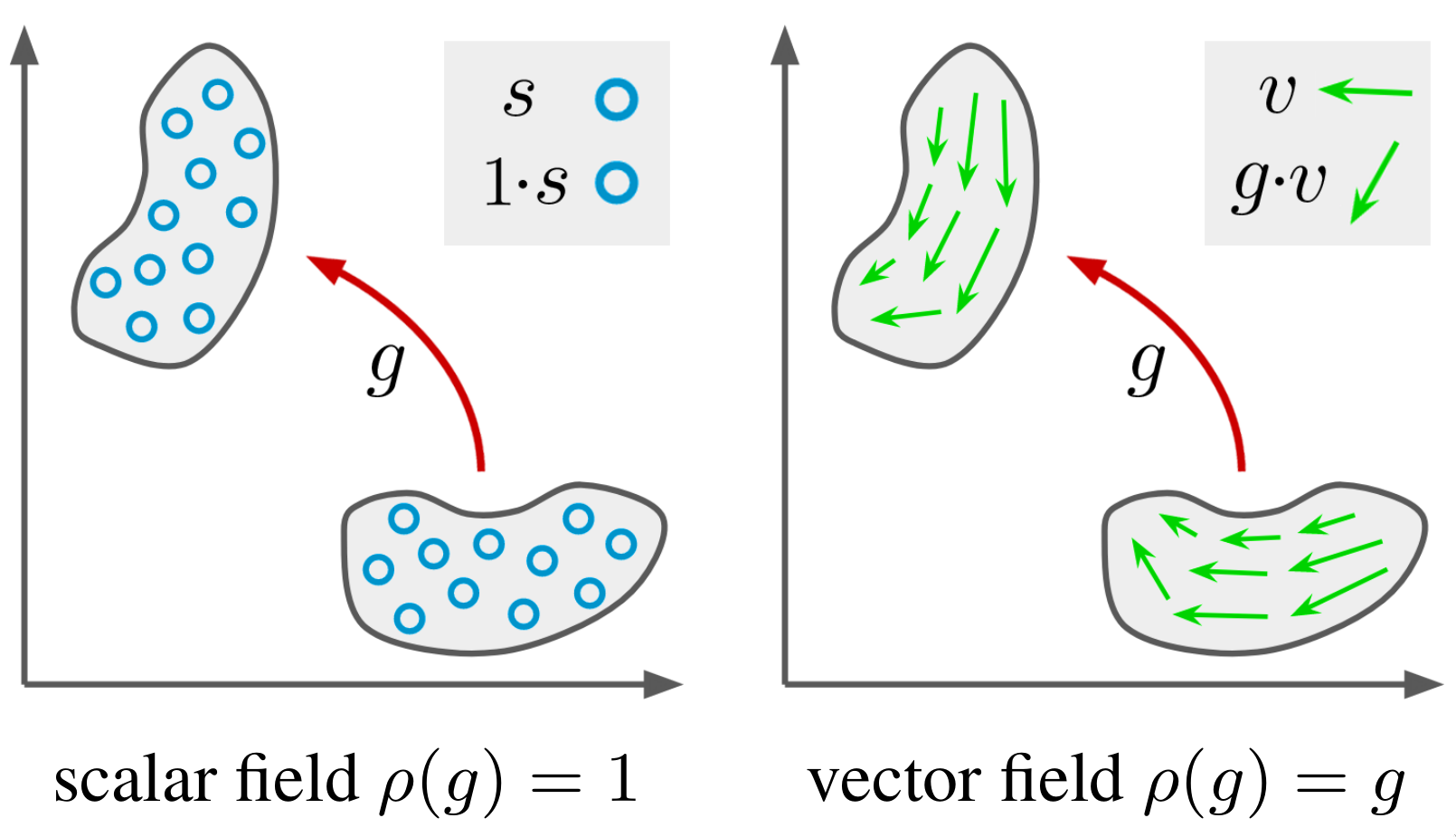}
  \caption{\small Transformation of scalar and vector fields (reproduced with
    permission from~\citet{weiler2019})}\label{fig:feature_fields}
  \vspace{-1em}
\end{wrapfigure}
In this section, we develop the theory of equivariant PDOs. We will represent
all features as smooth fields \(f: \R^{d} \to \R^{c}\) that associate a feature
vector \(f(x) \in \R^{c}\), called the \emph{fiber} at \(x\), with each point
\(x \in \R^{d}\). We write \(\mathcal{F}_{i} = C^{\infty}(\R^{d}, \R^{c_{i}})\) for the space
of these fields \(f\) in layer \(i\). Additionally, we have a
group of transformations acting on the input space \(\R^{d}\), which describes
under which symmetries we want the PDOs to be equivariant. We will
always use a group of
the form \(H = (\R^{d}, +) \rtimes G\), for some \(G \leq \GL(d, \R)\). Here,
\((\R^{d}, +)\) refers to the group of translations of \(\R^{d}\), while \(G\) is
some group of linear invertible transformations. The full group of symmetries
\(H\) is the semidirect product of these two, meaning that each element \(h \in H\)
can be uniquely written as \(h = tg\), where \(t \in \R^{d}\) is a translation and
\(g \in G\) a linear transformation. For example, if \(G = \{e\}\) is the trivial
group, we consider only equivariance under translations, as in classical CNNs,
while for \(G = \SO(d)\) we additionally consider rotational equivariance.

Each feature space \(\mathcal{F}_{i}\) has an associated group representation
\(\rho_{i}: G \to \GL(c_{i}, \R)\), which determines how each fiber \(\R^{c_{i}}\) transforms
under transformations of the input space. Briefly, \(\rho_{i}\) associates an
invertible matrix \(\rho_{i}(g)\) to each group element \(g\), such that
\(\rho_{i}(g)\rho_{i}(g') = \rho_{i}(gg')\); more details on representation theory can be
found in \cref{sec:representation_theory}. To see why these representations are
necessary, consider the feature space \(\mathcal{F} = C^{\infty}(\R^{2}, \R^{2})\) and
the group \(G = \SO(2)\). The two channels could simply be two independent
scalar fields, meaning that rotations of the input move each fiber but do not
transform the fibers themselves. Formally, this would mean using
\emph{trivial representations} \(\rho(g) = 1\) for both channels. On the other
hand, the two channels could together form a vector field, which means that each
fiber would need to itself be rotated in addition to being moved.
This would correspond to the representation \(\rho(g) = g\). These two cases are
visualized in \cref{fig:feature_fields}.

In general, the transformation of a feature \(f \in \mathcal{F}_{i}\) under an input
transformation \(tg\) with \(t \in \R^{d}\) and \(g \in G\) is given by
\begin{equation}\label{eq:action_definition}
  ((tg) \rhd_{i} f)(x) := \rho_{i}(g)f(g^{-1}(x - t))\,.
\end{equation}
The \(g^{-1}(x - t)\)
term moves each fiber spatially, whereas the \(\rho_{i}(g)\) is responsible for the
individual transformation of each fiber.

For a network, we will need maps between adjacent feature spaces
\(\mathcal{F}_{i}\) and \(\mathcal{F}_{i + 1}\). Since during this theory section, we
only consider single layers in isolation, we will drop the index \(i\) and
simply denote the layer map as \(\Phi: \mathcal{F}_{\text{in}} \to \mathcal{F}_{\text{out}}\).
We are particularly interested in \emph{equivariant} maps \(\Phi\), i.e.\ maps that commute with the action of \(H\) on the feature
spaces:
\begin{equation}\label{eq:equivariance_def}
  \Phi(h \rhd_{\text{in}} f) = h \rhd_{\text{out}} \Phi(f)\quad \forall h \in H, f \in \mathcal{F}_{\text{in}}\,.
\end{equation}

We call \(\Phi\) \emph{translation}-equivariant if \cref{eq:equivariance_def} holds for
\(h \in (\R^{d}, +)\), i.e.\ for pure translations.
Analogously, \(\Phi\) is \(G\)-equivariant if it holds
for linear transformations \(h \in G\). Because \(H\) is the semidirect product of
\(\R^{d}\) and \(G\), a map is \(H\)-equivariant if and only if it is both
translation- and \(G\)-equivariant.

\subsection{PDOs as maps between feature spaces}\label{sec:pdos}
We want to use PDOs for the layer map \(\Phi\), so we need to introduce
some notation for PDOs between multi-dimensional feature fields.
As shown in \cref{fig:feature_field_map}, such a multi-dimensional PDO
can be interpreted as a \emph{matrix} of one-dimensional PDOs.
Specifically, a PDO from \(C^{\infty}(\R^{d}, \R^{\cin})\) to
\(C^{\infty}(\R^{d}, \R^{\cout})\) is described by a \(\cout \times \cin\) matrix. For
example, the 2D divergence operator, which maps from \(\R^{2}\) to \(\R = \R^{1}\) can
be written as the \(1 \times 2\) matrix \(\mat{\partial_{1} & \partial_{2}}\). This is
exactly analogous to convolutional kernels, which can also be interpreted as
\(\cout \times \cin\) matrices of scalar-valued kernels.

To work with the one-dimensional PDOs that make up the entries of this matrix,
we use multi-index notation, so for a tuple
\(\alpha = (\alpha_{1}, \ldots, \alpha_{d}) \in \N_{0}^{d}\), we write \(\partial^{\alpha} := \partial_{1}^{\alpha_{1}}\ldots\partial_{d}^{\alpha_{d}}\).
A general one-dimensional PDO is a sum \(\sum_{\alpha}c_{\alpha}\partial^{\alpha}\), where the
coefficients \(c_{\alpha}\) are smooth functions \(c_{\alpha}: \R^{d} \to \R\) (so for now no
spatial weight sharing is assumed). The sum ranges over all multi-indices \(\alpha\),
but we require all but a finite number of coefficients to be zero everywhere, so
the sum is effectively finite. As described, a PDO between general feature
spaces is then a matrix of these one-dimensional PDOs.

\subsection{Equivariance constraint for PDOs}\label{sec:pdo_equivariance_constraint}
We now derive a complete characterization of the PDOs that are \(H\)-equivariant
in the sense defined by \cref{eq:equivariance_def}. Because a map
is equivariant under the full symmetries \(H = (\R^{d}, +) \rtimes G\) if and only if
it is both translation equivariant and \(G\)-equivariant, we split up our
treatment into these two requirements.

First, we note that translation equivariance corresponds to spatial
weight sharing, just like in CNNs (see \cref{sec:diffop_equivariance_proofs} for the proof):
\begin{proposition}\label{thm:translation_equivariance}
  A \(\cout \times \cin\)-PDO \(\Phi\) with matrix entries \(\sum_{\alpha}c_{\alpha}^{ij}\partial^{\alpha}\) is
  translation equivariant if and only if all coefficients \(c_{\alpha}^{ij}\) are
  constants, i.e.\ \(c_{\alpha}^{ij}(x) = c_{\alpha}^{ij}(x')\) for all \(x, x' \in \R^{d}\).
\end{proposition}

So from now on, we restrict our attention to PDOs with
constant coefficients and ask under which circumstances they are additionally equivariant
under the action of the point group \(G \leq \GL(d, \R)\). To answer that, we make use of a
duality between polynomials and PDOs that will appear throughout this paper:
for a PDO \(\sum_{\alpha}c_{\alpha}\partial^{\alpha}\), with \(c_{\alpha} \in \R\),\footnote{Note
that we had $c_\alpha \in C^\infty(\R^d, \R)$ before, but we now restrict ourselves to
constant coefficients and identify constant functions $\R^d \to \R$ with real numbers.} there is an associated
polynomial simply given by \(\sum_{\alpha}c_{\alpha}x^{\alpha}\), where
\(x^{\alpha} := x_{1}^{\alpha_{1}}\ldots x_{d}^{\alpha_{d}}\). Conversely,
for any polynomial \(p = \sum_{\alpha}c_{\alpha}x^{\alpha}\), we get a PDO by
formally plugging in \(\partial = (\partial_{1}, \ldots, \partial_{d})\) for \(x\), yielding
\(p(\partial) := \sum_{\alpha}c_{\alpha}\partial^{\alpha}\). We will denote the map in this direction by \(D\),
so we write \(D(p) := p(\partial)\).

We can extend this duality to PDOs between multi-dimensional fields: for a matrix \(P\) of
polynomials, we define \(D(P)\) component-wise, so \(D(P)\) will be a matrix of
PDOs given by \(D(P)_{ij} := D(P_{ij})\). To avoid confusion, we will always
denote polynomials by lowercase letters, such as \(p, q\), and matrices
of polynomials by uppercase letters, like \(P\) and \(Q\).

As a simple example of the correspondence between polynomials and PDOs, the
Laplacian operator \(\Delta\) is given by \(\Delta = D(\abs{x}^{2})\). The gradient
\((\partial_{1}, \ldots, \partial_{d})^{T}\) is induced by the \(d \times 1\)
matrix \(\mat{x_{1}, \ldots, x_{d}}^{T}\) and the 3D curl (\(d = 3\)) is induced by the \(3 \times 3\) matrix
\begin{equation}
  P = \mat{\phantom{-}0 & -x_{3} & \phantom{-}x_{2} \\ \phantom{-}x_{3} & \phantom{-}0 & -x_{1} \\ -x_{2} & \phantom{-}x_{1} & \phantom{-}0}
  \qquad \Longrightarrow \qquad
  D(P) = \mat{\phantom{-}0 & -\partial_{3} & \phantom{-}\partial_{2} \\ \phantom{-}\partial_{3} & \phantom{-}0 & -\partial_{1} \\ -\partial_{2} & \phantom{-}\partial_{1} & \phantom{-}0}\,.
\end{equation}

Finally, note that \(D\) is a ring
isomorphism, i.e.\ a bijection with \(D(p + q) = D(p) + D(q)\) and
\(D(pq) = D(p) \circ D(q)\), allowing us to switch viewpoints at will.

We are now ready to state our main result on the \(G\)-equivariance of PDOs:
\begin{theorem}\label{thm:diffop_equivariance}
  For any matrix of polynomials \(P\), the differential operator \(D(P)\) is
  \(G\)-equivariant if and only if it satisfies the \emph{PDO \(G\)-steerability constraint},
  \begin{equation}\label{eq:diffop_constraint_polynomials}
    P\left((g^{-1})^{T}x\right) = \rho_{\text{out}}(g)P(x)\rho_{\text{in}}(g)^{-1}\quad \forall g \in G, x \in \R^{d}\,.
  \end{equation}
\end{theorem}
\(P(x)\) just means that \(x\) is plugged into the polynomials in each entry of
\(P\), which results in a real-valued matrix \(P(x) \in \R^{\cout \times \cin}\), and
similarly for \(P\left((g^{-1})^{T}x\right)\). \Cref{sec:polynomial_intuition} provides
more intuition for the action of group elements on polynomials. Because \(D\) is an isomorphism,
this result completely characterizes when a PDO with constant coefficients is
\(G\)-equivariant, and in conjunction with \cref{thm:translation_equivariance}, when any PDO
between feature fields is \(H\)-equivariant. The proof of
\cref{thm:diffop_equivariance} can again be found in
\cref{sec:diffop_equivariance_proofs}.

To avoid confusion, we would like to point out that this description of PDOs as
polynomials is only a useful trick that lets us express certain operations
more easily and will later let us connect steerable PDOs to steerable kernels.
Convolving with these polynomials is \emph{not} a meaningful
operation; they only become useful when \(\partial\) is plugged in for \(x\).

\subsection{Examples of equivariant PDOs}\label{sec:examples}
To build intuition, we explore the space of equivariant PDOs for two
simple cases before covering the general solution. Consider
\(\R^{2}\) as a base space with the symmetry group \(G = \SO(2)\) and trivial
representations \(\rho_{\text{in}}\) and \(\rho_{\text{out}}\), i.e.\ maps between two
scalar fields.
The equivariance condition then becomes \(p(gx) = p(x)\), so the polynomial \(p\) has to be
rotation invariant. This is the case if and only if it can be written
as a function of \(\abs{x}^{2}\), i.e.\ \(p(x) = q(\abs{x}^{2})\) for some
\(q: \R_{\geq 0} \to \R\). Since we want \(p\) to be a polynomial, \(q\) needs to be a polynomial in one variable.
Because \(\abs{x}^{2} = x_{1}^{2} + x_{2}^{2}\), we get the PDO \(D(p) = q(\Delta)\),
where \(\Delta := \partial_{1}^{2} + \partial_{2}^{2}\) is the Laplace operator. So the \(\SO(2)\)-equivariant
PDOs between two scalar fields are exactly the polynomials in the
Laplacian, such as \(2\Delta^{2} + \Delta + 3\).

As a second example, consider PDOs that map from a vector to a scalar field, still with
\(\SO(2)\) as the symmetry group. There are two such PDOs that often occur in
the natural sciences, namely the divergence,
\(\operatorname{div} v := \partial_{1}v_{1} + \partial_{2}v_{2}\), and the 2D curl,
\(\operatorname{curl_{2D}} v := \partial_{1}v_{2} - \partial_{2}v_{1}\). Both of these are
\(\SO(2)\)-equivariant (see \cref{sec:more_examples}).
We get additional equivariant PDOs by composing with equivariant
scalar-to-scalar maps, i.e.\ polynomials in the Laplacian. Specifically, \(q(\Delta) \circ \operatorname{div}\) and
\(q(\Delta) \circ \operatorname{curl_{2D}}\) are also equivariant vector-to-scalar PDOs, for
any polynomial \(q\). We will omit the \(\circ\) from now on.

We show in \cref{sec:more_examples} that these PDOs already span the \emph{complete} space
of \(\SO(2)\)-equivariant PDOs from vector to scalar fields. Explicitly, the equivariant
PDOs in this setting are all of the form
\(q_{1}(\Delta) \operatorname{div}{} + q_{2}(\Delta) \operatorname{curl_{2D}}\)
for polynomials \(q_{1}\) and \(q_{2}\). One example of this is the PDO shown in
\cref{fig:feature_field_map}.

In these examples, as well as in other simple cases (see
\cref{sec:more_examples}), the equivariant PDOs are all combinations of
well-known operators such as the divergence and Laplacian. That
the notion of equivariance can reproduce all the intuitively \enquote{natural}
differential operators suggests that it captures the right concept.

\section{Bases for spaces of steerable PDOs}
For the purposes of deep learning, we need to be able to learn steerable PDOs.
To illustrate how to achieve this, consider again \(\SO(2)\)-equivariant PDOs
mapping from vector to scalar field. We have seen in \cref{sec:examples} that
they all have the form
\begin{equation}
q_{1}(\Delta) \operatorname{div}{} + q_{2}(\Delta) \operatorname{curl_{2D}}{}\,.
\end{equation}
In practice, we will need to limit the order of the PDOs we consider, so that we
can discretize them. For example, we could consider only polynomials \(q_{1}\)
and \(q_{2}\) of up to first order, i.e.\ \(q_{1}(z) = c_{1}z + c_{2}\) and
\(q_{2}(z) = c_{3}z + c_{4}\). This leads to PDOs of up to order three (since
\(\Delta\) is a second order PDO and \(\operatorname{div}\) and
\(\operatorname{curl_{2D}}\) are first order). The space of such equivariant
PDOs is then
\begin{equation}
  c_{1}\Delta \operatorname{div}{} + c_{2}\operatorname{div}{} + c_{3}\Delta\operatorname{curl_{2D}}{} + c_{4}\operatorname{curl_{2D}}{},\qquad c_{i} \in \R\,.
\end{equation}
We can now train the real-valued parameters \(c_{1}, \ldots, c_{4}\) and thereby
learn arbitrary equivariant PDOs of up to order three.

The general principle is that we need to find a \emph{basis} of the real vector space
of steerable PDOs; then we learn weights for a linear combination of the
basis elements, yielding arbitrary equivariant PDOs.

Different group representations are popular in equivariant deep learning practice; for
example PDO-eConvs and group convolutions correspond to so-called \emph{regular}
representations (see \cref{sec:pdo_econvs}) but quotient representations have
also been used very successfully~\citep{weiler2019}. We therefore
want to find bases of steerable PDOs for \emph{arbitrary} representations
\(\rho_{\text{in}}\) and \(\rho_{\text{out}}\). To do so, we make use of the existing work
on solving the closely related \(G\)-steerability constraint for \emph{kernels}.
In this section, we will first give a brief
overview of this kernel steerability constraint and then describe how to transfer
its solutions to ones of the \emph{PDO} steerability constraint.

\subsection{The steerability constraint for kernels}
The kernel \(G\)-steerability constraint characterizes when convolution with a kernel
\(\kappa: \R^{d} \to \R^{\cout \times \cin}\) is \(G\)-equivariant, like the PDO steerability
constraint \cref{eq:diffop_constraint_polynomials}
does for PDOs. Namely,
\begin{equation}\label{eq:steerability_constraint}
  \kappa(gx) = \abs{\det g}^{-1}\rho_{\text{out}}(g)\kappa(x)\rho_{\text{in}}(g)^{-1}\qquad \forall g \in G
\end{equation}
has to hold. This constraint was proven in~\citep{weiler2018a} for orthogonal
\(G\) and later in~\citep{weiler2021} for general \(G\).
It is very similar to the PDO steerability constraint: the only
differences are the determinant and the \(\kappa(gx)\) term on the LHS where we had \((g^{-1})^{T}\) instead of \(g\).

An explanation for this similarity is that both constraints can be seen as special
cases of a more general steerability constraint for \emph{Schwartz distributions},
which we derive in \cref{sec:distribution_equivariance}. Schwartz distributions
are a generalization of classical functions and convolutions with Schwartz distributions
can represent both classical convolutions and PDOs. As we prove in
\cref{sec:distribution_equivariance}, such distributional convolutions are the most general translation equivariant,
continuous, linear maps between feature spaces, strictly more general than either PDOs
or classical kernels. We also show how steerable PDOs can be interpreted as the
Fourier transform of steerable kernels, which we use to explain the remaining differences
between the two steerability constraints.

But for the purposes of this section, we want to draw particular attention to the fact that for
\(G \leq \orth(d)\), the two constraints become exactly identical: the
determinant then becomes \(1\), and \((g^{-1})^{T} = g\). So
for this case, which is by far the most practically important one, we will use
existing solutions of the kernel steerability constraint to find a complete basis of
equivariant PDOs.

\subsection{Transferring kernel solutions to steerable PDO bases}\label{sec:solving_constraint}
Solutions of the kernel steerability constraint have been published for
subgroups of \(\orth(2)\)~\citep{lang2021,weiler2019} and
\(\orth(3)\)~\citep{lang2021,weiler2018a},
and they all use a basis of the form
\begin{equation}
  \setcomp{\kappa_{\alpha\beta}(x) := \phi_{\alpha}(\abs{x})\chi_{\beta}\left(x / \abs{x}\right)}{\alpha \in \mathcal{A}, \beta \in \mathcal{B}}\,,
\end{equation}
where \(\mathcal{A}\) and \(\mathcal{B}\) are index sets for the radial and
angular part and the \(\phi_{\alpha}\) span the entire space of radial functions.
The reason is that the steerability constraint \cref{eq:steerability_constraint}
constrains only the angular part if \(G \leq \orth(d)\), because orthogonal groups
preserve distances.

The angular functions \(\chi_{\beta}\) are only defined on the sphere \(S^{d - 1}\).
But crucially, we show in \cref{sec:solutions_proofs} that they can all be
canonically extended to polynomials defined on \(\R^{d}\).\footnote{This holds for
  the cases we discuss here, i.e.\ subgroups of \(\orth(2)\) and \(\orth(3)\).
  In higher dimensions, the situation is more complicated, but for those the kernel
  solutions have not yet been worked out anyway.}
Concretely, we define \(\tilde{\chi}_{\beta}(x) := \abs{x}^{l_{\beta}}\chi_{\beta}(x/\abs{x})\), where
\(l_{\beta}\) is chosen minimally such that \(\tilde{\chi}_{\beta}\) is a polynomial. What
we prove in \cref{sec:solutions_proofs} is that such an \(l_{\beta}\) always exists.

The radial part \(\phi_{\alpha}\) in the kernel basis consists of \emph{unrestricted} radial
functions. To get a basis for the space of \emph{polynomial} steerable kernels,
it is enough to use only powers of \(\abs{x}^{2}\) for the radial part.
Specifically, we show in \cref{sec:transfer_solutions} that
\begin{equation}
  \setcomp{p_{k\beta} := \abs{x}^{2k}\tilde{\chi}_{\beta}}{k \in \N_{0}, \beta \in \mathcal{B}}
\end{equation}
is a basis for the space of polynomial steerable kernels.
As a final step, we interpret each polynomial as a PDO using the isomorphism
\(D\) defined in \cref{sec:pdo_equivariance_constraint}, yielding a \emph{complete}
basis of the space of steerable PDOs. The \(\abs{x}^{2k}\) terms become
Laplacian powers \(\Delta^{k}\), while in the examples from \cref{sec:examples},
\(\tilde{\chi}_{\beta}\) corresponds to divergence and curl.

In \cref{sec:solutions}, we apply this procedure to all compact subgroups of \(\orth(2)\),
and to \(\SO(3)\) and \(\orth(3)\) to obtain concrete solutions. Recently, \citet{cesa2022}
described a more general framework for constructing steerable kernels for subgroups of \(\orth(d)\).
In particular, they implement explicit solutions for other subgroups of
\(\orth(3)\), which could be adapted to PDOs using the method described here.

\section{Convolutions with Schwartz distributions}\label{sec:distributions_overview}
Steerable PDOs are an example showing that not all equivariant maps can be represented
as convolutions with steerable kernels. A natural question is thus: what \emph{are} the
most general equivariant maps? And how can we characterize their equivariance?
The steerability constraints for kernels and PDOs are strikingly similar,
yet they are not quite identical for non-orthogonal groups. So what is their common
generalization?

In this section, we address these questions by introducing a framework that represents
maps between feature spaces as convolutions, but with \emph{Schwartz distributions} rather
than kernels. A technical introduction to Schwartz distributions can be found in
\cref{sec:distributions}, but for the purposes of this section, a cursory understanding suffices:
distributions are a generalization of functions; in our context, we consider distributions
that generalize functions $\kappa: \R^d \to \R^{\cout \times \cin}$
(to be understood as kernels).\footnote{Technically, the \enquote{distributions} we consider
here are \emph{matrices} of Schwartz distributions, just like kernels are matrices of functions
 that map into $\R$. We discuss this formally in \cref{sec:distributions}, but in this section
 we won't make the distinction.}

Any classical (i.e.\ non-distributional) kernel can be interpreted as a distribution,
thus it is immediately clear that convolutions with distributions generalize the usual
convolutions with kernels.
As we discuss in more detail in \cref{sec:distributions_diffops}, PDOs can also be written
as convolutions with distributions (more specifically, with derivatives of the so-called
Dirac $\delta$ distribution). Thus, convolutions with distributions at least unify PDOs and
classical convolutions.

It turns out that convolutions with Schwartz distributions are in fact the most general
translation equivariant linear maps in quite a strong sense:
\begin{theorem}
  Let \(\Phi: C_c^\infty(\R^d, \R^{\cin}) \to C_c^\infty(\R^d, \R^{\cout})\) be a
  translation equivariant continuous linear map. Then there is a Schwartz distribution $T$
  such that \(\Phi(f) = T * f\) for all $f \in C_c^\infty(\R^d, \R^{\cin})$.
\end{theorem}
Here, $C_c^\infty$ denotes the space of smooth functions with compact support. Continuity
of $T$ is defined with respect to the canonical LF topology, see \cref{sec:distributions,sec:distribution_equivariance}
for details and the proof.

To put this result into perspective, we note that analogous results have been proven
under the additional assumption that $\Phi$ is an integral transform, i.e.
\begin{equation}
  \Phi(f)(x) = \int_{\R^d} k(x, y)f(y)dy
\end{equation}
for some function $k: \R^d \times \R^d \to \R^{\cout \times \cin}$
\citep{weiler2018a,cohen2019}. With this assumption,
one recovers the often-stated claim that all translation equivariant maps are convolutions
with classical kernels.
However, assuming an integral transform excludes maps such as PDOs right from the beginning.
Our characterization relaxes this assumption and only requires continuity and linearity instead.

We now generalize the steerability constraints for kernels and PDOs to a single constraint
for convolutions with Schwartz distributions:
\begin{theorem}
  Let $T$ be a Schwartz distribution.
  Then the map \(f \mapsto T * f\) is \(G\)-equivariant if and only if
  \begin{equation}\label{eq:general_steerability_constraint}
    T \circ g = \abs{\det g}^{-1}\rho_{\text{out}}(g) T \rho_{\text{in}}(g)^{-1}\,.
  \end{equation}
\end{theorem}
The kernel steerability constraint is recovered as a straightforward special case.
In \cref{sec:fourier_duality}, we show how the PDO steerability constraint can be
obtained by taking the Fourier transform of \cref{eq:general_steerability_constraint}
and identifying PDOs with the Fourier transform of kernels. This explains why
the PDO steerability constraint looks slightly different from the one for kernels:
since PDOs can be interpreted as the Fourier transform of kernels, they transform
differently under $\GL(\R^d)$. The constraints coincide for $G \leq \orth(d)$
because Fourier transforms commute with rotations and reflections.

\section{Experiments}\label{sec:experiments}

\begin{wraptable}[32]{R}{0.58\linewidth}
  \small
  \centering
  \caption{\small MNIST-rot results. Test errors \(\pm\)
    standard deviations are averaged over six runs. Vanilla CNN
    is a solely translation equivariant model (\(G = \{e\}\)) with the same general architecture. See main text
  for details on the models.}\label{tab:mnist_rot}
  \begin{tabular}{p{1cm}p{1cm}lll}
    \toprule
    Represen- tation $\rho$ & Method & Stencil &  Error [\%] & Params \\
    \midrule
    \multirow{2}{1cm}{--} & \multirow{2}{1.3cm}{Vanilla CNN} & \(3 \times 3\) & \(2.001 \pm 0.030\) & \multirow{2}{0.8cm}{1.1M} \\
                   & & \(5 \times 5\) & \(1.959 \pm 0.055\) & \\
    \cmidrule{1-5}
    \multirow{8}{1cm}[-0.5em]{regular (our basis)} & \multirow{2}{1.3cm}{Kernels} & \(3 \times 3\) & \(0.741 \pm 0.036\) & 837K \\
                   & & \(5 \times 5\) & \(0.683 \pm 0.021\) & 1.1M \\
    \cmidrule{2-5}
                   & \multirow{2}{1.3cm}{FD}      & \(3 \times 3\) & \(1.196 \pm 0.062\) & 837K \\
                   & & \(5 \times 5\) & \(1.54\phantom{0} \pm 0.32\) & 941K \\
    \cmidrule{2-5}
                   & \multirow{2}{1.3cm}{RBF-FD}  & \(3 \times 3\) & \(1.313 \pm 0.065\) & 837K \\
                   & & \(5 \times 5\) & \(1.475 \pm 0.020\) & 941K \\
    \cmidrule{2-5}
                   & \multirow{2}{1.3cm}{Gauss}   & \(3 \times 3\) & \(0.795 \pm 0.030\) & 837K \\
                   & & \(5 \times 5\) & \(0.750 \pm 0.017\) & 941K \\
    \cmidrule{1-5}
    \multirow{2}{1cm}{regular (PDO-eConv)} & \sr FD\footnotemark{} & \(5 \times 5\) & \(1.98\phantom{0} \pm 0.11\) & \multirow{2}{0.8cm}[-0.5em]{982K} \\
    & Gauss & \(5 \times 5\) & \(0.831 \pm 0.039\) & \\
    \cmidrule{1-5}
    \multirow{8}{1cm}[-0.5em]{quotient (our basis)} & \multirow{2}{1.3cm}{Kernels} & \(3 \times 3\) & \(0.717 \pm 0.026\) & 877K \\
                   & & \(5 \times 5\) & \(0.670 \pm 0.011\) & 1.1M \\
    \cmidrule{2-5}
                   & \multirow{2}{1.3cm}{FD}      & \(3 \times 3\) & \(1.143 \pm 0.063\) & 877K \\
                   & & \(5 \times 5\) & \(1.347 \pm 0.026\) & 951K \\
    \cmidrule{2-5}
                   & \multirow{2}{1.3cm}{RBF-FD}  & \(3 \times 3\) & \(1.303 \pm 0.077\) & 877K \\
                   & & \(5 \times 5\) & \(1.422 \pm 0.040\) & 951K \\
    \cmidrule{2-5}
                   & \multirow{2}{1.3cm}{Gauss}   & \(3 \times 3\) & \(0.825 \pm 0.053\) & 877K \\
                   & & \(5 \times 5\) & \(0.744 \pm 0.040\) & 951K \\
    \bottomrule
  \end{tabular}
\end{wraptable}
\footnotetext{As in the original PDO-eConv paper~\citep{shen2020}. Note that
  their performance is better, which is simply caused by their different
  architecture and hyperparameters.}
\paragraph{Implementation}
We developed the theory of steerable PDOs in a continuous setting, but for
implementation the PDOs need to be discretized, just like steerable kernels.
The method of discretization is completely independent of
the steerable PDO basis, so steerable PDOs can be combined with any
discretization procedure. We compare three methods,
\emph{finite differences} (FD), \emph{radial basis function finite
  differences} (RBF-FD) and \emph{Gaussian derivatives}.

Finite differences are a generalization of the usual central difference
approximation and are the method used by PDO-eConvs~\citep{shen2020}.
RBF-FD finds stencils by demanding that the discretization should become
exact when applied to radial basis functions placed on the stencil points. Its
advantage over FD is that it can be applied to structureless point clouds rather
than only to regular grids.
Gaussian derivative stencils work by placing a Gaussian on the target point and
then evaluating its derivative on the stencil points.
Like RBF-FD, this also works on point clouds, and the Gaussian also has a
slight smoothing effect, which is why this discretization is often used in Computer Vision.
Formal descriptions of all three discretization methods can be found in \cref{sec:discretization}.

In addition to discretization, the infinite basis of steerable PDOs or kernels
needs to be restricted to a finite subspace. For kernels, we use the bandlimiting filters by
\citet{weiler2019}. For PDOs, we limit the total derivative order to two for
\(3 \times 3\) stencils and to three for \(5 \times 5\) stencils (except for PDO-eConvs, where
we use the original basis that limits the maximum order of
\emph{partial} derivatives).

Finally, steerable PDOs and steerable kernels only replace the convolutional
layers in a classical CNN. To achieve an equivariant network, all the other
layers, such as nonlinearities or Batchnorm also need to be equivariant.
\Citet{weiler2019} discuss in details how this can be achieved for various types
of layers. In our experiments, we use exactly the same implementation they do.
Care also needs to be taken with biases in the PDO layers.
Here, we again follow~\citep{weiler2019} by adding a bias only to
the trivial irreducible representations that make up \(\rho_{\text{out}}\).

\paragraph{Rotated MNIST}
We first benchmark steerable PDOs on rotated MNIST~\citep{larochelle2007}, which
consists of MNIST images that have been rotated by different angles, with 12k
train and 50k test images. Our results can be found in \cref{tab:mnist_rot}. The
models with \(5 \times 5\) stencils use an architecture that \citet{weiler2019} used
for steerable CNNs, with six \(C_{16}\)-equivariant layers followed by two fully
connected layers. The first column gives the representation under which the six
equivariant layers transform (see \cref{sec:representation_theory} for their
definitions). PDO-eConvs implicitly use regular representations (see
\cref{sec:pdo_econvs}), but with a slightly different basis than the one we
present, so we test both bases. We also tested models that are
\(D_{16}\)-equivariant in their first layers and \(C_{16}\)-equivariant in their
last one but did not find any improvements, see \cref{sec:more_experiments}.
For the models with \(3 \times 3\) stencils,
we use eight instead of six \(C_{16}\)-equivariant layers, in order to compensate for the
smaller receptive field and keep the parameter count comparable. The remaining
differences between kernel and PDO parameter counts come from the fact that the
basis restrictions necessarily work slightly differently (via bandlimiting
filters or derivative order restriction respectively).
All models were trained with 30 epochs and hyperparameters
based on those by \citet{weiler2019}, though we changed the learning rate
schedule and regularization slightly because this improved performance for all
models, including kernel-based ones. The training data is augmented with random
rotations. Precise descriptions of the architecture and hyperparameters can be
found in \cref{sec:experiment_details}.

\begin{wraptable}[17]{R}{0.5\linewidth}
  \vspace{-1em}
  \small
  \centering
  \caption{\small STL-10 results, again over six runs. All models except the
    vanilla CNN use regular representations, see main text for details.}\label{tab:stl_10}
  \begin{tabular}{p{1.6cm}cll}
    \toprule
    Method & Groups & Error [\%] & Params \\
    \midrule
    Vanilla CNN & -- & \(12.7 \pm 0.2\) & 11M \\
    \midrule
    \multirow{2}{1.5cm}{Kernels} & \(D_{8}D_{4}D_{1}\) & \(10.7 \pm 0.6\) & \multirow{2}{0.8cm}{4.2M} \\
            & \(D_{4}D_{4}D_{1}\) & \(10.2 \pm 0.4\) & \\
    \midrule
    \multirow{2}{1.5cm}{FD}      & \(D_{8}D_{4}D_{1}\) & \(12.1 \pm 0.6\) & \multirow{6}{0.8cm}[-0.5em]{3.2M} \\
            & \(D_{4}D_{4}D_{1}\) & \(12.1 \pm 0.7\) & \\
    \cmidrule{1-3}
    \multirow{2}{1.5cm}{RBF-FD}  & \(D_{8}D_{4}D_{1}\) & \(14.3 \pm 0.4\) & \\
            & \(D_{4}D_{4}D_{1}\) & \(14.3 \pm 0.4\) & \\
    \cmidrule{1-3}
    \multirow{2}{1.5cm}{Gauss}   & \(D_{8}D_{4}D_{1}\) & \(11.2 \pm 0.3\) & \\
            & \(D_{4}D_{4}D_{1}\) & \(10.6 \pm 0.8\) & \\
    \bottomrule
  \end{tabular}
\end{wraptable}
\paragraph{STL-10}
The rotated MNIST dataset has global rotational symmetry by design, so it is
unsurprising that equivariant models perform well. But interestingly, rotation
equivariance can also help for natural images without global rotational
symmetry~\citep{weiler2019,shen2020}. We therefore benchmark steerable PDOs on
STL-10~\citep{coates2011}, where we only use the labeled portion of 5000
training images. The results are shown in \cref{tab:stl_10}. The model
architecture and hyperparameters are exactly the same as in~\citep{weiler2019},
namely a Wide-ResNet-16-8 trained for 1000 epochs with random crops, horizontal
flips and Cutout~\citep{devries2017} as data augmentation. The group column
describes the equivariance group in each of the three residual blocks. For
example, \(D_{8}D_{4}D_{1}\) means that the first block is equivariant under
reflections and 8 rotations, the second under 4 rotations and the last one only
under reflections. All layers use regular representations. The \(D_{8}\)-equivariant layers use \(5 \times 5\) filters to
improve equivariance, whereas the other layers use \(3 \times 3\) filters.

\paragraph{Fluid flow prediction}
In the previously described tasks, the input and output representations are all
trivial. To showcase the use of non-trivial output representations, we predict
laminar fluid flow around various objects, following \citet{ribeiro2020}. In this
case, the network outputs a \emph{vector field}, which behaves differently under rotations
than scalar outputs, and whose equivariance cannot be represented using PDO-eConvs,
since they only implement trivial and regular representations. \Cref{fig:flow_prediction}
illustrates these network outputs. Our hyperparameters and architecture closely
follow \citet{ribeiro2020}, though we add data augmentation that makes the problem
more challenging, see \cref{sec:experiment_details} for details.

\begin{figure}
  \centering
  \includegraphics[width=\textwidth]{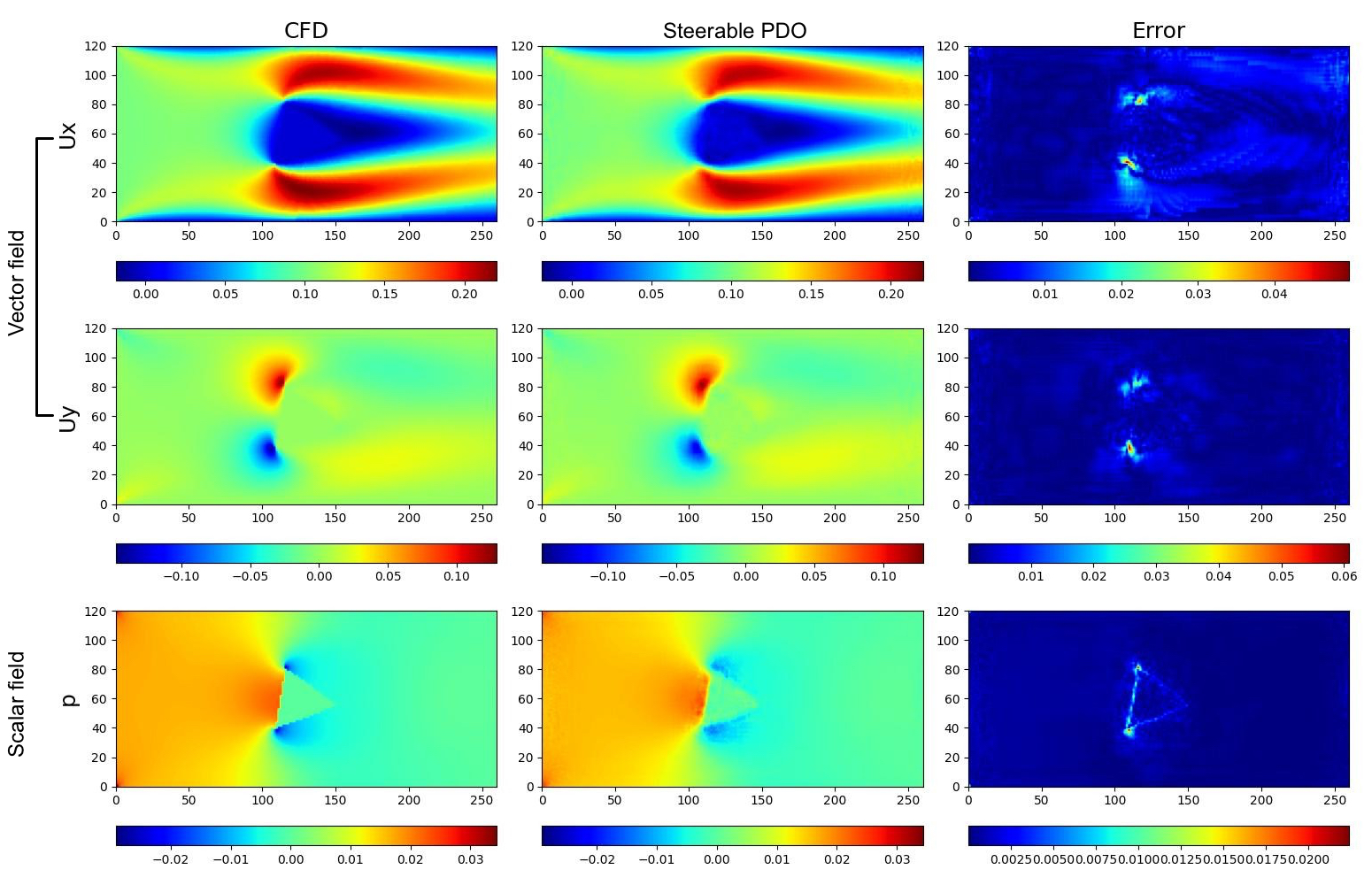}
  \caption{\small Ground truth simulation of a fluid flow (left), prediction generated by a $C_8$-equivariant PDO network 
  (middle), and prediction error (right). First two rows show the two components of the velocity field, and for
  the purposes of equivariance, we treat them as a single feature field with a vector representation.
  The third row shows the pressure, a simple scalar field. Results with other methods look similar,
  the differences in prediction accuracy are hard to recognize visually.}\label{fig:flow_prediction}
\end{figure}

\begin{wraptable}[10]{R}{0.42\linewidth}
  \vspace{-1em}
  \small
  \centering
  \caption{\small Mean squared test error for prediction of velocity and pressure of laminar
  flow around different objects.}\label{tab:fluid_flow}
  \begin{tabular}{lll}
    \toprule
    Method & Equivariance & MSE \\
    \midrule
    \multirow{2}{1.3cm}{Kernel} & ---   & \(3.20 \pm 0.22\) \\
                                & $C_8$ & \(2.26 \pm 0.14\) \\
    \midrule
    \multirow{2}{1.3cm}{PDO}    & ---   & \(2.75 \pm 0.21\) \\
                                & $C_8$ & \(2.32 \pm 0.08\) \\
    \bottomrule
  \end{tabular}
\end{wraptable}

\Cref{tab:fluid_flow} shows a clear advantage of the equivariant networks over the non-equivariant
ones.\footnote{Note that our non-equivariant performance is significantly worse than the one obtained
by \citet{ribeiro2020}---this is because we randomly rotated the samples, resulting in a more challenging
task.} Steerable PDOs perform slightly worse than steerable kernels, though the difference is within the
error intervals. They still perform clearly better than any non-equivariant method.
The PDO results are based on Gaussian discretization, since that performed best in our other
experiments.

\paragraph{Equivariance errors}
In the continuum, steerable CNNs and steerable PDOs are both \emph{exactly} equivariant.
But the discretization on a square grid leads to unavoidable equivariance errors for rotations that aren't
multiples of $\frac{\pi}{2}$. The violation of equivariance in practice is thus closely
connected to the discretization error.
For finite differences, the discretization error is particularly easy to bound asymptotically,
and as pointed out by \citet{shen2020}, this places the same asymptotic bound on the equivariance error.
However, our experiments show that empirically, finite differences don't lead to a particularly low equivariance
error (kernels and all PDO discretizations perform similarly).
See \cref{tab:equivariance_errors} in \cref{sec:more_experiments} for details.

\paragraph{Locality of PDOs}
While all equivariant models improve significantly over the non-equivariant CNN,
the method of discretization plays an important role for PDOs.
The reason that FD and RBF-FD underperform kernels is that they don't make full use of the stencil,
since PDOs are inherently local operators. When a \(5 \times 5\) stencil is used, the
outermost entries are all very small compared to the inner ones, and even in
\(3 \times 3\) kernels, the four corners tend to be closer to zero (see
\cref{sec:more_experiments} for images of stencils to illustrate this). Gaussian
discretization performs significantly better and almost as well as kernels
because its smoothing effect alleviates these issues. This fits the observation
that kernels and Gaussian methods profit from using \(5 \times 5\) kernels, whereas
these do not help for FD and RBF-FD (and in fact decrease performance because of
the smaller number of layers).

\FloatBarrier
\WFclear

\section{Conclusion}\label{sec:conclusion}
We have described a general framework for equivariant PDOs acting on feature
fields over Euclidean space. With this framework, we found strong similarities
between equivariant PDOs and equivariant convolutions, even unifying the two using
convolutions with Schwartz distributions. We exploited these similarities to
find bases for equivariant PDOs based on existing solutions for steerable kernels.

Our experiments show that the locality of PDOs can be a disadvantage compared to
convolutional kernels. However, our approach for equivariance can easily be
combined with any discretization method, and we show that using Gaussian derivatives
for discretization alleviates the issue.
Equivariant PDOs could also be very useful in cases where their strong locality
is a desideratum rather than a drawback.

The theory developed in this work provides the necessary foundation for applications
where equivariant PDOs, rather than kernels, are needed. For example, Probabilistic
Numerical CNNs~\citep{finzi2020}
use PDOs in order to parameterize convolutions on \emph{continuous} input data.
\Citeauthor{finzi2020} also derive a constraint to make these PDOs equivariant, which is a special
case of our PDO \(G\)-steerability constraint \cref{eq:diffop_constraint_polynomials}.
The solutions to this constraint presented in this paper are the missing piece for
implementing and empirically evaluating equivariant Probabilistic Numerical CNNs -- neither of
which \citeauthor{finzi2020} do.

Another promising application of PDOs is an extension
to manifolds. Gauge CNNs~\citep{cohen2019a,kicanaoglu2019,haan2020,weiler2021} are a rather general
framework for convolutions on manifolds; see \citet{weiler2021} for a thorough treatment and
literature review.
As in the Euclidean case, Gauge CNNs use feature fields that
transform according to some representation \(\rho\). The kernels are
defined on the tangent space and are still constrained by the \(G\)-steerability constraint.
Because of that, our approach to equivariant Euclidean PDOs is very
well-suited for generalization to manifolds and our steerable PDO solutions will still remain
valid. One advantage of PDOs in this setting is that they require no Riemannian structure
and can achieve equivariance with respect to arbitrary diffeomorphisms, as is common in
physics, instead of mere isometry equivariance.

\subsubsection*{Reproducibility Statement}
The appendix contains complete proofs for all of our theoretical claims. 

The code necessary to reproduce our experiments can be found at
\url{https://github.com/ejnnr/steerable_pdo_experiments}. The required datasets are downloaded
and preprocessed automatically. The exact hyperparameters we used are available as pre-defined
configuration options in our scripts. We also include a version lockfile for installing
precisely the right versions of all required Python packages. Our implementation of
steerable PDOs is easy to adapt to different use cases and fully documented, allowing
other practitioners to test the method on different datasets or tasks.

\ificlrfinal
\subsubsection*{Acknowledgments}
We would like to thank Gabriele Cesa and Leon Lang for discussions on
integrating steerable PDOs into the E2CNN library and on solutions of the
kernel \(G\)-steerability constraint. This work was supported by funding from
QUVA Lab.
\fi

{
\small
\bibliographystyle{iclr2022_conference}
\bibliography{references}
}

\newpage
\section*{Supplementary material}
\appendix
\section{Steerable PDOs between vector and scalar fields}\label{sec:more_examples}
We give bases for the space of steerable PDOs for many important cases in
\cref{sec:solution_tables}. However, the generality of the description there
obscures the connection to well-known PDOs such as the gradient or divergence.
So to complement the general solutions, we discuss a few simple cases in much
more detail in this section. We will see that Laplacian, gradient, divergence,
and curl are all rotation equivariant and, more interestingly, that all
rotation equivariant PDOs can be constructed by combining these (in the simple
settings we cover in this section). In many cases, we rederive the solutions even
though all of them would follow immediately from the general case in
\cref{sec:solution_tables}, in order to provide some intuition on why these are
the only equivariant PDOs. Readers who are only interested in an overview of the
results may wish to skip to the end of this section.

\subsection{Scalar to scalar PDOs}
We have already argued in \cref{sec:examples} that the \(\SO(2)\)-equivariant
PDOs between two scalar fields are precisely polynomials in the Laplacian, i.e.\
of the form \(q(\Delta)\) for an arbitrary real polynomial \(q \in \R[x]\). The
derivation given there applies without changes to \(\SO(d)\) for \(d > 2\) and
to \(\orth(d)\) as well, so the same holds in these cases.

\subsection{Scalar to vector PDOs}
We start by considering the case where \(\rho_{\text{in}}\) is trivial (with
\(\cin = 1\)) and \(\rho_{\text{out}}\) is the vector field representation
(i.e.\ \(\cout = d\)). We can then represent PDOs by a \(d \times 1\) matrix of
polynomials, i.e.\ a column vector.
The PDO steerability constraint becomes
\begin{equation}
  P(gx) = gP(x)
\end{equation}
since \(\rho_{\text{in}}\) is trivial. A good mental model for this subsection
is to think of \(P\) as a vector field \(P: \R^{d} \to \R^{d}\) whose entries happen
to be polynomials. The steerability constraint simply states that this vector
field must \enquote{look the same} after a rotation, see
\cref{fig:equivariant_vector_fields} for examples.

We begin by discussing the case \(G = \SO(2)\), which is somewhat different from
\(\orth(2)\) and from \(d > 2\).
Any rotation equivariant vector field on \(\R^{2}\) is fully determined by its values
on the ray \(\setcomp{(x_{1}, 0)}{x_{1} > 0}\). Specifically, the rotation equivariant
vector fields \(v\) are precisely those that can in polar coordinates be written
as
\begin{equation}
  v(r, \phi) =
  \begin{pmatrix}
    \cos\phi & -\sin\phi \\
    \sin \phi & \phantom{-}\cos \phi
  \end{pmatrix}
  \begin{pmatrix}
    f_1(r) \\
    f_2(r)
  \end{pmatrix}
\end{equation}
for arbitrary radial parts \(f_1, f_2\). In Cartesian coordinates, we can write this as
\begin{equation}
  v(x, y) =
  \begin{pmatrix}
    x & -y \\
    y & \phantom{-}x
  \end{pmatrix}
  \begin{pmatrix}
    \tilde{f}_1(x^2 + y^2) \\
    \tilde{f}_2(x^2 + y^2)
  \end{pmatrix}
\end{equation}
where \(\tilde{f_i}(z) := \frac{f_i(\sqrt{z})}{\sqrt{z}}\). We need \(v_1\) and
\(v_2\) to be polynomials in \(x, y\), which means that the \(\tilde{f}_i\) need
to be polynomials. If we then apply the \(D\) isomorphism, we get:
\begin{proposition}
  The \(\SO(2)\)-equivariant differential operators from a scalar to a vector
  field are exactly those of the form
  \begin{equation}
    \begin{pmatrix}
      \partial_1 & -\partial_2 \\
      \partial_2 & \phantom{-}\partial_1
    \end{pmatrix}
    \begin{pmatrix}
      q_1(\Delta) \\
      q_2(\Delta)
    \end{pmatrix}
    = q_{1}(\Delta)\mat{\partial_{1} \\ \partial_{2}} + q_{2}(\Delta)\mat{-\partial_{2} \\ \phantom{-}\partial_{1}}
  \end{equation}
  where \(q_1\) and \(q_2\) are arbitrary polynomials.
\end{proposition}
In words, the \(\SO(2)\)-equivariant PDOs are all linear combinations of the
gradient and the transpose of the 2D curl, \((-\partial_{2}, \partial_{1})^{T}\), with
coefficients being polynomials in the Laplacian (rather than just real numbers).

\begin{figure}
  \centering
  \begin{subfigure}{.45\linewidth}
    \includegraphics[width=\linewidth]{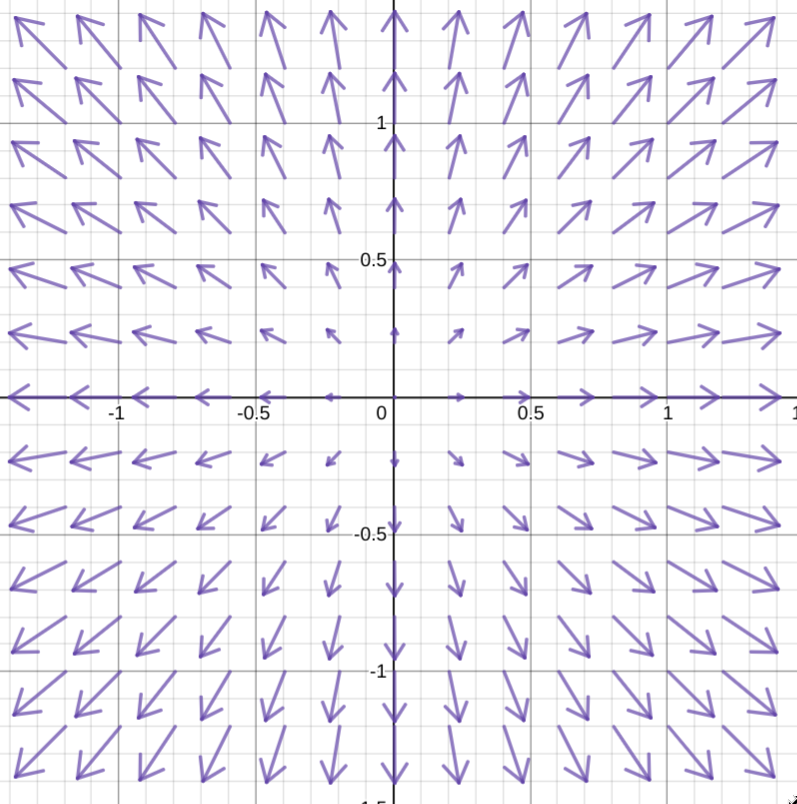}
    \caption{Vector field that induces the gradient}
    \label{fig:gradient_field}
  \end{subfigure}
  \begin{subfigure}{.45\linewidth}
    \includegraphics[width=\linewidth]{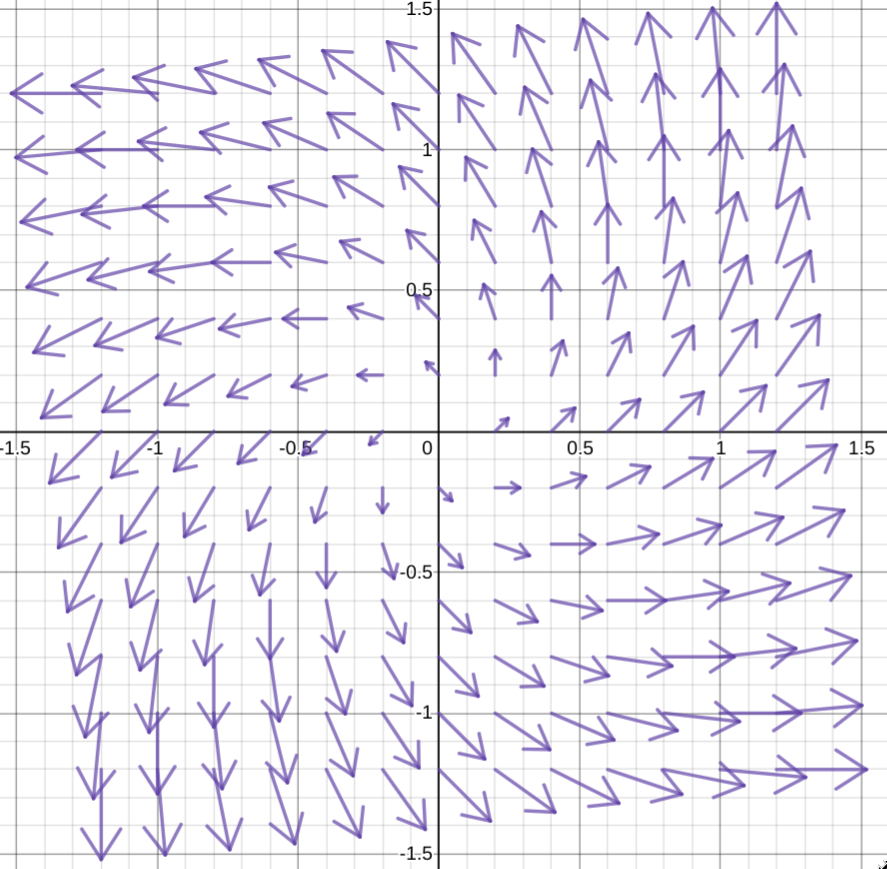}
    \caption{More complex equivariant vector field (a linear combination of
      gradient and transpose curl)}
    \label{fig:rotation_field}
  \end{subfigure}
  \caption{Two examples of \(\SO(2)\)-equivariant vector
    fields.\protect\footnotemark}\label{fig:equivariant_vector_fields}
\end{figure}
\footnotetext{Created using \url{https://www.desmos.com/calculator/eijhparfmd}}

The gradient is also equivariant under reflections, i.e.\
\(\orth(2)\)-equivariant, and it easily generalizes to higher dimensions.
However, the transpose 2D curl only appears in this particular setting: it is
not reflection-equivariant, and it does not have an analogon in higher
dimensions.\footnote{To avoid confusion, we remark that the 3D curl of course
  exists but maps between two vector fields. The 2D curl is in fact closely
  related to the 3D curl, and the fact that it does not have a
  higher-dimensional analogon corresponds to the fact that the 3D curl, as a
  vector to vector PDO cannot easily be generalized to higher dimensions.}
We summarize this in the following result:
\begin{proposition}\label{thm:2d_equivariance}
  Let \(G = \orth(d)\) for \(d \geq 2\) or \(G = \SO(d)\) for \(d > 2\). Then the
  \(G\)-equivariant differential operators from a scalar to a vector field are
  exactly those of the form
  \begin{equation}
    q(\Delta)\operatorname{grad}
  \end{equation}
  for an arbitrary polynomial \(q \in \R[x]\).
\end{proposition}
A possible intuition for why \(\SO(2)\) is a special case is that
\(\SO(2) \cong S^{1}\) whereas \(\SO(d) \not\cong S^{d - 1}\) for \(d > 2\) and
\(\orth(d) \not \cong S^{d - 1}\) for \(d \geq 2\). Our construction above heavily
makes use of the fact that rotations of \(\R^{2}\) correspond one-to-one to
angles, i.e.\ points on \(S^{1}\), and this construction thus does not
generalize to any other cases.

Before we prove \cref{thm:2d_equivariance}, we show a helpful lemma:
\begin{lemma}
  Let \(d > 2\). Then for any linearly independent vectors \(v, w \in \R^d\),
  there is a rotation \(g \in \SO(d)\) such that \(gv = v\) but \(gw \neq w\).
  For \(d = 2\), there is such a \(g \in \orth(2)\).
\end{lemma}
\begin{proof}[Proof of lemma]
  Since \(v\) and \(w\) are linearly independent, we can write
  \(\R^d = \vspan(v) \oplus W\), where \(W\) is a linear subspace containing
  \(w\). Pick an element \(\tilde{g} \in \SO(d - 1)\) (or
  \(\orth(d - 1) = \orth(1) = \{\pm 1\}\) in the \(d = 2\) case) such that
  \(\tilde{g}w \neq w\), where \(\tilde{g}\) acts on \(W\) (this always exists,
  note that \(w \neq 0\)). Then there is a \(g \in \SO(d)\) (or \(\orth(d)\))
  that restricts to \(\tilde{g}\) on \(W\) and is the identity on \(\vspan(v)\)
  (in terms of matrices with respect to a basis of the form \(\{v, \ldots\}\), \(g\)
  would be block-diagonal, with a \(1 \times 1\) identity block and a
  \((d - 1) \times (d - 1)\) block for \(\tilde{g}\)).
\end{proof}
\begin{proof}[Proof of \cref{thm:2d_equivariance}]
  One direction is clear: the gradient is induced by the matrix of polynomials
  \(P(x) = x, \; x \in \R^{d}\), which is clearly equivariant. We have also seen
  that polynomials of the Laplacian are equivariant, and we can compose these
  equivariant PDOs to get another equivariant PDO. So what remains to show is
  that no other equivariant PDOs exist.

  So let \(P\) be \(G\)-equivariant, with \(G\) as in
  \cref{thm:2d_equivariance}. Let furthermore \(x \in \R^{d}\) be arbitrary and
  \(g \in G\) be an element in the stabilizer of \(x\), i.e.\ \(gx = x\). Then we have
  \begin{equation}
    gP(x) = P(gx) = P(x)\,.
  \end{equation}
  In other words, for any \(g \in G\) with \(gx = x\), we also have \(gP(x) = P(x)\).
  By the lemma, \(x\) and \(P(x)\) are thus linearly dependent, i.e.\ \(P(x) = cx\)
  for this particular \(x\) and some \(c \in \R\).

  We can apply this reasoning to any \(x \in \R^{d}\), so there is a function
  \(c: \R^{d} \to \R\) such that \(P(x) = c(x)x\). Furthermore,
  \begin{equation}
    c(x)gx = gP(x) = P(gx) = c(gx)gx\,,
  \end{equation}
  so \(c\) has to be rotation invariant. As we have already argued,
  this implies that \(c(x) = q(x_{1}^{2} + \ldots + x_{d}^{2})\) for some function
  \(q\), and since \(c\) needs to be a polynomial, so does \(q\). Then
  \begin{equation}
    D(P) = q(\Delta)\operatorname{grad}
  \end{equation}
  as claimed.
\end{proof}

If \(d = 1\), the lemma also holds and the proof goes through -- this is just
the scalar case from the previous section, which is generalized here.
But for \(G = \SO(2)\), this argument does not work because \(SO(1) = \{1\}\),
so the decisive step in the proof of the lemma fails. That is what allows the
additional equivariant PDOs.

\subsection{Vector to scalar PDOs}
Equivariant PDOs mapping from vector to scalar fields are simply the transpose
of those mapping from scalar to vector fields (for orthogonal groups \(G\)):
\(P\) is now a \(1 \times d\) matrix and the steerability constraint is
\begin{equation}
  P(gx) = P(x)g^{-1} = P(x)g^T\,.
\end{equation}
By transposing, we get
\begin{equation}
  P^T(gx) = g P^T(x)\,,
\end{equation}
which is the equivariance condition for scalar to vector
layers. As solutions, we get the divergence (as the transpose of the gradient)
and for \(G = \SO(2)\) also the 2D curl. They are again combined linearly with
polynomials in the Laplacian as coefficients.

\subsection{Vector to vector PDOs}
For PDOs mapping between two vector fields, the steerability constraint is
\begin{equation}
  P(gx) = gP(x)g^{-1}\,.
\end{equation}
Since the solution space in this case is somewhat more complicated, we will only cover
\(G = \SO(2)\) in these examples; see \cref{sec:solution_tables} for more solutions.
In principle, we could apply the same method that we already used before for
\(\SO(2)\), choosing the radial components of \(P\) freely and using the
steerability constraint to determine the angular components. But since the
computations in this case are more involved and don't yield much additional
insight, we will instead use the solutions from \cref{sec:solution_tables}.
Vector fields correspond to frequency 1 irreps, and writing out the solutions
for those explicitly, we get that the equivariant PDOs are exactly linear
combinations of
\begin{equation}
  \begin{pmatrix}
    1 & 0 \\
    0 & 1
  \end{pmatrix},\quad
  \begin{pmatrix}
    0 & -1\\
    1 & 0
  \end{pmatrix},\quad
  \begin{pmatrix}
    \partial_{x}^{2} - \partial_{y}^{2} & 2\partial_{x}\partial_{y} \\
    2\partial_{x}\partial_{y} & \partial_{y}^{2} - \partial_{x}^{2}
  \end{pmatrix},\quad
  \begin{pmatrix}
    -2\partial_{x}\partial_{y} & \partial_{x}^{2} - \partial_{y}^{2} \\
    \partial_{x}^{2} - \partial_{y}^{2} & 2\partial_{x}\partial_{y}
  \end{pmatrix}
\end{equation}
with polynomials in the Laplacian as coefficients.

The first two operators are simply the identity and a \(\frac{\pi}{2}\) rotation
matrix, both zeroth order PDOs. Note that the rotation matrix rotates the
fibers of the vector field, it does not act on the base space. The other two
operators are less interpretable, but we can replace them through a change of
basis:
\begin{equation}
  \frac{1}{2}\Delta
  \begin{pmatrix}
    1 & 0 \\
    0 & 1
  \end{pmatrix}
  + \frac{1}{2}
  \begin{pmatrix}
    \partial_{x}^{2} - \partial_{y}^{2} & 2\partial_{x}\partial_{y} \\
    2\partial_{x}\partial_{y} & \partial_{y}^{2} - \partial_{x}^{2}
  \end{pmatrix}
  =
  \begin{pmatrix}
    \partial_{x}^{2} & \partial_{x}\partial_{y}\\
    \partial_{x}\partial_{y} & \partial_{y}^{2}
  \end{pmatrix}\,,
\end{equation}
which is the matrix describing the composition
\(\operatorname{grad}{} \circ \operatorname{div}\).
Similarly,
\begin{equation}
  -\frac{1}{2}\Delta
  \begin{pmatrix}
    0 & -1 \\
    1 & 0
  \end{pmatrix}
  + \frac{1}{2}
  \begin{pmatrix}
    -2\partial_{x}\partial_{y} & \partial_{x}^{2} - \partial_{y}^{2} \\
    \partial_{x}^{2} - \partial_{y}^{2} & 2\partial_{x}\partial_{y}
  \end{pmatrix}
  =
  \begin{pmatrix}
    -\partial_{x}\partial_{y} & \partial_{x}^{2} \\
    -\partial_{y}^{2} & \partial_{x}\partial_{y}
  \end{pmatrix}\,,
\end{equation}
which is the matrix describing \(\operatorname{grad}{} \circ \operatorname{curl_{2D}}\).
So if we write \(R\) for the \(\frac{\pi}{2}\) rotation matrix (interpreted as a
PDO), then the equivariant PDOs mapping between vector fields are exactly linear
combinations of
\begin{equation}
  \operatorname{id},\quad R,\quad \operatorname{grad}{} \circ \operatorname{div},\quad \operatorname{grad}{} \circ \operatorname{curl_{2D}}\,,
\end{equation}
as always with polynomials in the Laplacian as coefficients.

We can be even more economical and describe these PDOs with fewer building
blocks. For two vectors \(P, Q\) of one-dimensional PDOs, e.g.\
\(P = \partial = (\partial_{1}, \partial_{2})^{T}\), we write \(P \otimes Q\) for the \(2 \times 2\) PDO with
entries \((P \otimes Q)_{ij} = P_{i}Q_{j}\). Then we can write for example
\(\operatorname{grad}{} \circ \operatorname{div}{} = \partial \otimes \partial\). We furthermore note that
\begin{equation}
  R\partial = \mat{0 & -1 \\ 1 & \phantom{-}0}\mat{\partial_{x} \\ \partial_{y}} = \mat{-\partial_{y} \\ \phantom{-}\partial_{x}} = \operatorname{curl_{2D}}\,.
\end{equation}
This means we can write the basis from above as
\begin{equation}
  \operatorname{id},\quad R,\quad \partial \otimes \partial,\quad \partial \otimes R\partial\,.
\end{equation}

\subsection{Summary of results}
\begin{itemize}
  \item \(\SO(d)\)- or \(\orth(d)\)-equivariant PDOs between two scalar fields
  are exactly polynomials in the Laplacian.
  \item \(G\)-equivariant PDOs mapping from scalar to vector fields are exactly those
  of the form \(q(\Delta)\operatorname{grad}\) for \(G = \SO(d)\) with \(d > 2\) or
  \(G = \orth(d)\). For PDOs from vector to scalar fields, we similarly get
  \(q(\Delta)\operatorname{div}\).
  \item \(\SO(2)\)-equivariant PDOs from vector to scalar fields are exactly
  those of the form
  \begin{equation}
    q_{1}(\Delta)\operatorname{div}{} + q_{2}(\Delta)\operatorname{curl_{2D}}\,.
  \end{equation}
  For PDOs from scalar to vector fields, we get \(\operatorname{grad}\) instead
  of \(\operatorname{div}\) and the transpose of the 2D curl instead.
  \item \(\SO(2)\)-equivariant PDOs between two vector fields are linear
  combinations of the identity, a \(\frac{\pi}{2}\) rotation,
  \(\operatorname{grad}{} \circ \operatorname{div}\) and
  \(\operatorname{grad}{} \circ \operatorname{curl_{2D}}\), with polynomials in the
  Laplacian as coefficients.
\end{itemize}

\section{Representation theory primer}\label{sec:representation_theory}
This section introduces the fundamental definitions of representation theory
that we need.

\subsection{Basic definitions}
\begin{definition}
  A \emph{group representation} of a group \(G\), or \emph{representation} for
  short, is a group homomorphism \(\rho: G \to \GL(V)\) for some vector space \(V\).
  This means that \(\rho(gg') = \rho(g)\rho(g')\) for all \(g, g'\), so multiplication of
  group elements in \(G\) is represented as matrix multiplication in \(\GL(V)\).
\end{definition}
For this paper, we only need \(V = \R^{c}\), so we focus on this case. In
particular, some of the following definitions make use of the fact that we
only consider finite-dimensional representations.

Given multiple representations of the same group, we can \enquote{stack them
  together} using direct sums:
\begin{definition}
  Let \(\rho_{1}: G \to \GL(\R^{c_{1}})\) and \(\rho_{2}: G \to \GL(\R^{c_{2}})\) be two
  representations of \(G\). Then we define the \emph{direct sum} representation
  \(\rho_{1} \oplus \rho_{2}: G \to \GL(\R^{c_{1} + c_{2}})\) by
  \begin{equation}
    (\rho_{1} \oplus \rho_{2})(g) := \mat{\rho_{1}(g) & \\ & \rho_{2}(g)}\,,
  \end{equation}
  which acts independently on the subspaces \(\R^{c_{1}}\) and \(\R^{c_{2}}\) of
  \(\R^{c_{1} + c_{2}}\).
\end{definition}
This corresponds exactly to stacking feature fields of different types. For
example, we can stack a vector and a scalar field into one four-dimensional
field, such that its vector and scalar part transform \emph{independently}. The
representation of this four-dimensional field will be the direct sum of the
vector and scalar field representations.

Often, two representations are formally different but can be transformed into
one another using a change of basis; they then behave the same in all
relevant aspects. For example, \(\SO(2)\) has a representation
\begin{equation}
  \rho(\phi) := \mat{\cos \phi & -\sin \phi \\ \sin \phi & \phantom{-}\cos \phi}
\end{equation}
that represents a rotation angle \(\phi\) by a counterclockwise rotation matrix
(this representation is the one used for vector fields). However, we could just
as well use
\begin{equation}
  \rho(\phi) := \mat{\phantom{-}\cos \phi & \sin \phi \\ -\sin \phi & \cos \phi}\,,
\end{equation}
where the rotation is clockwise. Using one or the other is pure convention and
we would like to treat them as \enquote{the same} representation. This is
formalized as follows:
\begin{definition}
  Two representations \(\rho_{1}, \rho_{2}: G \to \GL(\R^{c})\) are \emph{equivalent} if
  there is a matrix \(Q \in \GL(\R^{c})\) such that
  \begin{equation}
    \rho_{2}(g) = Q^{-1} \rho_{1}(g) Q
  \end{equation}
  for all \(g \in G\).
\end{definition}
Intuitively, \(\rho_{1}\) and \(\rho_{2}\) differ only by a change of basis, which is
given by \(Q\).

\subsection{Decomposition into irreducible representations}
We can now discuss irreducible representations, which play an important role for
solving the kernel and PDO steerability constraints.
\begin{definition}
  A linear subspace \(W \subseteq \R^{c}\) is called \emph{invariant} under a
  representation \(\rho: G \to \GL(\R^{c})\) if \(\rho(g)w \in W\) for all \(g \in G\) and
  \(w \in W\). In this case, we can define the restriction \(\rho_{|W}: G \to \GL(W)\),
  called a \emph{subrepresentation} of \(\rho\).
\end{definition}
\begin{definition}
  A representation \(\rho\) is called \emph{irreducible} if all its
  subrepresentations are either \(\rho\) itself or
  representations \(G \to \GL(\{0\})\), where \(\{0\}\) is the trivial vector space.
\end{definition}
For example, \(\rho_{1}\) and \(\rho_{2}\) are both subrepresentations of
\(\rho_{1} \oplus \rho_{2}\), so direct sums are never irreducible. A natural question is
the converse: if a representation is \emph{not} equivalent to a direct sum, does
that mean that it is irreducible? In other words, can all representations be
split into a direct sum of irreducible ones? In general, this is false, but it
holds for the cases that interest us:
\begin{definition}
  A \emph{topological group} is a group \(G\) equipped with a topology, such that
  the group multiplication \(G \times G \to G\) and the inverse map \(G \to G\) are
  continuous with respect to that topology. A \emph{compact group} is a
  topological group that is compact as a topological space.
\end{definition}
\begin{theorem}
  Let \(G\) be a compact group. Then every finite-dimensional representation of
  \(G\) is equivalent to a direct sum of irreducible representations.
\end{theorem}
Since we only consider compact subgroups of \(\orth(2)\) and \(\orth(3)\), this theorem
applies to all the cases we solve.

So we can always write
\begin{equation}
  \begin{split}
    \rho_{\text{in}} &= Q_{\text{in}}^{-1}\bigoplus_{i \in I_{\text{in}}} \psi_{i} \;Q_{\text{in}}\\
    \rho_{\text{out}} &= Q_{\text{out}}^{-1}\bigoplus_{i \in I_{\text{out}}} \psi_{i} \;Q_{\text{out}}
  \end{split}
\end{equation}
where the \(\psi_{i}\) are irreducible representations. It is then easy to
show~\citep{weiler2019} that a kernel \(k\) solves the \(G\)-steerability
constraint
\begin{equation}
  k(gx) = \rho_{\text{out}}(g)^{-1}\,k(x)\,\rho_{\text{in}}(g)^{-1}
\end{equation}
if and only if \(\kappa := Q_{\text{out}}\,k \,Q_{\text{in}}^{-1}\) solves a block-wise
steerability constraint between irreducible representations. Concretely,
\begin{equation}\label{eq:block_constraint}
  \kappa^{ij}(gx) = \psi_{i}(g)^{-1}\kappa^{ij}(x)\psi_{j}(g)^{-1}\qquad \forall i, j
\end{equation}
where \(\kappa^{ij}(x)\) is the submatrix of \(\kappa(x)\) that belongs to \(\psi_{i}\) and
\(\psi_{j}\).

The approach to solving the steerability constraint is thus to solve
\cref{eq:block_constraint} for arbitrary irreducible representations \(\psi_{i}\)
and \(\psi_{j}\). For general (not necessarily irreducible) \(\rho_{\text{out}}\) and
\(\rho_{\text{in}}\), we then first find the decompositions into irreducible
representations. Each basis element \(\kappa^{ij}\) of the solution to
\cref{eq:block_constraint} is then padded with zeros to the right size and
finally transformed via \(k = Q_{\text{out}}^{-1}\kappa Q_{\text{in}}\) to get the
final basis elements. See \citet{weiler2019} for a more detailed discussion and
visualization.

Clearly, this procedure works just as well for PDOs as it does for kernels: for
PDOs, we need to find the restriction of the kernel solution space to
polynomials, and it does not matter whether we restrict on the level of
irreducible representations and then combine them, or first combine irreducible
representations and then restrict.

\subsection{Specific representations}
We now define all the types of representations that occur in the main paper:
\begin{itemize}
  \item As already mentioned in the paper, scalar fields are described by the
  \emph{trivial representation} \(\rho(g) := 1\) and vector fields by the
  representation \(\rho(g) := g\) (for \(G \leq \GL(\R^{d})\)).
  \item For a finite group \(G\), the \emph{regular representation} is
  \(\rho: G \to \GL(\R^{\abs{G}})\), defined by
  \begin{equation}
    \rho(g)e_{g'} := e_{gg'}\,,
  \end{equation}
  where \((e_{g})_{g \in G}\) is the canonical basis of \(\R^{\abs{G}}\) (for some
  ordering of \(G\)). So this representation associates one basis vector to each
  group element and then acts by permuting these basis vectors. \(\rho(g)\) is thus
  always a permutation matrix.
  \item \emph{Quotient representations} are a generalization of regular representations,
  defined as follows: let \(G\) be a finite group and \(H \leq G\) a subgroup. Then we define
  the quotient representation \(\rho_{\text{quot}}^{G/H}: G \to \GL(\R^{\abs{G} / \abs{H}})\)
  by
  \begin{equation}
    \rho_{\text{quot}}^{G/H}(g)e_{g'H} := e_{gg'H}\,,
  \end{equation}
  where we now use a basis indexed by the cosets \(gH \in G/H\) for \(g \in G\). For \(H = \{e\}\),
  we recover regular representations.
  Appendix C by \citet{weiler2019} provides some intuition for these quotient representations
  in the case where \(G\) and \(H\) are both cyclic groups \(C_N\) and \(C_M\).
\end{itemize}

\section{Intuition for the group action on polynomials}\label{sec:polynomial_intuition}
Our work makes heavy use of terms of the form \(p(gx)\), where \(p\) is a
polynomial \(p: \R^d \to \R\), \(g \in G\) is a group element, and \(x \in
\R^d\).  We would now like to provide a bit more intuition for this action of
the group \(G\) on polynomials. Note that we will only cover
\emph{scalar-valued} polynomials in this appendix, i.e.\ using trivial
representations. The general case is straight-forward: the group acts on each
polynomial in a matrix of polynomials the way we describe here, while
\(\rho_{\text{in}}(g)\) and \(\rho_{\text{out}(g)}\) act via matrix
multiplication.

To prevent confusion, let us reiterate that we can think of polynomials in two different ways.
The first is as a formal expression, where \(x\) is a placeholder for things to be plugged in.
This is the approach we take when connecting polynomials to PDOs, where we plug in
differential operators for \(x\). The second is as a specific type
of function on \(\R^d\)---in which case \(x \in \R^d\) is simply
the argument of that function.

This second perspective is the one in which the group
action on polynomials is easiest to understand. Specifically, the action of \(g\)
on \(x\) is simply matrix multiplication. Similar to how \(x \mapsto p(x)\) defines
a function on \(\R^d\), \(x \mapsto p(gx)\) defines a different function on \(\R^d\).
We could think of it
as composing the function \(p\) with the group action of \(g\) on \(\R^d\).

Crucially, this new function is still a polynomial in the components of
\(x\), just with different coefficients. For purposes of illustration,
consider a very simple example with \(d = 2\) and \(G = \SO(2)\).
We will use the polynomial \(p(x) = x_2^2 + x_1\) (this is just meant to be
a simple non-trivial polynomial, it does not have special equivariance properties).

We can represent \(g\) as a rotation matrix parameterized by an angle \(\theta\).
The action on \(x\) is then given by
\begin{equation}
  gx = \mat{\cos\theta & -\sin\theta \\ \sin\theta & \phantom{-}\cos\theta} \mat{x_1 \\ x_2}
  = \mat{x_1\cos\theta - x_2\sin\theta \\ x_1\sin\theta + x_2\cos\theta}\,.
\end{equation}
The left-hand side is what we now plug into our polynomial, which yields
\begin{equation}
  p(gx) = (gx)_2^2 + (gx)_1 = (x_1\sin\theta + x_2\cos\theta)^2 + (x_1\cos\theta - x_2\sin\theta)\,.
\end{equation}

We can expand this and collect the coefficients for each power of \(x_1\) and \(x_2\), which yields
\begin{equation}
  p(gx) = \sin^2(\theta) x_1^2 + 2\sin(\theta)\cos(\theta) x_1 x_2 + \cos^2(\theta) x_2^2 + \cos(\theta) x_1 - \sin(\theta) x_2\,.
\end{equation}

If desired, we can now switch to the first perspective, and interpret this polynomial
a formal expression defined by its coefficients. In that perspective, \(g\) acts on
\(p\) by modifying its coefficients, rather than by composition. Computing the new
coefficients is straightforward analytically (given a matrix representation of \(g\));
we just gave a simple example for a polynomial of order two. The general case proceeds
along the same lines, it just requires using the binomial theorem (or the multinomial
theorem for \(d > 2\)).

\section{Background on distributions}\label{sec:distributions}
In \cref{sec:distribution_equivariance}, we will describe our framework for
equivariant maps using convolutions with Schwartz distributions. To
facilitate that, we now give the necessary background on Schwartz distributions.
We restrict ourselves to what is absolutely necessary for our purposes; for a
much more thorough introduction and for proofs, see e.g.\ \citet{treves1967}.

\subsection{Basic definitions}
\begin{definition}
  For \(U \subset \R^d\) open, \(\mathcal{D}(U) := C_c^\infty(U)\) is called the space of
  \emph{test functions} on \(U\) (\(C_{c}^{\infty}(U)\) is the space of compactly
  supported smooth functions \(U \to \R\)). A sequence \((\phi_n)\) in \(\mathcal{D}(U)\) is
  defined to converge to 0 iff
  \begin{enumerate}[(i)]
    \item there is a compact subset \(K \subset U\) such that the support of each
    \(\phi_n\) is contained in \(K\) and
    \item \(\partial^\alpha \phi_n \to 0\) uniformly for all multi-indices \(\alpha\).
  \end{enumerate}
\end{definition}
One can define the so-called \emph{canonical LF topology} on \(\mathcal{D}(U)\).
This topology induces the notion of convergence just given. However,
constructing this topology explicitly is unnecessary for our purposes since
knowing when sequences converge will be enough.

\begin{definition}
  A linear functional \(T: \mathcal{D}(U) \to \R\) is defined as continuous if for
  every sequence \(\phi_n\) that converges to 0 in \(\mathcal{D}(U)\), \(T\phi_n \to 0\)
  (in \(\R\)).

  Such a continuous functional is called a \emph{distribution} on \(U\). The
  space of all distributions on \(U\) is written as \(\mathcal{D}'(U)\).
\end{definition}
This notion of continuity is also induced by the canonical LF topology and then
\(\mathcal{D}'(U)\) is the topological dual of \(\mathcal{D}(U)\) as the
notation suggests.

So intuitively, a distribution is a \enquote{reasonable} way of linearly
assigning a number in \(\R\) to each test function in \(C_c^\infty\).

A function \(f: U \to \R\) induces a distribution \(T_f \in \mathcal{D}'(U)\),
defined by
\begin{equation}
  T_f\phi := \int_{U} f\phi d\lambda^d\,.
\end{equation}
Since \(\phi\) has compact support, we don't need many
restrictions on \(f\) (for example, any locally integrable function
\(f \in L_{1, \text{loc}}(U)\) works).

We will also use the \emph{duality pairing}
\begin{equation}
  \langle T, \phi \rangle := T(\phi)
\end{equation}
and under slight abuse of notation also write
\begin{equation}
  \langle f, \phi \rangle := T_f(\phi) = \int_U f\phi d\lambda^d\,.
\end{equation}
Note that this coincides with the inner product on \(L_2(U)\), but it allows
functions \(f\) that are not in \(L_2\) while in exchange requiring \(\phi\) to
have compact support.

\subsection{Convolutions}
We define the translation \(\tau_x: \R^d \to \R^d\) by
\(\tau_x(y) = y + x\) and set \(\tau_x f := f \circ \tau_{-x}\) for functions \(f\) on \(\R^{d}\). This is exactly the same as
the action of \((\R^{d}, +)\) we used in the main paper, just with more explicit
notation. Furthermore, we write \(\check{f}(x) := f(-x)\).

Then we can define the convolution between a distribution and a test function: for
\(f \in \mathcal{D}(U)\) and \(T \in \mathcal{D}'(U)\), the convolution
\(T * f \in C^\infty(U)\) is defined by
\begin{equation}
  (T * f)(x) := T(\tau_{x} \check{f})\,.
\end{equation}
Explicitly, \((\tau_{x}\check{f})(y) = f(x - y)\).
This immediately shows that this notion of convolution extends the classical
one, i.e.
\begin{equation}
  T_g * f = g * f\,.
\end{equation}

\subsection{Derivatives}
\begin{definition}
  For \(T \in \mathcal{D}'(U)\), we define the distributional derivative
  \(\partial^\alpha T\) as the unique distribution on \(U\) for which
  \begin{equation}
    \langle \partial^\alpha T, \phi\rangle = (-1)^{\abs{\alpha}}\langle T, \partial^\alpha \phi \rangle\,.
  \end{equation}
\end{definition}
The distributional derivative of all orders always exists, i.e.\
\enquote{distributions are infinitely differentiable}, but of course only in
this distributional sense. At least for \(f \in C^\infty(U)\) (but also under much
weaker assumptions), this definition also extends the definition of derivatives
of functions, in the sense that \(T_{\partial^\alpha f} = \partial^\alpha T_f\).

\subsection{Composition with diffeomorphisms}
Let \(F: U \to U\) be a diffeomorphism and \(T \in \mathcal{D}'(U)\). Then we
define the composition \(T \circ F \in \mathcal{D}'(U)\) as
\begin{equation}
  \langle T \circ F, \phi\rangle := \langle T, \abs{\det DF^{-1}} \phi \circ F^{-1}\rangle
\end{equation}
where \(\det DF^{-1}\) is the inverse Jacobian.

As in the previous constructions, this extends the definition for classical functions in the sense that
\begin{equation}
  T_{f \circ F} = T_f \circ F\,.
\end{equation}
This follows immediately from the transformation theorem for integrals:
\begin{align}
  \pair{T_{f \circ F}, \phi} &= \int_U (f \circ F) \cdot \phi \diff\lambda\\
                      &= \int_U f \cdot (\phi \circ F^{-1}) \abs{\det DF^{-1}}\diff\lambda\\
                      &= \pair{T_{f}, (\phi \circ F^{-1}) \abs{\det DF^{-1}}}\\
                      &= \pair{T_{f} \circ F, \phi}\,.
\end{align}
Furthermore, this type of composition is associative:
\begin{align}
  \pair{(T \circ F) \circ G, \phi} &= \pair{T \circ F, \abs{\det DG^{-1}}\phi \circ G^{-1}}\\
                        &= \pair{T, \abs{\det DF^{-1}}\abs{\det DG^{-1} \circ F^{-1}}\phi \circ G^{-1} \circ F^{-1}}\\
                        &= \pair{T, \abs{\det (DG^{-1} \circ F^{-1})DF^{-1}}\phi \circ (F\circ G)^{-1}}\\
                        &= \pair{T, \abs{\det D(F \circ G)^{-1}}\phi \circ (F\circ G)^{-1}}\\
                        &= \pair{T \circ (F \circ G), \phi}\,.
\end{align}

We are particularly interested in the case where \(F\) is a linear
transformation, i.e.\ \(F \in \GL(\R^{d})\). Then we have
\begin{equation}
  \langle T \circ F, \phi\rangle = \abs{\det F^{-1}} \langle T, \phi \circ F^{-1}\rangle
\end{equation}
(note that in this case, we can pull out the determinant because
it is just a constant).
For translations, we get
\begin{equation}
  \langle T \circ \tau_x, \phi\rangle = \langle T, \phi \circ \tau_{-x}\rangle = \langle T, \tau_x \phi\rangle\,.
\end{equation}

\subsection{The Dirac delta distribution}
A very simple but important distribution is the following:
\begin{definition}
  For any \(x \in \R^d\), the Dirac delta distribution
  \(\delta_{x} \in \mathcal{D}(\R^d)\) is defined by
  \begin{equation}
    \delta_{x}[\phi] := \phi(x)\,.
  \end{equation}
\end{definition}
It is clear from the definition that the distributional derivatives of the delta
distribution are given by
\begin{equation}
  (\partial^\alpha \delta_{x})[\phi] = (-1)^{\abs{\alpha}}\partial^\alpha\phi(x)\,.
\end{equation}

\subsection{Convergence of distributions}
\begin{definition}
  We say that a sequence \(T_{n}\) of distributions converges to \(T\),
  written \(T_{n} \to T\), if it converges pointwise, i.e.
  \begin{equation}
    \pair{T_{n}, \phi} \to \pair{T, \phi}\quad \text{for } n \to \infty
  \end{equation}
  for all \(\phi \in \mathcal{D}(U)\).
\end{definition}

Abusing notation a bit, we also write \(f_{n} \to T\) for functions \(f_{n}\) if
\(T_{f_{n}} \to T\).

One important type of function sequences are \emph{Dirac sequences}, which
converge to the Delta distribution:
\begin{lemma}\label{thm:dirac_sequence}
  Let \(f_{n}\) be a sequence in \(L^{1}(\R^{d})\) such that
  \begin{enumerate}[(i)]
    \item \(f_{n} \geq 0\),
    \item \(\norm{f_{n}}_{1} = 1\), and
    \item \(\int_{\R^d\setminus B_{\eps}(0)}f_n d\lambda^d \to 0\) for all \(\eps > 0\).
  \end{enumerate}
  Then \(f_{n} \to \delta_{0}\) in the sense of distributions.
\end{lemma}
We also note that convergence plays together nicely with derivatives and with convolutions:
\begin{lemma}\label{thm:convergence_derivative}
  If \(T_{n} \to T\), then \(\partial^{\alpha}T_{n} \to \partial^{\alpha}T\) for all \(\alpha\).
\end{lemma}
\begin{proof}
  \begin{equation}
    \begin{split}
      \pair{\partial^{\alpha}T_{n}, \phi} &= (-1)^{\abs{\alpha}}\pair{T_{n}, \partial^{\alpha}\phi}\\
      &\to (-1)^{\abs{\alpha}}\pair{T, \partial^{\alpha}\phi}\\
      &= \pair{\partial^{\alpha}T, \phi}\,.
    \end{split}
\end{equation}
\end{proof}
\begin{lemma}\label{thm:convergence_convolution}
  If \(T_{n} \to T\), then \(T_{n} * f \to T * f\) pointwise for all \(f \in \mc{D}(U)\).
\end{lemma}
\begin{proof}
  \begin{equation}
    \begin{split}
      (T_{n} * f)(x) &= T_{n}(\tau_{x}\check{f}) \\
      &\to T(\tau_{x}\check{f}) \\
      &= (T * f)(x)\,.
    \end{split}
\end{equation}
\end{proof}

\subsection{Tempered distributions and the Fourier transform}
The \emph{Schwartz space} \(\mathcal{S}(\R^{d})\) is the space of functions for
which all derivatives decay very quickly as \(\abs{x} \to \infty\). More precisely,
a smooth function \(f: \R^{d} \to \R\) is in \(\mathcal{S}(\R^{d})\) iff
\begin{equation}
  \sup_{x \in \R^{d}} \norm{x^{\beta}\partial^{\alpha}f(x)} < \infty\,.
\end{equation}
Intuitively, all derivatives of \(f\) must decay more quickly than any
polynomial. Examples are compactly supported functions or Gaussians.

Its dual space \(\mathcal{S}'(\R^{d})\) is
called the space of \emph{tempered distributions} and can be continuously
embedded into \(\mathcal{D}'(\R^{d})\). Tempered distributions are still
very general; in particular the delta distribution, distributions induced by
functions, and derivatives of tempered distributions are all tempered.

The reason we're interested in tempered distributions is that there is a Fourier
transform defined by
\begin{equation}
  \pair{\fourier{T}, \phi} := \pair{T, \fourier{\phi}}
\end{equation}
for tempered distributions \(T\). This is an automorphism on \(\mathcal{S}'(\R^{d})\).

We use the following convention for the Fourier transform on functions:
\begin{equation}
  \fourier{f}(\xi) := (2\pi)^{-d/2}\int_{\R^d} f(x) \exp(-i x \cdot \xi) \diff x\,,
\end{equation}
which means the inverse is given by
\begin{equation}
  \invfourier{f}(x) := (2\pi)^{-d/2}\int_{\R^d} f(x) \exp(i x \cdot \xi) \diff\xi\,.
\end{equation}

We will later need the Fourier transform of derivatives of the Dirac delta distribution:
\begin{equation}
  \begin{split}
    \pair{\fourier{\partial^{\alpha}\delta_{0}}, \phi} &= \pair{\partial^{\alpha}\delta_{0}, \fourier{\phi}}\\
    &= \left.(-1)^{\abs{\alpha}}\partial^{\alpha}\fourier{\phi}\right|_{0}\\
    &= \left.(-i)^{\abs{\alpha}}\fourier{x^{\alpha}\phi}\right|_{0}\\
    &= (-i)^{\abs{\alpha}}(2\pi)^{-d/2}\int x^{\alpha}\phi(x)\exp(-ix \cdot 0)dx\\
    &= (-i)^{\abs{\alpha}}(2\pi)^{-d/2}\int x^{\alpha}\phi(x)dx\,,
  \end{split}
\end{equation}
so \(\fourier{\partial^{\alpha}\delta_{0}} \propto x^{\alpha}\)
(or more precisely, the distribution induced by the function \(x \mapsto x^{\alpha}\)).

\subsection{Matrices of distributions}
Analogously to how we defined matrices of PDOs or of polynomials, we will need
matrices of distributions. We will write \(\mathcal{D}'(U, \R^{\cout \times \cin})\)
for the space of \(\cout \times \cin\) dimensional matrices with entries in
\(\mathcal{D}'(U)\). We get a pairing
\begin{equation}
  \langle \cdot, \cdot \rangle: \mathcal{D}'(U, \R^{\cout \times \cin}) \times \mathcal{D}(U, \R^\cin) \to \R^{\cout}
\end{equation}
defined by
\begin{equation}
  \langle T, f\rangle_i := \sum_j \langle T_{ij}, f_j \rangle
\end{equation}
The convolution of \(T \in \mathcal{D}'(U, \R^{\cout \times \cin})\) and
\(f \in \mathcal{D}(U, \R^\cin)\) is defined analogously to the scalar case, i.e.\
\((T * f)(x) := \pair{ T, \tau_x \check{f}}\), where we now use the more general duality
pairing just defined.

We can also define the multiplication of matrices of distributions by
matrices over \(\R\). For a matrix \(A \in \R^{\cout \times \cout}\)
and a matrix of distributions \(T \in \mathcal{D}'(\R^d, \R^{\cout \times \cin})\),
we write
\begin{equation}
  (AT)_{ij} := \sum_l A_{il} T_{lj} \in \mathcal{D}'(\R^d, \R^{\cout \times \cout})
\end{equation}
(multiplying a distribution by a scalar obviously defines another distribution).
Analogously, we define \(TB\) for \(B \in \R^{\cin \times \cin}\). It immediately
follows that
\begin{equation}
  A\pair{T, B\phi} = \pair{ATB, \phi}
\end{equation}
holds.

Composition with diffeomorphisms of \(\R^{d}\) can simply be defined
component-wise for matrices of distributions. This commutes with multiplication
by constant matrices, meaning that
\begin{equation}
  ATB \circ F = A(T \circ F)B\,.
\end{equation}

Finally, we also define the Fourier transform component-wise, and this commutes
with multiplication by matrices in the same way.

\section{Distributional framework for equivariant maps}\label{sec:distribution_equivariance}
In this section, we develop two main results: first that all linear continuous
translation equivariant maps between feature spaces are convolutions with some
distribution, and then the equivariance constraint for such convolutions.
See \cref{sec:distributions} for background on distributions. We equip
\(\mathcal{D}(U)\) with the canonical LF topology throughout this section.

\subsection{Translation equivariant maps are convolutions with distributions}
We begin by showing that the framework using convolutions with Schwartz
distributions encompasses all translation equivariant continuous linear maps
between feature spaces. As preparation, we prove a simple Lemma on the
reflection map \(s\):
\begin{lemma}\label{thm:check_continuous}
  The map \(s: \mc{D}(U) \to \mc{D}(U)\) given by \(s(f) := \check{f}\) is linear
  and continuous.
\end{lemma}
\begin{proof}
  Linearity is clear:
  \begin{equation}
    s(af + bg)(x) = af(-x) + bg(-x) = as(f)(x) + bs(g)(x)\,.
  \end{equation}
  For continuity, we use the fact that a linear map from \(\mc{D}(U)\) to itself is
  continuous if and only if it is sequentially
  continuous~\citep[][Proposition~14.7]{treves1967}. So take any sequence
  \(f_{n} \to 0\) in \(\mc{D}(U)\). This means that
  \begin{enumerate}[(i)]
    \item there is a compact subset \(K \subset U\) such that the support of each
    \(f_n\) is contained in \(K\)
    \item \(\partial^\alpha f_n \to 0\) uniformly for all multi-indices \(\alpha\)
  \end{enumerate}
  We set \(-K := \setcomp{-x}{x \in K}\), then the support of \(s(f_{n})\) (which
  is just the mirror of the support of \(f_{n}\)) is contained in \(-K\), and
  \(-K\) is compact. Additionally,
  \begin{equation}
    \partial^{\alpha}s(f_{n}) = (-1)^{\alpha}s\left(\partial^{\alpha}f_{n}\right)\,.
  \end{equation}
  Since \(\partial^{\alpha}f_{n} \to 0\) uniformly, the same is true for
  \(s\left(\partial^{\alpha}f_{n}\right)\). It follows that \(s\) is sequentially continuous, and
    thus continuous, which concludes the proof.
\end{proof}

We first prove our main generality result for the special case of
one-dimensional fibers:
\begin{lemma}
  Let \(\Phi: \mathcal{D}(U) \to \mathcal{D}(U)\) be a translation equivariant
  continuous linear map. Then there is a distribution \(T \in \mc{D}'(U)\) such
  that \(\Phi(f) = T * f\).
\end{lemma}
\begin{proof}
  Let \(\Phi: \mc{D}(U) \to \mc{D}(U)\) be any continuous linear map, where
  continuity is understood with respect to the canonical LF topology. We then
  define
  \begin{equation}
    T: \mc{D}(U) \to \R,\quad T(f) := \Phi(\check{f})(0)\,.
  \end{equation}
  Equivalently, we can write \(T = \delta_{0} \circ \Phi \circ s\), understood as normal
  composition of functions. \(\delta_{0}\) and \(\Phi\) are linear and
  continuous, as is \(s\) by \cref{thm:check_continuous}. Therefore, \(T\) is
  also linear and continuous, and thus a Schwartz distribution \(T \in \mc{D}'(U)\).

  We will also need that
  \begin{equation}
    \begin{split}
      \left(\widecheck{\tau_{-x} f}\right)(y) &= (\tau_{-x}f)(-y)\\
      &= f(x - y) \\
      &= \check{f}(y - x) \\
      &= (\tau_{x}\check{f})(y)\,.
    \end{split}
  \end{equation}

  Now using the assumption that \(\Phi\) is translation equivariant, i.e.\
  \begin{equation}
    \tau_{x} \circ \Phi = \Phi \circ \tau_{x}\,,
  \end{equation}
  it follows that convolution with \(T\) is given by
  \begin{equation}
    \begin{split}
      (T * f)(x) &= \pair{T, \tau_{x} \check{f}}\\
                 &= \pair{T, \widecheck{\tau_{-x}f}} \\
                 &= \Phi(\tau_{-x}f)(0) \\
                 &= \tau_{-x}\Phi(f)(0) \\
                 &= \Phi(f)(x)\,.
      \end{split}
    \end{equation}
  This shows that convolution with \(T\) is \(\Phi\), which concludes the proof.
\end{proof}
Finally, we generalize to multi-dimensional feature fields:
\begin{theorem}\label{thm:translation_equivariance_implies_convolution}
  Let \(\Phi: \mathcal{D}(U, \R^{\cin}) \to \mathcal{D}(U, \R^{\cout})\) be a
  translation equivariant continuous linear map. Then there is a distribution
  \(T \in \mc{D}'(U, \R^{\cout \times \cin})\) such that \(\Phi(f) = T * f\).
\end{theorem}
\begin{proof}
  We write \(e_{j}\) for the \(j\)-th canonical basis vector of \(\R^{\cin}\).
  For \(f \in \mc{D}(U, \R^{\cin})\), we write \(f_{j}\) for its \(j\)-th
  component. Then we have
  \begin{equation}
    \Phi(f) = \Phi\left(\sum_{j = 1}^{\cin}f_{j}e_{j}\right)
    = \sum_{j = 1}^{\cin}\Phi\left(f_{j}e_{j}\right)\,.
  \end{equation}
  We now define the components \(\Phi_{ij}: \mc{D}(U) \to \mc{D}(U)\) by
  \begin{equation}
    \Phi_{ij}(g) = \left(\Phi(g e_{j})\right)_{i}\,.
  \end{equation}
  Here, \(g e_{j} \in \mc{D}(U, \R^{\cin})\) is the map with \(g\) in its \(j\)-th
  output component and 0 in the other ones. We can then write
  \begin{equation}
    \Phi(f)_{i} = \sum_{j = 1}^{\cin} \Phi(f_{j}e_{j})_{i} = \sum_{j = 1}^{\cin} \Phi_{ij}(f_{j})\,.
  \end{equation}
  It is clear that the components \(\Phi_{ij}\) of \(\Phi\) are still
  translation equivariant, continuous and linear. Therefore, by the previous
  Lemma, there are distributions \(T_{ij} \in \mc{D}'(U)\) such that
  \(\Phi_{ij}(g) = T_{ij} * g\). We then get
  \begin{equation}
    \Phi(f)_{i} = \sum_{j = 1}^{\cin} T_{ij} * f_{j} = T * f
  \end{equation}
  where \(T \in \mc{D}'(U, \R^{\cout \times \cin})\) is the matrix of distributions
  with entries \(T_{ij}\).
\end{proof}

\subsection{Equivariance constraint for distributions}
We will now characterize the distributions
\(T \in \mathcal{D}'(\R^d, \R^{\cout \times \cin})\) for which the operator
\begin{equation}
  T_*: \mathcal{D}(\R^d, \R^\cin) \to C^\infty(\R^d, \R^\cout)
\end{equation}
given by \(f \mapsto T * f\) is \(H\)-equivariant. The codomain of \(T_{*}\) may
contain functions that are not compactly supported for some \(T\). We are mostly
interested in distributions \(T\) for which the codomain can be restricted to
\(\mc{D}(\R^{d}, \R^{\cout})\) but for the discussion of equivariance this does
not make any difference, so we keep the derivation general by not restricting
the distribution \(T\).

First, we can note that convolution with a distribution is
always translation equivariant:
\begin{proposition}
  The map \(f \mapsto T * f\) is translation equivariant for any distribution \(T\).
\end{proposition}
\begin{proof}
  The general equivariance condition is
  \begin{equation}
    (h \cdot (T * f))(x) = (T * (h \cdot f))(x)
  \end{equation}
  for all \(f \in \mathcal{D}(\R^d, \R^{\cin}), h \in H, x \in \R^d\). More explicitly, this
  means
  \begin{equation}\label{eq:explicit_general_equivariance}
    \rho_{\text{out}}(g)\langle T, \tau_{h^{-1}x} \check{f}\rangle = \langle T, \tau_{x} (\rho_{\text{in}}(g) \widecheck{f \circ h^{-1}})\rangle\,,
  \end{equation}
  where \(g\) is the linear component of \(h\).

  Let \(h = t \in \R^d\), i.e.\ a pure translation. The condition then becomes
  \begin{equation}
    \pair{T, \tau_{(x - t)} \check{f}} = \pair{T, \tau_{x}\widecheck{(f \circ \tau_{-t})}} = \pair{T, \tau_{x}\tau_{-t} \check{f}}\,,
  \end{equation}
  which always holds (since \(\tau_{(x - t)} = \tau_{x}\tau_{-t}\)).
  So \(T_*\) is translation equivariant by construction.
\end{proof}

Therefore, it suffices to consider pure linear transformations
\(g \in \GL(\R^{d})\). For those, we have the following result, which generalizes
the steerability constraint for PDOs and the one for kernels:
\begin{theorem}\label{thm:distribution_constraint}
  Let \(T \in \mc{D}'(\R^{d}, \R^{\cout \times \cin})\).
  Then the map \(f \mapsto T * f\) is \(G\)-equivariant if and only if
  \begin{equation}\label{eq:distribution_constraint}
    T \circ g = \abs{\det g}^{-1}\rho_{\text{out}}(g) T \rho_{\text{in}}(g)^{-1}\,.
  \end{equation}
\end{theorem}
As is shown in \cref{sec:distributions}, the definitions of convolution with
distributions and of composition with a diffeomorphism are compatible with those
for classical functions. So if \(T\) is a classical kernel, then the constraint
on \(T\) is given by the same equation, which already gives us the steerability
constraint for kernels.
\begin{proof}
  We start with the constraint \cref{eq:explicit_general_equivariance} from
  above with \(h = g \in G\) and will transform this into the desired form
  \cref{eq:distribution_constraint}. First, to simplify notation, we can remove
  all the reflections \(\check{\cdot}\) in \cref{eq:explicit_general_equivariance}.
  This is now possible because \(\widecheck{f \circ g^{-1}} = \check{f} \circ g^{-1}\), which was
  not true for translations. Furthermore, we can pull the translations to the
  other side of the duality pairing:
  \begin{equation}\label{eq:linear_equivariance_constraint}
    \begin{split}
      \rho_{\text{out}}(g)\pair{T \circ \tau_{g^{-1}x}, f} &\overset{!}{=} \pair{T \circ \tau_x, \rho_{\text{in}}(g)f \circ g^{-1}} \\
      &= \abs{\det g}\pair{T \circ \tau_x \circ g, \rho_{\text{in}}(g) f}\,.
    \end{split}
  \end{equation}
  Now we'll use
  \begin{equation}
    \tau_x \circ g = g \circ \tau_{g^{-1}x}\,,
  \end{equation}
  since
  \begin{equation}
    x + gy = g(g^{-1}x + y)\,.
  \end{equation}
  Plugging this into \cref{eq:linear_equivariance_constraint}, we get
  \begin{equation}
    \rho_{\text{out}}(g)\pair{T\circ\tau_{g^{-1}x}, f} \overset{!}{=} \abs{\det g}\pair{T \circ g \circ \tau_{g^{-1}x}, \rho_{\text{in}}(g) f}\,.
  \end{equation}
  We want to have only \(f\) on the right side of the duality pairing, so we use
  the notation we introduced for multiplying distributions by matrices and get
  \begin{equation}
    \pair{\rho_{\text{out}}(g)T \circ \tau_{g^{-1}x}, f} \overset{!}{=} \abs{\det g}\pair{T \circ g \circ \tau_{g^{-1}x} \rho_{\text{in}}(g), f}\,.
  \end{equation}
  This holds for all \(f\) iff we have equality of distributions,
  \begin{equation}
    \rho_{\text{out}}(g)T \circ \tau_{g^{-1}x} \overset{!}{=} \abs{\det g} T \circ g \circ \tau_{g^{-1}x} \rho_{\text{in}}(g)\,.
  \end{equation}
  Finally, multiplication with matrices and composition with diffeomorphisms commute, so
  we can cancel the \(\tau_{g^{-1}x}\). This means our final constraint on \(T\) is
  \begin{equation}
    \rho_{\text{out}}(g) T \overset{!}{=} \abs{\det g} (T \circ g) \rho_{\text{in}}(g)\,,
  \end{equation}
  which we can slightly rewrite as
  \begin{equation}
    T \circ g \overset{!}{=} \abs{\det g}^{-1}\rho_{\text{out}}(g) T \rho_{\text{in}}(g)^{-1}\,.
  \end{equation}
  This is exactly \cref{eq:distribution_constraint}. All steps of the derivation
  work equally well in the other direction, which proves the \enquote{if and
    only if}.
\end{proof}

\subsection{Differential operators as convolutions}\label{sec:distributions_diffops}
It is clear that convolutions with classical kernels are a special case of
convolutions with distributions (see \cref{sec:distributions}). But we have
claimed that convolutions with distributions also cover PDOs, which is what we
show now.

We have already seen in \cref{sec:distributions} that derivatives of the delta
distribution are closely related to PDOs. So we calculate the convolution with
such derivatives:
\begin{equation}
  \begin{split}
    ((\partial^{\alpha}\delta_{0}) * f)(x) &= \pair{\partial^{\alpha}\delta_{0}, \tau_{x}\check{f}}\\
    &= (-1)^{\abs{\alpha}}\partial^{\alpha}(\tau_{x}\check{f})(0)\\
    &= (-1)^{\abs{\alpha}}\tau_{x}(\partial^{\alpha}\check{f})(0)\\
    &= \tau_{x}\left(\widecheck{\partial^{\alpha}f}\right)(0) \\
    &= \left(\widecheck{\partial^{\alpha}f}\right)(-x) \\
    &= \partial^{\alpha}f(x)\,.
  \end{split}
\end{equation}
Noting that the map \(T \mapsto T_{*}\) is \(\R\)-linear, we see that
\(D(p)f = p(\partial)\delta_{0} * f\) for polynomials \(p\). It then also immediately
follows that this is true for matrices of polynomials \(P\) because
\begin{equation}
  \begin{split}
    \left(D(P)f\right)_{i} &= \sum_{j = 1}^{\cin}D(P_{ij})f_{j} \\
    &= \sum_{j = 1}^{\cin}P_{ij}(\partial)\delta_{0} * f_{j} \\
    &= P(\partial)\delta_{0} * f\,.
  \end{split}
\end{equation}
This shows that PDOs can be interpreted as convolutions with distributions, as claimed.

\subsection{The Fourier duality between kernels and PDOs}\label{sec:fourier_duality}
Interpreting PDOs as convolutions with distributions allows us to relate them to
classical convolutional kernels via a Fourier transform. We have already seen in
\cref{sec:distributions} that
\begin{equation}
  \fourier{\partial^{\alpha}\delta_{0}} = (2\pi)^{-d/2}(-i)^{\abs{\alpha}}x^{\alpha}\,,
\end{equation}
which by linearity of the Fourier transform immediately implies
\begin{equation}
  \fourier{P(\partial)\delta_{0}}_{ij} = (2\pi)^{-d/2}\sum_{\alpha}(-i)^{\abs{\alpha}}c^{ij}_{\alpha}x^{\alpha}\,.
\end{equation}
for \(P_{ij} = \sum_{\alpha}c^{ij}_{\alpha}x^{\alpha}\). So the Fourier transform of derivatives
of the delta distribution are polynomials, and of course vice versa via the
inverse Fourier transform. Since PDOs are convolutions with such delta
distribution derivatives, we can also interpret them as \emph{convolutions with
  the (inverse) Fourier transform of polynomials}. We will use this
interpretation to give a derivation of the PDO steerability constraint that
sheds some light on the similarities and differences to the kernel steerability constraint.

We begin by proving a basic fact about the Fourier transform of a composition of functions:
\begin{lemma}\label{thm:fourier_composition}
  Let \(g \in \GL(\R^{d})\) and \(\phi \in \mc{D}(U)\). Then
  \begin{equation}
    \fourier{\phi \circ g} = \abs{\det g^{-1}}\fourier{\phi} \circ g^{-T}\,,
  \end{equation}
  where we use the shorthand \(g^{-T} := \left(g^{-1}\right)^{T}\).
\end{lemma}
\begin{proof}
  First we apply the transformation theorem for integrals:
  \begin{equation}
    \begin{split}
      \fourier{\phi \circ g}(\xi) &= (2\pi)^{-d/2}\int \phi(gx)\exp(-i x \cdot \xi)dx\\
      &= \abs{\det g^{-1}}(2\pi)^{-d/2}\int \phi(x)\exp(-i g^{-1}x \cdot \xi) dx\,.
    \end{split}
  \end{equation}
  Now we just rewrite \(g^{-1}x \cdot \xi = x \cdot g^{-T}\xi\), which gives the desired result.
\end{proof}

This result holds more generally for tempered distributions (a subset
of distributions for which the Fourier transform can be defined, see \cref{sec:distributions}):
\begin{proposition}
  For a tempered distribution \(T\) and \(g \in \GL(\R^{d})\),
  \begin{equation}
    \fourier{T \circ g} = \abs{\det g^{-1}}\fourier{T} \circ g^{-T}\,.
  \end{equation}
\end{proposition}
\begin{proof}
  \begin{equation}
    \begin{split}
      \pair{\fourier{T \circ g}, \phi} &= \pair{T \circ g, \fourier{\phi}}\\
      &= \abs{\det g^{-1}}\pair{T, \fourier{\phi} \circ g^{-1}}\\
      &\overset{(1)}{=} \pair{T, \fourier{\phi \circ g^{T}}}\\
      &= \pair{\fourier{T}, \phi \circ g^{T}}\\
      &= \abs{\det g^{-1}}\pair{\fourier{T} \circ g^{-T}, \phi}\,,
    \end{split}
  \end{equation}
  where we used \cref{thm:fourier_composition} for (1).
\end{proof}

Now note that the equivariance constraint for distributions,
\cref{eq:distribution_constraint}, is equivalent to the constraint we get when
we take the Fourier transform on both sides. That's because the Fourier
transform is an automorphism on the space of tempered distributions. By applying
the result we just proved, we then get the constraint
\begin{equation}
  \abs{\det g^{-1}}\fourier{T} \circ g^{-T} = \abs{\det g^{-1}}\rho_{\text{out}}(g)\fourier{T}\rho_{\text{in}}(g)^{-1}\,.
\end{equation}
We can cancel the determinants, which gives the equivariance constraint in
Fourier space:
\begin{equation}\label{eq:fourier_constraint}
  \fourier{T} \circ g^{-T} = \rho_{\text{out}}(g)\fourier{T}\rho_{\text{in}}(g)^{-1}\,.
\end{equation}
Now let \(T = P(\partial)\delta_{0}\), so that \(T * f = D(P)f\)
for some matrix of polynomials \(P\). We've seen above that in this case
\(\fourier{T} = P \circ m_{-i}\) (up to a constant coefficient), where \(m_{-i}\) is
multiplication by the negative imaginary unit \(-i\). So \(D(P)\) is
equivariant iff
\begin{equation}
  P \circ m_{-i} \circ g^{-T} = \rho_{\text{out}}(g)P \circ m_{-i}\rho_{\text{in}}(g)^{-1}\,,
\end{equation}
but since \(m_{-i}\) is invertible and commutes with the other maps, we can cancel it
and get the equivariance condition
\begin{equation}
  P \circ g^{-T} = \rho_{\text{out}}(g)P\rho_{\text{in}}(g)^{-1}\,,
\end{equation}
which is precisely the PDO steerability constraint. The reason that it differs
slightly from the kernel steerability constraint can now be traced back to
\cref{thm:fourier_composition}. Intuitively speaking, since PDOs are in a sense
the Fourier transform of convolutional kernels, they transform differently under
\(\GL(\R^{d})\), which leads to superficial differences in the steerability constraints.
However, \emph{Fourier transforms commute with rotations and reflections} (i.e.\
transformations from \(\orth(d)\)), which is why for \(G \leq \orth(d)\), the two
steerability constraints coincide.

\subsection{PDOs as the infinitesimal limit of kernels}
Let \(\psi_{n}\) be any sequence of functions such that \(\psi_{n}\to \delta_{0}\), for
example a Dirac sequence. Then for a polynomial
\(p = \sum_{\alpha}c_{\alpha}x^{\alpha}\), we also have
\begin{equation}
  p(\partial)\psi_{n} \to p(\partial)\delta_{0}
\end{equation}
because of \cref{thm:convergence_derivative}. Then
\cref{thm:convergence_convolution} implies that
\begin{equation}
  p(\partial)\psi_{n} * f \to p(\partial)\delta_{0} * f = D(p)f\,,
\end{equation}
where the convergence is understood pointwise. Therefore, any sequence of
kernels that approximates the delta distribution (by becoming
\enquote{increasingly narrow}) can be used to approximate arbitrary PDOs by
convolving with the derivatives of the kernels.

One example is a sequence of Gaussians
\begin{equation}
  \psi_{\eps}(x) := \frac{1}{(2\pi \eps)^{d/2}}e^{-x^{2}/\eps}
\end{equation}
(now indexed by \(\eps > 0\) instead of natural numbers).
For \(\eps \to 0\), we have \(\psi_{\eps} \to \delta_{0}\), as is easy to check with \cref{thm:dirac_sequence}.
This naturally leads to the \enquote{derivative of Gaussian} discretization used
in our experiments: we can approximate the PDO as convolution with a derivative
of a Gaussian kernel and then simply discretize this Gaussian derivative by
sampling it on the grid points.

The discussion in this subsection focused on \(1 \times 1\) PDOs and kernels, i.e.\
\(\cin = \cout = 1\), but since convergence for multi-dimensional PDOs and
kernels works component-wise, everything generalizes immediately to that setting.

\section{The relation between kernels and PDOs}\label{sec:kernels_vs_pdos}
Convolutional kernels and PDOs are closely related but also differ in some important ways.
This appendix is meant to briefly summarize various aspects of their relation that would
otherwise only be scattered throughout this paper.

First and foremost, we would like to emphasize that, in a continuous setting,
PDOs are \emph{not} just a special case of convolutions with classical kernels.
For example, the gradient operator is not represented by convolution with any kernel.
In fact, even the identity operator (a zeroth-order PDO) is not the convolution
with any function.

However, there are two caveats to the above statement. First, if we allow convolutions
with \emph{Schwartz distributions}, rather than the usual kernels, then this generalizes
both PDOs and convolutions with functions. In this broader framework, we \emph{can} thus
interpret PDOs as convolutions. But to the best of our knowledge, convolutions with
Schwartz distributions have not been previously discussed in a deep learning context,
so for the common usage of \enquote{convolution}, PDOs are distinct operators.

The second caveat is that when we discretize (translation-equivariant) PDOs on a regular grid,
they become convolutions in this discrete setting. Even so, the differences in the continuum
can matter in practice. First, we might want to discretize on point clouds or meshes. Second,
even on a regular grid, there are different methods for discretizing PDOs. While all of them
lead to convolutions, they lead to convolutions with slightly different kernels. This shows
that there is no clear one-to-one correspondence between kernels and PDOs even on discrete
regular grids. Finally, the space of discretizations of \emph{equivariant} PDOs is not necessarily
the same as the space of discretizations of \emph{equivariant} kernels.

Finally, we would like to mention a connection that is outside the scope of this paper:
\emph{infinite series} of differential operators and convolutions can be the same
even in the continuous setting. As an example, consider the diffusion equation
\begin{equation}
  \partial_t u(x, t) = \Delta u(x, t)\,.
\end{equation}
Its solution can be written as
\begin{equation}
  u(x, t) = \exp(t\Delta)u(x, t = 0)\,.
\end{equation}
The time evolution operator $\exp(t\Delta)$ is an infinite series of differential operators,
but can also be written as a convolution with a heat kernel. However, these infinite series
are not covered by our work---we restrict ourselves to PDOs of \emph{finite} order. We mention
this connection between infinite PDOs and convolutions only to avoid confusion for readers
familiar with this correspondence.

\section{Proof of the PDO equivariance constraint}\label{sec:diffop_equivariance_proofs}
In this section, we prove \cref{thm:translation_equivariance} and
\cref{thm:diffop_equivariance}, the characterizations of translation- and
\(G\)-equivariance for PDOs. \Cref{sec:distribution_equivariance} already
contains a proof of \cref{thm:diffop_equivariance} that is arguably more
insightful, but in this section we present an elementary proof that does not
rely on Schwartz distributions.

\subsection{Translation equivariance}
We begin by proving that translation equivariance of a PDO is equivalent to spatially
constant coefficients, \cref{thm:translation_equivariance}.

Let \(\Phi_{ij} = \sum_{\alpha}c^{ij}_{\alpha}\partial^{\alpha}\), where \(c^{ij}_{\alpha} \in C^{\infty}(\R^{d})\) and
\(i, j\) index the \(\cout \times \cin\) matrix describing the PDO.
The equivariance condition \cref{eq:equivariance_def} for the special case of
translations is
\begin{equation}\label{eq:translation_equivariance_def}
  \Phi(t \rhd_{\text{in}} f) = t \rhd_{\text{out}} \Phi(f)\quad \forall t \in \R^{d}, f \in \mathcal{F}_{\text{in}}\,.
\end{equation}
Using the definition of the \(\rhd\) action in \cref{eq:action_definition}, this becomes
\begin{equation}\label{eq:translation_equivariance_explicit}
 \sum_{\alpha}c_{\alpha}(x)\partial^{\alpha}f(x - t) = \sum_{\alpha}c_{\alpha}(x - t)\partial^{\alpha}f(x - t) \quad\forall x, t \in \R^{d}, f \in \mathcal{F}_{\text{in}}\,.
\end{equation}
Here, \(c_{\alpha}\) are matrix-valued, \(f\) is vector valued, and we have a
matrix-vector product between the two. Explicitly, this means
\begin{equation}
   \sum_{j}\sum_{\alpha}c^{ij}_{\alpha}(x)\partial^{\alpha}f_{j}(x - t) = \sum_{j}\sum_{\alpha}c^{ij}_{\alpha}(x - t)\partial^{\alpha}f_{j}(x - t)
\end{equation}
for all indices \(i\). If \(f\) is zero in all but one component, the sum over
\(j\) reduces to one summand. We therefore get a simpler scalar constraint
\begin{equation}
  \sum_{\alpha}c^{ij}_{\alpha}(x)\partial^{\alpha}g(x - t) = \sum_{\alpha}c^{ij}_{\alpha}(x - t)\partial^{\alpha}g(x - t)
\end{equation}
that has to hold for all indices \(i, j\), where \(g \in C^{\infty}(\R^{d})\) is now scalar-valued.

Since this must hold for all functions \(g\), we get an equality of differential operators,
\begin{equation}
  \sum_{\alpha}c^{ij}_{\alpha}(x)\partial^{\alpha} = \sum_{\alpha}c^{ij}_{\alpha}(x - t)\partial^{\alpha} \qquad\forall x, t \in \R^{d}\,.
\end{equation}
This implies \(c^{ij}_{\alpha}(x - t) = c^{ij}_{\alpha}(x)\) for all
\(x\), \(t\) and \(\alpha\), which in turn means that \(c^{ij}_{\alpha}\) must be constants.
From \cref{eq:translation_equivariance_explicit}, it is also apparent that
constant coefficients are sufficient for translation equivariance, which proves
the converse direction.

\subsection{\(G\)-equivariance}
Secondly we prove \cref{thm:diffop_equivariance}, the \(G\)-steerability
constraint for PDOs.
Recall that the \(G\)-equivariance condition for \(D(P)\) is
\begin{equation}
  D(P)(g \rhd_{\text{in}} f) = g \rhd_{\text{out}} (D(P)f)\,,
\end{equation}
where \(f \in \mathcal{F}_{\text{in}}\) and \(g \in G\). If we write out the
definition of the induced action \(\rhd\), this becomes
\begin{equation}\label{eq:diffop_constraint_explicit}
  D(P)(\rho_{\text{in}}(g)f \circ g^{-1}) = \rho_{\text{out}}(g)(D_{P}f) \circ g^{-1}\,.
\end{equation}
We write \(P = \sum_{\alpha}c_{\alpha}x^{\alpha}\), where \(c_{\alpha}\) are constant
matrix-valued coefficients \(c_{\alpha} \in \R^{\cout \times \cin}\). The LHS
of \cref{eq:diffop_constraint_explicit} is then
\begin{equation}\label{eq:lhs}
  D(P)(\rho_{\text{in}}(g)f \circ g^{-1}) = \sum_{\alpha}c_{\alpha}\rho_{\text{in}}(g)\partial^{\alpha}(f \circ g^{-1})\,.
\end{equation}
In \cref{thm:chain_rule}, we will show that the last term in \cref{eq:lhs} is given by
\begin{equation}
  \partial^{\alpha} (f \circ g^{-1}) = \left(\left(g^{-T}\partial\right)^{\alpha}f\right) \circ g^{-1}\,,
\end{equation}
where \(g^{-T} := \left(g^{-1}\right)^{T}\) and \(\partial\) is understood as a
vector that \(g^{-T}\) acts on. Plugging this into \cref{eq:lhs}, we get
\begin{equation}\label{eq:lhs_explicit}
  D(P)(\rho_{\text{in}}(g)f \circ g^{-1}) = \sum_{\alpha}c_{\alpha}\left(\left(g^{-T}\partial\right)^{\alpha}\rho_{\text{in}}(g)f\right) \circ g^{-1}
\end{equation}
for the LHS of \cref{eq:diffop_constraint_explicit}.
For the matrix of polynomials \(P = \sum_{\alpha}c_{\alpha}x^{\alpha}\), we now define
\begin{equation}
  g \cdot P := \sum_{\alpha}c_{\alpha}(g^{-1}x)^{\alpha}\,,
\end{equation}
which is again a matrix of polynomials, each one rotated by \(g\).
Then \cref{eq:lhs_explicit} can be written more compactly as
\begin{equation}
  D(P)(\rho_{\text{in}}(g)f \circ g^{-1}) = (D(g^{T} \cdot P)\rho_{\text{in}}(g)f) \circ g^{-1}\,.
\end{equation}
Plugging this back into \cref{eq:diffop_constraint_explicit}, canceling the
\(g^{-1}\) and using the fact that this has to hold for all \(f\), we get
\begin{equation}
  D(g^{T} \cdot P)\rho_{\text{in}}(g) = \rho_{\text{out}}(g) D(P)
\end{equation}
as an equality of differential operators. We move the \(\rho_{\text{in}}(g)\) to
the other side and use the fact that the map \(D\) is bijective, which
yields our final constraint
\begin{equation}
  P(g^{-T}x) = \rho_{\text{out}}(g)P(x)\rho_{\text{in}}(g)^{-1}\,.
\end{equation}
This concludes the proof of the \(G\)-steerability constraint for PDOs.

Finally, we prove the Lemma we just made use of, a higher-dimensional chain rule
for the special case we need:
\begin{lemma}\label{thm:chain_rule}
  Let \(f \in C^{\infty}(\R^{d}, \R^{c})\) and \(g \in \GL(\R^{d})\). Then for any
  multi-index \(\alpha\),
  \begin{equation}
    \partial^{\alpha} (f \circ g^{-1}) = \left(\left(g^{-T}\partial\right)^{\alpha}f\right) \circ g^{-1}\,,
  \end{equation}
  where \(g^{-T} := \left(g^{-1}\right)^{T}\).
\end{lemma}
\begin{proof}
  In general, for a linear map \(A \in \GL(\R^{d})\), we have
  \begin{equation}
    \partial_{i}A_{j}(x) = A_{ji}\quad \forall x\,.
  \end{equation}
  Therefore,
  \begin{align}
    \partial_{i}(f \circ g^{-1}) &= \sum_{j}(\partial_{j}f) \circ g^{-1} \partial_{i}(g^{-1})_{j}\\
                      &= \sum_{j} (\partial_{j}f) \circ g^{-1}\cdot(g^{-1})_{ji}\\
                      &= \left(\left(\left(g^{-1}\right)^{T}\partial\right)_{i} f\right) \circ g^{-1}\,.
  \end{align}
  We can apply this iteratively to show that
  \begin{equation}
    \partial^{\alpha} (f \circ g^{-1}) = \left(\left(\left(g^{-1}\right)^{T}\partial\right)^{\alpha}f\right) \circ g^{-1}\,.
  \end{equation}
\end{proof}

\section{Transferring steerable kernel bases to steerable PDO bases}\label{sec:transfer_solutions}
In this section, we develop the method presented in
\cref{sec:solving_constraint} in more detail and with proofs.

We fix a group \(G \leq \orth(d)\) and representations \(\rho_{\text{in}}\) and
\(\rho_{\text{out}}\). Then we write \(\mathcal{K}\) for the space of
\(G\)-steerable kernels. Because the steerability constraints for kernels and
PDOs are identical in this setting, the space of equivariant PDOs is the image
under the isomorphism \(D\) of the intersection
\begin{equation}
  \mathcal{K}_{\text{pol}} := \R[x_{1}, \ldots, x_{d}] \cap \mathcal{K}\,.
\end{equation}
In words, the space of equivariant PDOs is isomorphic to the space of
\emph{polynomial} steerable kernels. Both spaces are infinite-dimensional
real vector spaces. The question we tackle now is how
we can find a basis of \(\mathcal{K}_{\text{pol}}\) given a basis of
\(\mathcal{K}\) under certain conditions.

It will vastly simplify our discussion to treat \(\mathcal{K}\) and
\(\mathcal{K}_{\text{pol}}\) as \emph{modules} over invariant kernels instead of
as real vector spaces, at
least for now. A module is a generalization of a vector space, where the scalars
for scalar multiplication can form a ring instead of a field.
Because the radial parts of steerable kernels are unrestricted, it makes sense
to think of \(\mathcal{K}\) as a module over the ring of radial functions. Formally:
\begin{lemma}\label{thm:module}
  \(\mathcal{K}\) is a \(C^{\infty}(\Rplus)\)-module, with scalar multiplication
  defined by
  \begin{equation}
    (f\kappa)(x) := f\left(\abs{x}^{2}\right)\kappa(x)
  \end{equation}
  for \(\kappa \in \mathcal{K}\) and \(f \in C^{\infty}(\Rplus)\).
  Similarly, \(\mathcal{K}_{\text{pol}}\) is an
  \(\R[r^{2}]\)-module, where \(r^{2} := x_{1}^{2} + \ldots + x_{d}^{2}\) and
  multiplication is simply multiplication of polynomials.
\end{lemma}
\begin{proof}
  The steerability constraint
  \begin{equation}
    \kappa(gx) = \rho_{\text{out}}(g)\kappa(x)\rho_{\text{in}}(g)^{-1}
  \end{equation}
  is clearly \(\R\)-linear and in particular, \(\mathcal{K}\) is closed under
  addition and \(0 \in \mathcal{K}\). Furthermore, for \(\kappa \in \mathcal{K}\),
  \begin{equation}
    \begin{split}
      (f\kappa)(gx) &= f(\abs{gx}^{2})\kappa(gx)\\
      &= f(\abs{x}^{2})\rho_{\text{out}}(g)\kappa(x)\rho_{\text{in}}(g)^{-1}\\
      &= \rho_{\text{out}}(g)(f\kappa)(x)\rho_{\text{in}}(g)^{-1}\,,
    \end{split}
  \end{equation}
  so \(\mathcal{K}\) is also closed under the given scalar multiplication.
  The proof for \(\mathcal{K}_{\text{pol}}\) is exactly analogous.
\end{proof}
In \cref{thm:module} and in the following, we write \(r^{2}\) instead of
\(\abs{x}^{2}\) simply to emphasize its role as a polynomial; so when reading
\(r^{2}\), think of it as a polynomial, and when reading \(\abs{x}^{2}\) simply
as a function of \(x\).

A basis of a module is defined analogously to a basis of a vector space, as a
set of linearly independent vectors that span the entire module. However, linear
combinations now allow coefficients in the ring of radial functions, instead of
only real numbers. This means that fewer vectors are needed to span the entire
space, because the coefficients \enquote{do more work}.

In contrast to vector spaces, not every module has a basis.
\(\mathcal{K}\) and \(\mathcal{K}_{\text{pol}}\) do have a basis in the cases we
consider in the paper but for this section that doesn't matter to us: we will
simply assume that \(\mathcal{K}\) has a basis with certain properties and then
transfer this basis to \(\mathcal{K}_{\text{pol}}\). That the method developed
here is indeed applicable to subgroups of \(\orth(2)\) and \(\orth(3)\) will be
the topic of \cref{sec:solutions_proofs}.

We roughly proceed in two steps:
\begin{enumerate}
  \item We show that a basis of \(\mathcal{K}\) (over
  \(C^{\infty}(\Rplus)\)) that consists only of polynomials (and fulfills a few other
  technical conditions) is also a basis of
  \(\mathcal{K}_{\text{pol}}\), but this time of course over \(\R[r^{2}]\).
  \item We then show how to turn this \emph{module basis} into a \emph{vector
    space} basis of \(\mathcal{K}_{\text{pol}}\).
\end{enumerate}

The first step is formalized as follows:
\begin{proposition}\label{thm:polynomial_basis}
  Let \(B \subset \R[x_{1}, \ldots, x_{d}]^{\cout \times \cin}\) be a basis of \(\mathcal{K}\)
  such that no matrix of polynomials in \(B\) is divisible by \(r^{2}\) and each
  one is homogeneous. Then \(B\) is also a basis of \(\mathcal{K}_{\text{pol}}\)
  as an \(\R[r^{2}]\)-module.
\end{proposition}
Here, we say that \(b \in \R[x_{1}, \ldots, x_{d}]^{\cout \times \cin}\) is divisible by
\(r^{2}\) if \emph{every} component is divisible by \(r^{2}\) (as a polynomial).
We call it homogeneous if all its entries are homogeneous polynomials of the
same degree.
\begin{proof}
  Linear independence is obvious: \(\R[r^{2}] \subset C^{\infty}(\Rplus)\), so if \(B\) is
  linearly independent over \(C^{\infty}(\Rplus)\), then it is also linearly
  independent over \(\R[r^{2}]\).

  To show that \(B\) generates \(\mathcal{K}_{\text{pol}}\), first let
  \(p \in \mathcal{K}_{\text{pol}}\) be homogeneous of degree \(l\).
  Because \(p \in \mathcal{K}\), there are
  \(f_{i} \in C^{\infty}(\Rplus)\) and \(\kappa_{i} \in B\) such that
  \begin{equation}
    p(x) = \sum_{i}f_{i}(\abs{x}^{2})\kappa_{i}(x)\,.
  \end{equation}
  We want to show that \(f_{i} \in \R[z]\), i.e.\ that \(f_{i}\) is a polynomial.
  Since each \(\kappa_{i}\) is homogeneous of degree \(l_{i}\),
  we get
  \begin{equation}
    \sum_{i} f_{i}(\lambda^{2}\abs{x}^{2})\lambda^{l_{i}}\kappa_{i}(x) = \lambda^{l}\sum_{i} f_{i}(\abs{x}^{2})\kappa_{i}(x)
  \end{equation}
  because of \(p(\lambda x) = \lambda^{l}p(x)\) for any \(\lambda \geq 0\). Because the \(\kappa_{i}\) are
  linearly independent, we must have
  \begin{equation}
    f_{i}(\lambda^{2}\abs{x}^{2}) = \lambda^{l - l_{i}}f_{i}(\abs{x}^{2})\,.
  \end{equation}
  Thus, \(f_{i}\) is homogeneous of degree \(\frac{l - l_{i}}{2}\) and as we
  show in \cref{thm:homogeneous_function} below, this implies
  \begin{equation}
    f_{i}(z) = c z^{(l - l_{i})/2}\,,
  \end{equation}
  or alternatively
  \begin{equation}
    f_{i}(\abs{x}^{2}) = c_{i}\abs{x}^{l - l_{i}}\,.
  \end{equation}
  What remains to show is that \(\frac{l - l_{i}}{2}\) is a natural number,
  i.e.\ that \(l - l_{i}\) is even and non-negative.
  To prove this, we divide all the \(i\) into two groups: those for which \(l - l_{i}\)
  is even and those for which it is odd. Then we get an expression of the form
  \begin{equation}
    p(x) = \abs{x}\sum_{k} a_{k}r^{2m_{k}}\kappa_{k} + \sum_{k} b_{k}r^{2n_{k}}\kappa_{k}\,.
  \end{equation}
  Each summand is a rational function (the quotient of two polynomials), so both sums are rational functions.
  \(p(x)\) is also rational, so the first term on the RHS has to be rational. This is only
  possible if \(a_{k} = 0\) for all \(k\), otherwise we could divide by the
  (rational) sum and should get a rational function, but \(\abs{x}\) is not rational.
  This shows that \(l - l_{i}\) is even for all \(i\).

  Now let \(l_{\text{max}} := \max_{i} l_{i}\). Then
  \begin{equation}
    p(x) = \frac{\sum_{i}c_{i}\abs{x}^{l_{\text{max}} - l_{i}}\kappa_{i}(x)}{\abs{x}^{l_{\text{max}} - l}}\,.
  \end{equation}
  It is not possible to cancel any terms: because \(\kappa_{i}\) is not divisible by
  \(r^{2}\) and for one \(i\), \(l_{\text{max}} - l_{i} = 0\), the enumerator is
  not divisible by \(r^{2}\) (as a polynomial). Since the denominator is a power
  of \(r^{2}\), the fraction can't be simplified. But we know that \(p(x)\) is a
  polynomial. Therefore, \(l_{\text{max}} \leq l\).

  In summary, we've shown that
  \begin{equation}
    f_{i}(z) = c_{i}z^{(l - l_{i})/2}\,
  \end{equation}
  where \(l - l_{i}\) is even and non-negative. Therefore, all \(f_{i}\) are
  polynomials.

  Now recall that we assumed \(p\) to be homogeneous. But this is
  no significant restriction: we can write any polynomial as a sum of
  homogeneous polynomials, each of which can be written as a linear combination
  of the \(\kappa_{i}\), as we just showed. Adding those up leads to a linear
  combination of the \(\kappa_{i}\) for arbitrary polynomials. This complete the proof.
\end{proof}
\begin{lemma}\label{thm:homogeneous_function}
  Let \(f: \Rplus \to \R\) be homogeneous of degree \(l \in \R\), meaning that
  \(f(\lambda x) = \lambda^{l}x\) for all \(\lambda \geq 0\). Then \(f(z) = cz^{l}\) for some
  \(c \in \R\) (and if \(l \in \N\), then \(f\) is a polynomial).
\end{lemma}
Note that this Lemma does not hold in higher dimensions -- in general there are
many more homogeneous functions than polynomials!
\begin{proof}
  For any \(z \geq 0\),
  \begin{equation}
    f(z) = z^{l}f(1)\,,
  \end{equation}
  which proves the claim by setting \(c := f(1)\).
\end{proof}

In \cref{sec:solutions_proofs}, we will show that the construction of angular basis
elements described in \cref{sec:solving_constraint} leads to a basis \(B\) of
\(\mathcal{K}\) with the properties required by \cref{thm:polynomial_basis}. It
then follows that this also defines a basis of \(\mathcal{K}_{\text{pol}}\) as a
module. We now come to the second step, turning this module basis into a
vector space basis.
\begin{proposition}
  Let \(B\) be a basis of \(\mathcal{K}_{\text{pol}}\) as an
  \(\R[r^{2}]\)-module. Then
  \begin{equation}
    \setcomp{r^{2k}b}{k \in \N_{\geq 0}, b \in B}
  \end{equation}
  is a basis of \(\mathcal{K}_{\text{pol}}\) as a real vector space.
\end{proposition}
\begin{proof}
  Let \(v \in \mathcal{K}_{\text{pol}}\). By assumption, there are then basis vectors
  \(b_{1}, \ldots, b_{n} \in B\) and coefficients \(p_{1}, \ldots, p_{n} \in \R[r^{2}]\) such that
  \begin{equation}
    v = \sum_{i = 1}^{n} p_{i}b_{i}\,.
  \end{equation}
  For each \(i = 1, \ldots, n\) there are also real coefficients
  \(a^{(i)}_{0}, \ldots, a^{(i)}_{m_{i}}\) such that
  \begin{equation}
    p_{i} = \sum_{k = 0}^{m_{i}}a^{(i)}_{k}r^{2k}\,.
  \end{equation}
  Combining these equations, we get
  \begin{equation}
    v = \sum_{i = 1}^{n}\sum_{k = 0}^{m_{i}}a^{(i)}_{k}r^{2k}b_{i}\,.
  \end{equation}
  This is a real linear combination of elements of the form \(r^{2k}b\) for
  \(b \in B\). So since \(v\) was arbitrary, the set of such elements does indeed
  span \(\mathcal{K}_{\text{pol}}\).

  To prove linear independence, let
  \begin{equation}
    \sum_{i = 0}^{n} a_{i} r^{2k_{i}}b_{i} = 0
  \end{equation}
  for some choice of real coefficients \(a_{i}\). Now we only need to note that
  \(a_{i}r^{2k_{i}} \in \R[r^{2}]\). So we can interpret this expression as a
  linear combination over \(\R[r^{2}]\). Because \(B\) is a basis, it follows
  that \(a_{i}r^{2k_{i}} = 0\) for all \(i\), and thus \(a_{i} = 0\). That
  proves linear independence.
\end{proof}

As a final note, we show that the condition in \cref{thm:polynomial_basis} that
the basis elements are not divisible by \(r^{2}\) is purely technical; any basis
can easily be transformed into one that fulfills it in a canonical way, as
formalized by the following lemma. We do not formally need this result anywhere
but it may prove useful if this approach is extended to other groups \(G\)
because it clarifies which parts of the conditions in
\cref{thm:polynomial_basis} are actually important.
\begin{lemma}\label{thm:removing_r2}
  Let \(B \subset \R[x_{1}, \ldots, x_{d}]^{\cout \times \cin}\) be a basis of \(\mathcal{K}\).
  For \(b \in B\), we write \(b = r^{2k}b'\), where \(k\) is chosen maximally such
  that there is an (automatically unique) polynomial matrix
  \(b' \in \R[x_{1}, \ldots, x_{d}]^{\cout \times \cin}\). Then
  \(B' := \setcomp{b'}{b \in B}\) is a basis of \(\mathcal{K}\) and no \(b'\) is
  divisible by \(r^{2}\). Furthermore, if \(b\) is homogeneous, then so is \(b'\).
\end{lemma}
\begin{proof}
  First, note that \(b'\) is in fact well-defined: For \(k = 0\), \(b' = b\)
  works, while for \(k > \frac{\degree(b)}{2}\), no fitting \(b'\)
  exists\footnote{Here, the degree of a matrix of polynomials is the maximum of
    the degrees of all components (though in this case, even if \(k\) is larger
    than the minimum degree, no \(b'\) exists).}. So there is some maximal
  \(k \geq 0\) with the desired property. \(b'\) is clearly unique, namely
  \(b' = \abs{x}^{-2k}b\) as functions on \(\R^{d}\). Furthermore, \(b'\) is
  not divisible by \(r^{2}\) because \(k\) was chosen maximally.

  \(B'\) is also clearly a generating set:
  \(\kappa = \sum_{i}f_{i}b_{i} = \sum_{i}(f_{i}r^{2k_{i}})b_{i}'\).

  To prove linear independence, assume that
  \begin{equation}
    \sum_{i}f_{i}b_{i}' = 0\,.
  \end{equation}
  We can write \(k := \max_{i} k_{i}\) and then get
  \begin{equation}
    0 = r^{2k}\sum_{i}f_{i}b_{i}' = \sum_{i}f_{i}r^{2(k - k_{i})}b_{i}\,.
  \end{equation}
  This means that
  \begin{equation}
    f_{i}r^{2(k - k_{i})} = 0 \quad \forall i,
  \end{equation}
  which implies \(f_{i} = 0\) for all \(i\), since \(r^{2(k - k_{i})}\) is
  non-zero everywhere except in the origin and \(f_{i}\) is continuous.
  Therefore, \(B'\) is linearly independent and hence a basis of \(\mathcal{K}\).
\end{proof}
Now we can formulate \cref{thm:polynomial_basis} more generally: for any basis
\(B\) of \(\mathcal{K}\) consisting only of homogeneous matrices of polynomials,
the corresponding \(B'\) is a basis of \(\mathcal{K}_{\text{pol}}\).

\section{Solutions for important groups}\label{sec:solutions}
This and the next two appendices describe the solutions of the PDO equivariance
constraint for subgroups of \(\orth(2)\) and \(\orth(3)\). In this appendix,
we describe the general form of these solutions; we recommend readers who
are not interested in all the details focus on this one. \Cref{sec:solutions_proofs}
contains proofs for some claims we make in this appendix where these proofs do
not provide as much insight. Finally, \cref{sec:solution_tables} contains tables
with concrete solutions, it is mainly relevant to the implementation of steerable PDOs.

\subsection{Solutions for subgroups of \(\orth(2)\)}\label{sec:o2_solutions}
To solve the steerability constraint for all (compact) subgroups \(G \leq \orth(2)\) and for
arbitrary representations \(\rho_{\text{in}}\) and \(\rho_{\text{out}}\),
\citet{weiler2019} derive explicit bases only for irreducible representations,
since the general case can then easily be computed (see
\cref{sec:representation_theory} for details). This works exactly the same for
PDOs as well.

The angular basis elements \(\chi_{\beta}\) that they describe for irreps are all
matrices with entries of the form \(\cos(k\phi)\) and \(\sin(k\phi)\), where \(k\)
differs between basis elements but is the same for all entries of one matrix.
For example, for \(G = \SO(2)\), \(\rho_{\text{in}}\) the
frequency \(n\) irrep, i.e.\ \(\rho_{\text{in}}(g) = g^{n}\), and
\(\rho_{\text{out}}\) trivial, the angular basis elements are
\begin{equation}
  \chi_{1} = \mat{\cos(n\phi), & \sin(n\phi)}\qquad \text{and} \qquad \chi_{2} = \mat{-\sin(n\phi), & \cos(n\phi)}\,.
\end{equation}
We will show how to find the steerable PDO basis using this example, but the
method works exactly the same in all cases. Tables with all
explicit solutions can be found in \cref{sec:solution_tables}.

As described in \cref{sec:solving_constraint}, we now need to multiply these
matrices with the smallest power of \(\abs{x}\) such
that all entries become polynomials. We show in \cref{sec:solutions_proofs}
that the necessary coefficient for \(\cos(n\phi)\) and \(\sin(n\phi)\) is \(\abs{x}^{n}\),
so in the example above, we get
\begin{equation}
  \tilde{\chi}_{1} = \mat{\abs{x}^{n}\cos(n\phi), & \abs{x}^{n}\sin(n\phi)}\qquad \text{and} \qquad
  \tilde{\chi}_{2} = \mat{-\abs{x}^{n}\sin(n\phi), & \abs{x}^{n}\cos(n\phi)}\,.
\end{equation}
The entries are written in polar coordinates here, but they are in fact
polynomials in the Cartesian coordinates \(x_{1}\) and \(x_{2}\).
More precisely, we show in \cref{sec:solutions_proofs} that they are closely
related to Chebyshev polynomials, based on which we derive the following
explicit expressions:
\begin{equation}\label{eq:chebyshev_explicit}
  \begin{split}
    \abs{x}^{n}\cos(n\phi) = \sum_{\mathclap{i \leq n \text{ even}}}(-1)^{\frac{i}{2}}{n \choose i}x_{1}^{n - i} x_{2}^{i}
    \quad \text{and} \quad
    \abs{x}^{n}\sin(n\phi) = \sum_{\mathclap{i \leq n \text{ odd}}}(-1)^{\frac{i + 1}{2}}{n \choose i}x_{1}^{n - i} x_{2}^{i}\,.
  \end{split}
\end{equation}
This Cartesian form then allows us to interpret the polynomial as a differential
operator, by applying the ring isomorphism \(D\), i.e.\ plugging in \(\partial_{1}\) and \(\partial_{2}\) for \(x_{1}\) and \(x_{2}\).

If we set \(n = 1\) in our example above, which corresponds to a vector field,
we get simply \(\abs{x}\cos(\phi) = x_{1}\) and \(\abs{x}\sin(\phi) = x_{2}\), so we
recover the angular PDO basis
\begin{equation}
  D(\tilde{\chi}_{1}) = \mat{\partial_{1}, & \partial_{2}} = \operatorname{div}\qquad \text{and} \qquad
  D(\tilde{\chi}_{2}) = \mat{-\partial_{2}, & \partial_{1}} = \operatorname{curl_{2D}}
\end{equation}
that we already derived in \cref{sec:examples}. To get a complete basis, we
combine these PDOs with powers of the Laplacian.

\subsection{Solutions for \(\orth(3)\) and \(\SO(3)\)}\label{sec:o3_solutions}
We now turn to the other cases for which steerable kernel solutions have been
published, namely \(\orth(3)\) and \(\SO(3)\).
Like for \(\orth(2)\), we only need to consider pairs of irreducible representations.
As described in~\Citep{weiler2018a,lang2021}, we can build the corresponding
angular parts \(\chi_{\beta}\) out of real spherical harmonics \(Y_{lm}\) using
Clebsch-Gordan coefficients. We show in \cref{sec:solutions_proofs} that
we can apply this procedure to \(\abs{x}^{l}Y_{lm}\) instead of \(Y_{lm}\)
to obtain the corresponding \(\tilde{\chi}_{\beta}\). We then build the basis by
combining with powers of \(\abs{x}^{2}\), as we described in \cref{sec:solving_constraint} and
already did for \(\orth(2)\). To find the corresponding differential operators
\(D(\tilde{\chi}_{\beta})\), we only need a Cartesian representation of the polynomials
\(\abs{x}^{l}Y_{lm}\), which is fortunately well-known~\citep{varshalovich1988}.
This then again leads to a complete basis of the space of steerable PDOs, with
the same general form containing powers of the Laplacian.

\section{Proofs for solutions for subgroups of \(\orth(2)\) and \(\orth(3)\)}\label{sec:solutions_proofs}
In this section, we apply the method from \cref{sec:transfer_solutions} to
subgroups of \(\orth(2)\) and \(\orth(3)\), by describing bases for the space of
steerable kernels that satisfy the conditions of \cref{thm:polynomial_basis}.

\subsection{\(\orth(2)\)}
As described in \cref{sec:o2_solutions}, it is sufficient to consider
irreducible representations \(\rho_{\text{in}}\) and \(\rho_{\text{out}}\) and for
these representations, the angular part of the kernel has a basis consisting of
matrices with entries of the form \(\cos(n\phi)\) and \(\sin(n\phi)\) (where
\(n \in \Z\)). Crucially, \(n\) is the same for all entries of one matrix (though
it may differ between basis elements).

We claim that multiplying each angular basis element with \(r^{\abs{n}}\) gives
a basis of the space of \emph{polynomial} steerable kernels
\(\mc{K}_{\text{pol}}\) (as a \(\R[r^{2}]\)-module), as described in
\cref{sec:transfer_solutions}. This follows from \cref{thm:polynomial_basis} if
we can prove that \(r^{\abs{n}}\cos(n\phi)\) and \(r^{\abs{n}}\sin(n\phi)\) are
\begin{enumerate}[(i)]
  \item polynomials (in \(x, y\)),
  \item homogeneous (of degree \(n\)),
  \item and not divisible by \(r^{2}\).
\end{enumerate}

To show that \(r^{\abs{n}}\cos(n\phi)\) and \(r^{\abs{n}}\sin(n\phi)\) are polynomials,
we use \emph{Chebyshev polynomials}, which can among other things be seen as a
generalization of the addition theorems for \(\sin(2\phi)\) and \(\cos(2\phi)\).
Specifically, they are families of polynomials \(T_{n}\) (\emph{first kind})
and \(U_{n}\) (\emph{second kind}) defined for \(n \geq 0\) with the property that
\begin{align}\label{eq:chebyshev}
  \cos(n\phi) &= T_{n}(\cos \phi)\\
  \sin(n\phi) &= U_{n - 1}(\cos \phi)\sin \phi\,.
\end{align}
We extend the definition to negative \(n\) by setting
\begin{align}\label{eq:negative_signs}
  T_{-n} &:= T_{n}&&\text{for } n \geq 0\\
  U_{-1} &:= 0&&\\
  U_{-n - 1} &:= -U_{n - 1}&&\text{for } n \geq 1\,.
\end{align}
Then \cref{eq:chebyshev} holds for all \(n \in \Z\), as follows immediately from
the parity of \(\sin\) and \(\cos\).

Motivated by the close relation to Chebyshev polynomials, we then define the
following notation for the matrix entries we are considering:
\begin{align}
  \T_{n} &:= r^{\abs{n}} T_{n}(\cos \phi) = r^{\abs{n}}\cos(n\phi)\\
  \U_{n} &:= r^{\abs{n}}\sin(\phi) U_{n - 1}(\cos \phi) = r^{\abs{n}}\sin(n\phi)\,.
\end{align}

We now prove claim (i), that \(\T_{n}\) and \(\U_{n}\) are polynomials in the
Cartesian coordinates \(x, y\), by using a few well-known facts about Chebyshev polynomials:
First, \(T_{n}\) has degree \(\abs{n}\), so the highest order term in
\(\T_{n}\) is \(r^{\abs{n}}(\cos \phi)^{\abs{n}} = x^{\abs{n}}\). Similarly,
\(U_{n - 1}\) has degree \(\abs{n} - 1\) and the highest order term in \(\U_{n}\) is
thus \(r^{\abs{n}}(\cos \phi)^{\abs{n} - 1}\sin \phi = x^{\abs{n} - 1}y\). Lower order
terms have additional powers of \(r\) that aren't \enquote{matched} by a cosine
or sine. But the second fact about Chebyshev polynomials is that they are either
even or odd, so all of these powers of \(r\) are
even (i.e.\ powers of \(r^{2}\)) and thus themselves polynomials.

That \(\T_{n}\) and \(\U_{n}\) are homogeneous of degree \(n\) -- claim (ii) -- is immediately
clear from their definition.

It remains to show claim (iii), that they are not divisible by \(r^{2}\).
For that, we use the following Lemma:
\begin{lemma}\label{thm:non_divisibility}
  Let \(p\) be a non-zero, harmonic, homogeneous polynomial. Then \(p\) is not
  divisible by \(r^{2}\).
\end{lemma}
Here, \(p\) is harmonic if its Laplacian vanishes.
\begin{proof}
  For any homogeneous polynomial \(p\), there is a \emph{unique} decomposition
  \begin{equation}
    p = r^{2}q + h\,,
  \end{equation}
  such that \(h\) is harmonic and homogeneous. Since \(q = 0, h = p\) is one
  such decomposition for \(p\) harmonic, there is no solution with \(h = 0\)
  (unless \(p = 0\)). So non-zero homogeneous harmonic polynomials are
  not divisible by \(r^{2}\).
\end{proof}

Now we only need to show that \(\T_{n}\) and \(\U_{n}\) are in fact harmonic. This can
be done with a brief calculation in polar coordinates:
\begin{equation}
  \begin{split}
    \Delta \T_{n} &= \left(\partial_{r}^{2} + \frac{1}{r}\partial_{r} + \frac{1}{r^{2}}\partial_{\phi}^{2}\right)\left(r^{n}\cos(n\phi)\right)\\
    &= n(n - 1)r^{n - 2}\cos(n\phi) + nr^{n - 2}\cos(n\phi) -n^{2}r^{n - 2}\cos(n\phi)\\
    &= 0\,.
  \end{split}
\end{equation}
The same holds for \(\U_{n}\):
\begin{equation}
  \begin{split}
    \Delta \U_{n} &= \left(\partial_{r}^{2} + \frac{1}{r}\partial_{r} + \frac{1}{r^{2}}\partial_{\phi}^{2}\right)\left(r^{n}\sin(n\phi)\right)\\
    &= n(n - 1)r^{n - 2}\sin(n\phi) + nr^{n - 2}\sin(n\phi) -n^{2}r^{n - 2}\sin(n\phi)\\
    &= 0\,.
  \end{split}
\end{equation}
In summary, we have shown that \(\T_{n}\) and \(\U_{n}\) are homogeneous
polynomials not divisible by \(r^{2}\), which makes \cref{thm:polynomial_basis} applicable.

In order to make the basis practically applicable, we also give explicit
Cartesian expressions for \(\T_{n}\) and \(\U_{n}\). We restrict ourselves to
non-negative \(n\) to keep the notation less cluttered. \Cref{eq:negative_signs}
immediately gives the case for \(n < 0\).
An explicit formula for the
Chebyshev polynomials of the first kind is
\begin{equation}
  T_{n}(x) = \sum_{k=0}^{\left\lfloor\frac{n}{2}\right\rfloor}{n \choose 2k}\left(x^{2}-1\right)^{k} x^{n-2 k}\,.
\end{equation}
It follows that
\begin{equation}
  \begin{split}
    \T_{n} &= r^{n}T_{n}(\cos \phi) \\
    &=\sqrt{x^{2} + y^{2}}^{n}T_{n}\left(\frac{x}{\sqrt{x^{2} + y^{2}}}\right)\\
    &= \sum_{k=0}^{\left\lfloor\frac{n}{2}\right\rfloor}{n \choose 2k}\left(-y^{2}\right)^{k} x^{n - 2k}\\
    &= \sum_{k=0}^{\left\lfloor\frac{n}{2}\right\rfloor}{n \choose 2k}(-1)^{k}y^{2k} x^{n - 2k}\,.
  \end{split}
\end{equation}
Similarly, the Chebyshev polynomials of the second kind are given by
\begin{equation}
  U_{n}(x) = \sum_{k=0}^{\left\lfloor\frac{n}{2}\right\rfloor}{n + 1 \choose 2k + 1}\left(x^{2}-1\right)^{k} x^{n - 2k}\,,
\end{equation}
which means that
\begin{equation}
  \begin{split}
    \U_{n} &= r^{n}\sin \phi U_{n - 1}(\cos \phi) \\
    &= y\sqrt{x^{2} + y^{2}}^{n - 1}U_{n - 1}\left(\frac{x}{\sqrt{x^{2} + y^{2}}}\right) \\
    &= y\sum_{k=0}^{\left\lfloor\frac{n - 1}{2}\right\rfloor}{n \choose 2k + 1}\left(-y^{2}\right)^{k} x^{n - 1 - 2k}\\
    &= \sum_{k=0}^{\left\lfloor\frac{n - 1}{2}\right\rfloor}{n \choose 2k + 1}(-1)^{k}y^{2k + 1} x^{n - (2k + 1)}\,.
  \end{split}
\end{equation}

\subsection{\(\orth(3)\)}
As already mentioned in \cref{sec:o3_solutions}, for \(\orth(3)\) and
\(\SO(3)\), the irreps angular basis between irreducible representation can be built
by combining real spherical harmonics \(Y_{lm}\) using Clebsch-Gordan coefficients.
We refer to \citet{weiler2018a,lang2021} for details on how this works since it
is exactly the same procedure whether one uses kernels or PDOs. The only fact we
need to know here is that each basis element is built from spherical harmonics
with the same \(l\).

The necessary steps are now very similar to those for \(\orth(2)\) and its
subgroups. We will use \(r^{l}Y_{lm}\) where we had \(\T_{n}\) and \(\U_{n}\)
for \(\orth(2)\). Here it becomes important that inside one basis element, only
one \(l\) appears, because this means we can multiply the entire matrix
of polynomials by \(r^{l}\).

We would then need to show that these are homogeneous polynomials
not divisible by \(r^{2}\) but for spherical harmonics, it is already very well known
that \(r^{l}Y_{lm}\) are homogeneous harmonic polynomials of degree \(l\);
\cref{thm:non_divisibility} then implies that they are not divisible by
\(r^{2}\). For explicit Cartesian expressions, we refer to e.g.\ \citet{varshalovich1988}.

\section{Solution tables}\label{sec:solution_tables}
Using the results from \cref{sec:transfer_solutions,sec:solutions_proofs}, we
now very easily get complete bases for all subgroups of \(\orth(2)\) and for all
irreducible representations by transferring the solutions by \citet{weiler2019}.
The appendix of~\citep{weiler2019} contains tables with all \emph{kernel}
solutions; we simply replace every term of the form \(\cos(k\phi)\) with \(\T_{k}\)
and \(\sin(k\phi)\) with \(\U_{k}\) to get the following tables.

For the sake of readability, we only write the polynomials \(\T_{k}\) and
\(\U_{k}\) inside the tables, though the PDOs themselves should of course be
\(D(\T_{k})\) and \(D(\U_{k})\).
Explicit formulas for \(\T_{k}\) and \(\U_{k}\) were given in
\cref{eq:chebyshev_explicit}.

All bases described here are \emph{module} bases for the space of steerable PDOs
as an \(\R[\Delta]\)-module. See \cref{sec:solving_constraint} and
\cref{sec:transfer_solutions} for details.

\subsection{Special orthogonal group \(\SO(2)\)}
The irreducible representations of \(\SO(2)\) are the trivial representation \(\psi_{0}\)
and those of the form
\begin{equation}
  \psi_{k}: \SO(2) \to \GL(\R^{2}), \qquad g \mapsto g^{k}\,.
\end{equation}
The bases for all the combinations of these irreps are:

\begin{tabular}{>{\columncolor{tbl_highlight}\centering\arraybackslash}m{1.45cm}>{\centering\arraybackslash}m{3cm}>{\centering\arraybackslash}m{8cm}}
  \rowcolor{tbl_highlight}
  \diagbox{Out}{In} & \(\psi_{0}\) & \(\psi_{n}, \quad n \in \N_{>0}\) \\
  \tablevspace\(\psi_{0}\) &  \tablevspace\(\mat{1}\) &  \tablevspace\(\mat{\T_n & \U_n}\), \(\mat{-\U_{n} & \T_{n}}\)\\
  \(\psi_{m}\),\quad \(m \in \N_{>0}\) & \(\mat{\T_{m} \\ \U_{m}}\), \(\mat{-\U_{m} \\ \phantom{-}\T_{m}}\)
                    & {\(\begin{aligned}
                        \mat{\T_{m - n} & -\U_{m - n} \\ \U_{m - n} & \phantom{-}\T_{m - n}},
                      &\mat{-\U_{m - n} & -\T_{m - n} \\ \phantom{-}\T_{m - n} & -\U_{m - n}},\\
                      \mat{\T_{m + n} & \phantom{-}\U_{m + n} \\ \U_{m + n} & -\T_{m + n}},
                      &\mat{-\U_{m + n} & \phantom{-}\T_{m + n} \\ \phantom{-}\T_{m + n} & \phantom{-}\U_{m + n}}
                      \end{aligned}\)}

\end{tabular}

\subsection{Orthogonal group \(\orth(2)\)}
Elements of \(\orth(2)\) can be written as a tuple \((r, s)\) consisting of a
rotation \(r \in \SO(2)\) and a flip \(s \in \{\pm 1\}\). The irreducible
representations with frequency \(k > 0\) are similar to those for \(\SO(2)\),
only with the additional flip:
\begin{equation}
  \psi_{1,k}: \SO(2) \to \GL(\R^{2}), \qquad (r, s) \mapsto r^{k} \circ s\,.
\end{equation}
Here, \(s\) is understood as a map acting on \(\R^{2}\) either as the identity or
by flipping along a certain axis. In contrast to \(\SO(2)\), there are now two irreducible
representations for \(k = 0\), namely the trivial representation
\(\psi_{0, 0}(r, s) = 1\) and the representation \(\psi_{1, 0}(r, s) = s\).
The solutions for all possible combinations of these irreducible representations
are as follows:

\begin{tabular}{>{\columncolor{tbl_highlight}\centering\arraybackslash}m{1.45cm}>{\centering\arraybackslash}m{2cm}>{\centering\arraybackslash}m{2cm}>{\centering\arraybackslash}m{6.5cm}}
  \rowcolor{tbl_highlight}
  \diagbox{Out}{In} & \(\psi_{0, 0}\) & \(\psi_{1, 0}\) & \(\psi_{1, n}, \quad n \in \N_{>0}\) \\
  \tablevspace\(\psi_{0, 0}\) &  \tablevspace\(\mat{1}\) & \tablevspace \(\emptyset\) & \tablevspace \(\mat{-\U_{n} & \T_{n}}\)\\
  \tablevspace\(\psi_{1, 0}\) &  \tablevspace\(\emptyset\) & \tablevspace \(\mat{1}\) & \tablevspace \(\mat{\T_{n} & \U_{n}}\)\\
  \(\psi_{1, m}\),\quad \(m \in \N_{>0}\) & \(\mat{-\U_{m} \\ \phantom{-}\T_{m}}\)
  & \(\mat{\T_{m} \\ \U_{m}}\)
                    & \(\mat{\T_{m - n} & -\U_{m - n} \\ \U_{m - n} & \phantom{-}\T_{m - n}}\),
                      \(\mat{\T_{m + n} & \phantom{-}\U_{m + n} \\ \U_{m + n} & -\T_{m + n}}\)
\end{tabular}

\subsection{Reflection group \({\pm 1}\)}
The reflection group has only two irreducible representations: the trivial
representation \(\psi_{0}\) and the representation \(\psi_{1}(s) = s\). Both are
one-dimensional, so all the PDOs are only \(1 \times 1\) matrices:

\begin{tabular}{>{\columncolor{tbl_highlight}\centering\arraybackslash}m{1.45cm}>{\centering\arraybackslash}m{3cm}>{\centering\arraybackslash}m{3cm}}
  \rowcolor{tbl_highlight}
  \diagbox{Out}{In} & \(\psi_{0}\) & \(\psi_{1}\) \\
  \tablevspace\(\psi_{0}\) &  \tablevspace\(\mat{\T_{\mu}}\) & \tablevspace \(\mat{\U_{\mu}}\) \\
  \tablevspace\(\psi_{1}\) &  \tablevspace\(\mat{\U_{\mu}}\) & \tablevspace \(\mat{\T_{\mu}}\) \\
\end{tabular}

To get the full basis, \(\mu\) needs to take on all natural numbers (in practice,
we use all \(\mu\) up to some maximum value). Note that we assume that the
reflection is with respect to the \(x_{1}\)-axis, both here and later for
\(D_{N}\). To get other reflection axes, the PDOs simply need to be rotated.

\subsection{Cyclic group \(C_{N}\)}
The irreducible representations of \(C_{N}\) are the same as those of \(\SO(2)\)
but only up to a frequency of \(k = \lfloor \frac{N - 1}{2} \rfloor\). If \(N\) is even, there is an
additional one-dimensional irreducible representation, namely
\begin{equation}
  \psi_{N/2}(\theta) = \cos\left(\frac{N}{2}\theta\right)\,.
\end{equation}
We then get the following solutions, where \(t\) ranges over \(\Z\) and
\(\hat{t}\) over \(\N\):

\scalebox{0.75}{
\begin{tabular}{>{\columncolor{tbl_highlight}\centering\arraybackslash}m{2cm}>{\centering\arraybackslash}m{2.5cm}>{\centering\arraybackslash}m{3cm}>{\centering\arraybackslash}m{9cm}}
  \rowcolor{tbl_highlight}
  \diagbox{Out}{In} & \(\psi_{0}\) & \(\psi_{N/2}\) (if \(N\) even) & \(\psi_{n}, \quad 1 \leq n < N/2\) \\
  \(\psi_{0}\) &
                          {\[\begin{split}& \mat{\T_{\hat{t}N}},\\ & \mat{\U_{\hat{t}N}}\end{split}\]}
            &  {\[\begin{split}& \mat{\T_{(\hat{t} + 1/2)N}}, \\ & \mat{\U_{(\hat{t} + 1/2)N}}\end{split}\]}
            &  {\[\begin{split}& \mat{-\U_{n +tN} & \T_{n + tN}}, \\ & \mat{\T_{n + tN} & \U_{n + tN}}\end{split}\]}\\
  \(\psi_{1}\) (\(N\) even) &   {\[\begin{split}& \mat{\T_{(\hat{t} + 1/2)N}}, \\ & \mat{\U_{(\hat{t} + 1/2)N}}\end{split}\]}
                          &  {\[\begin{split}& \mat{\T_{\hat{t}N}}, \\ & \mat{\U_{\hat{t}N}}\end{split}\]}
                          &  {\[\begin{split}& \mat{-\U_{n + (t + 1/2)N} & \T_{n + (t + 1/2)N}}, \\ & \mat{\T_{n + (t + 1/2)N} & \U_{n + (t + 1/2)N}}\end{split}\]}\\
  \(\psi_{n}\)\newline\(1 \leq n < N/2\) &  {\[\begin{split}& \mat{-\U_{m +tN} \\ \phantom{-}\T_{m + tN}}, \\ & \mat{\T_{m + tN} \\ \U_{m + tN}}\end{split}\]}
                    &  {\[\begin{split}& \mat{-\U_{m + (t + 1/2)N} \\ \phantom{-}\T_{m + (t + 1/2)N}}, \\ & \mat{\T_{m + (t + 1/2)N} \\ \U_{m + (t + 1/2)N}}\end{split}\]}
                    &  {\(\begin{aligned}
                        \mat{\T_{m - n + tN} & -\U_{m - n + tN} \\ \U_{m - n + tN} & \phantom{-}\T_{m - n + tN}},
                      &\mat{-\U_{m - n + tN} & -\T_{m - n + tN} \\ \phantom{-}\T_{m - n + tN} & -\U_{m - n + tN}},\\
                      \mat{\T_{m + n + tN} & \phantom{-}\U_{m + n + tN} \\ \U_{m + n + tN} & -\T_{m + n + tN}},
                      &\mat{-\U_{m + n + tN} & \phantom{-}\T_{m + n + tN} \\ \phantom{-}\T_{m + n + tN} & \phantom{-}\U_{m + n + tN}}
                      \end{aligned}\)}
\end{tabular}}
\subsection{Dihedral group \(D_{N}\)}
Similarly to \(C_{N}\), the irreducible representations of \(D_{N}\) are the same as those of \(\orth(2)\)
up to a frequency of \(k = \lfloor \frac{N - 1}{2} \rfloor\). If \(N\) is even, there are two
additional one-dimensional irreducible representations, namely
\begin{equation}
  \begin{split}
    \psi_{0, N/2}(\theta, s) &= \cos\left(\frac{N}{2}\theta\right)\,,\\
    \psi_{1, N/2}(\theta, s) &= s\cos\left(\frac{N}{2}\theta\right)\,.
  \end{split}
\end{equation}
The solutions are (again with \(t \in \Z\) and \(\hat{t} \in \N\)):

\scalebox{0.70}{
\begin{tabular}{>{\columncolor{tbl_highlight}\centering\arraybackslash}m{1.45cm}>{\centering\arraybackslash}m{2cm}>{\centering\arraybackslash}m{2cm}>{\centering\arraybackslash}m{3cm}>{\centering\arraybackslash}m{3cm}>{\centering\arraybackslash}m{6cm}}
  \rowcolor{tbl_highlight}
  \diagbox{Out}{In} & \(\psi_{0, 0}\) & \(\psi_{1, 0}\) & \(\psi_{0, N/2}\) (\(N\) even ) & \(\psi_{1, N/2}\) (\(N\) even) & \(\psi_{1, n}, \quad 1 \leq n < N/2\) \\
  \tablevspace\(\psi_{0, 0}\) & \tablevspace\(\mat{\T_{\hat{t}N}}\)
                          & \tablevspace \(\mat{\U_{\hat{t}N}}\)
                          & \tablevspace \(\mat{\T_{(\hat{t} + 1/2)N}}\)
                          & \tablevspace \(\mat{\U_{(\hat{t} + 1/2)N}}\)
                          & \tablevspace \(\mat{-\U_{n + tN} & \T_{n + tN}}\)\\
  \tablevspace\(\psi_{1, 0}\) & \tablevspace\(\mat{\U_{\hat{t}N}}\)
                          & \tablevspace \(\mat{\T_{\hat{t}N}}\)
                          & \tablevspace \(\mat{\U_{(\hat{t} + 1/2)N}}\)
                          & \tablevspace \(\mat{\T_{(\hat{t} + 1/2)N}}\)
                          & \tablevspace \(\mat{\T_{n + tN} & \U_{n + tN}}\)\\
  \tablevspace \(\psi_{0, N/2}\), (\(N\) even)
                          & \tablevspace \(\mat{\T_{(\hat{t} + 1/2)N}}\)
                          & \tablevspace \(\mat{\U_{(\hat{t} + 1/2)N}}\)
                          & \tablevspace\(\mat{\T_{\hat{t}N}}\)
                          & \tablevspace \(\mat{\U_{\hat{t}N}}\)
                          & \tablevspace \(\mat{-\U_{n + (t + 1/2)N} & \T_{n + (t + 1/2)N}}\)\\
  \tablevspace \(\psi_{1, N/2}\), (\(N\) even)
                          & \tablevspace \(\mat{\U_{(\hat{t} + 1/2)N}}\)
                          & \tablevspace \(\mat{\T_{(\hat{t} + 1/2)N}}\)
                          & \tablevspace \(\mat{\U_{\hat{t}N}}\)
                          & \tablevspace\(\mat{\T_{\hat{t}N}}\)
                          & \tablevspace \(\mat{\T_{n + (t + 1/2)N} & \U_{n + (t + 1/2)N}}\)\\
  \(\psi_{1, m}\), \newline \(1 \leq m \leq N/2\)
                          & \tablevspace \(\mat{-\U_{m + tN} \\ \phantom{-}\T_{m + tN}}\)
                          & \tablevspace \(\mat{\T_{m + tN} \\ \U_{m + tN}}\)
                          & \tablevspace \(\mat{-\U_{m + (t + 1/2)N} \\ \phantom{-}\T_{m + (t + 1/2)N}}\)
                          & \tablevspace \(\mat{\T_{m + (t + 1/2)N} \\ \U_{m + (t + 1/2)N}}\)
                          & {\[\begin{split}& \mat{\T_{m - n +tN} & -\U_{m - n + tN} \\ \U_{m - n + tN} & \phantom{-}\T_{m - n + tN}}, \\
                            & \mat{\T_{m + n + tN} & \phantom{-}\U_{m + n + tN} \\ \U_{m + n + tN} & -\T_{m + n + tN}}\end{split}\]}
\end{tabular}}

\section{Discretization methods for PDOs}\label{sec:discretization}
\subsection{Finite differences}
Finite difference methods are common in machine learning; for example, the
discretization of \(\frac{d}{dx}\) as \(\begin{bmatrix}-1 & 1\end{bmatrix}\) or
\(\begin{bmatrix}-1 & 0 & 1\end{bmatrix}\), or of \(\frac{d^{2}}{dx^{2}}\) as
\(\begin{bmatrix}1 & -2 & 1\end{bmatrix}\) all use finite difference methods. To
understand where these filters come from, we need the following well-known
result:
\begin{proposition}\label{thm:finite_difference_coefficients}
  Let \(x_{1}, \ldots, x_{N} \in \R\) be arbitrary but distinct grid points.
  Then for \(m \leq N - 1\), there are unique coefficients \(w_{n}^{(m)}\) such that
  the approximation
  \begin{equation}\label{eq:finite_differences}
    f^{(m)}(0) \approx \frac{1}{h^{m}}\sum_{n = 1}^{N} w_{n}^{(m)} f(hx_{n})
  \end{equation}
  has an error \(\mc{O}(h^{N - m})\) for any \(f \in C^{N}(\R)\).
\end{proposition}
The coefficients \(w_{n}^{(m)}\) are called \emph{finite difference
  coefficients} and approximating derivatives using \cref{eq:finite_differences}
is the finite difference method. We will soon describe how to generalize this to
higher dimensions as well.

We remark that \(\mc{O}(h^{N - m})\) is an asymptotic \emph{upper bound} on the
error, and it can sometimes be lower, even for all \(f\). For example, the
central difference discretization of \(\frac{d^{2}}{dx^{2}}\) as
\(\begin{bmatrix}1 & -2 & 1\end{bmatrix}\) uses \(N = 3\) grid points but still
achieves an error of \(\mc{O}(h^{2})\), rather than \(\mc{O}(h)\).
For details on when such a \enquote{boosted} order of accuracy occurs, see~\citep{sadiq2014}.

Note that \cref{eq:finite_differences} can be generalized to
\begin{equation}
  f^{(m)}(x) \approx \frac{1}{h^{m}}\sum_{n = 1}^{N} w_{n}^{(m)} f(hx_{n} + x)\,.
\end{equation}
This follow immediately because the coefficients \(w_{n}^{(m)}\) don't depend on
the function \(f\), so we can apply \cref{thm:finite_difference_coefficients} to
\(\tau_{-x}f\).

Particularly interesting for us is the case of a regular grid. We can use
infinitely many grid points \(x_{n} \in \Z\) as long as we demand that
\(w_{n}^{(m)}\) is zero for almost all \(n\). Then we get
\begin{equation}
  f^{(m)}(x) \approx \frac{1}{h^{m}}\sum_{n}w_{n}^{(m)}f(hn + x)\,.
\end{equation}
Fixing \(h = 1\), this is exactly the cross correlation
\begin{equation}
  f^{(m)} \approx w^{(m)} \star f
\end{equation}
if we interpret \(w^{(m)}\) as a function \(n \mapsto w_{n}^{(m)}\). This is why, in
the end, we discretize a derivative by convolving with some stencil, such as
\(\begin{bmatrix}1 & -2 & 1\end{bmatrix}\), at least on a regular 1D grid.

Generalizing \cref{thm:finite_difference_coefficients} to higher dimensions does
not work in a straightforward way. However, if we restrict ourselves to regular
grids, then finite difference methods can be easily applied to PDOs. The idea is
very simple: a PDO such as \(\partial_{x}\partial_{y}^{2}\) can be interpreted as first
applying \(\partial_{y}^{2}\) and then \(\partial_{x}\) (or the other way around). So we
discretize each of these with the one-dimensional finite difference method
described before, and then we convolve with both filters one after the
other.\footnote{As we have seen, finite difference methods can most immediately
  be seen as performing a cross-correlation rather than a convolution. However,
  we can easily switch to convolutions by flipping the filter.} We
can also combine the two one-dimensional filter into one two-dimensional filter,
the outer product of the two.\footnote{For a simple PDO such as
  \(\partial_{x}\partial_{y}^{2}\), this may be undesirable for computational reasons. But in
  practice, we have PDOs that are sums of such pure terms and thus don't
  factorize.} The asymptotic error of this discretization will simply be the
highest asymptotic error along all the dimensions, so we get similar guarantees.

\subsection{RBF-FD}
As mentioned, \cref{thm:finite_difference_coefficients} does not directly
generalize to higher dimensions. So to discretize a PDO on arbitrary
point clouds in higher dimensions, a somewhat different approach is needed.

RBF-FD is one such method and works as follows: we still want to approximate a
derivative using
\begin{equation}
  \partial^{\alpha}f(0) \approx \sum_{n = 1}^{N} w_{n}^{\alpha}f(x_{n})\,,
\end{equation}
similar to finite difference methods. Here, \(x_{n} \in \R^{d}\) are arbitrary (but
again distinct) points. The idea is now that we require this approximation to be
exact if \(f(x) = \phi(\norm{x - x_{n}})\), where \(\phi\) is an arbitrary but fixed
radial basis function. In words, the approximation should become exact for a
certain radial basis function centered on any of the points \(x_{n}\). This
leads to a linear system, which is solved for the coefficients
\(w_{n}^{\alpha}\).\footnote{In practice, one often solves an extended linear system
  containing additional low-order polynomials, but we won't discuss that here.}
For more details on both finite differences and RBF-FD, see for example \citet{fornberg2017}.

\subsection{Gaussian derivatives}
Discretizing PDOs using derivatives of Gaussians is very simple to describe:
given grid points \(x_{n} \in \R^{d}\), we approximate using
\begin{equation}
  \partial^{\alpha}f(0) \approx \sum_{n = 1}^{N} \left(\partial^{\alpha}G(x_{n}; \sigma)\right)f(x_{n})
\end{equation}
where \(G(x; \sigma)\) is a Gaussian kernel with standard deviation \(\sigma\) centered
around 0. \(\sigma\) is a free parameter; larger \(\sigma\) will lead to a stronger
denoising effect.

On regular grids, this again turns into a cross-correlation, with the filter
being the derivative \(\partial^{\alpha}G(x_{n}; \sigma)\) evaluated on the grid coordinates.

In \cref{sec:distribution_equivariance} we briefly touch on a possible
interpretation of this discretization method using the distributional framework
for PDOs.

\section{Relation to PDO-eConvs}\label{sec:pdo_econvs}
In this section, we describe how PDO-eConvs~\citep{shen2020} fit into the
framework of steerable PDOs. We mostly follow the original notation from
\citet{shen2020} when describing PDO-eConvs but do make some minor changes to
avoid clashes and confusion with our own notation.

As in our presentation, \citet{shen2020} use polynomials to describe PDOs. One
difference is that they never explicitly use \emph{matrices} of polynomials,
because they model the feature space somewhat differently (which we will discuss
in a moment). They write \(H\) for the polynomial describing a PDO (where we
would write e.g.\ \(p\)) and write \(\chi^{(A)}\) for the corresponding PDO
transformed by \(A \in \orth(d)\). In our notation,
\begin{equation}
  \chi^{(A)} := D(A \cdot H) = D(H \circ A^{-1})\,.
\end{equation}

PDO-eConvs use two types of PDO layers. The first one, \(\Psi\), can be
interpreted as a steerable PDO with \(\rho_{\text{in}}\) trivial and
\(\rho_{\text{out}}\) regular. It maps the scalar input to the network to the
internally used regular representation. The second layer type, \(\Phi\) maps
between regular representations and is used for hidden layers. At the end,
pooling is performed to obtain a scalar output again.

The first layer type is defined as
\begin{equation}
  \Psi : C^{\infty}(\R^{d}) \to C^{\infty}(\tilde{E}(d)), \qquad \Psi(f)(x, A) := (\chi^{(A)}f)(x)\,.
\end{equation}
Here, \(\tilde{E}(d) := \R^{d} \rtimes S\) with \(S \leq \orth(d)\); in practice, \(S\)
needs to be a finite subgroup, i.e.\ \(C_{N}\) or \(D_{N}\). Elements of
\(\tilde{E}(d)\) can be uniquely written as \((x, A)\) with \(x \in \R^{d}\) and
\(A \in S\).

There is an obvious bijection \(C^{\infty}(\tilde{E}(d)) \cong C^{\infty}(\R^{d}, \R^{c})\),
where \(c := \abs{S}\) is the order of \(S\), i.e.\ the number of group
elements. Concretely, we define
\begin{equation}
  \Theta: C^{\infty}\left(\tilde{E}(d)\right) \to C^{\infty}(\R^{d}, \R^{c}), \qquad
  f \mapsto \Big(x \mapsto \big(f(x, A_{1}), \ldots, f(x, A_{c})\big)\Big)
\end{equation}
where \(A_{1}, \ldots, A_{c}\) is an enumeration of the group elements of \(S\). We
will therefore interpret \(C^{\infty}(\tilde{E}(d))\) as a \(c\)-dimensional feature
field over \(\R^{d}\), and we will show that using regular representations for
this field (and trivial ones for the input) makes the PDO-eConv layers
equivariant and thus steerable PDOs.

First, note that under the \(\Theta\) bijection, the first PDO-eConv layer type \(\Psi\)
becomes a \(c \times 1\) matrix of PDOs, namely
\begin{equation}
  D(H_{\Psi}) := D\left(\mat{A_{1} \cdot H \\ \vdots \\ A_{c} \cdot H}\right) = \mat{\chi^{(A_{1})} \\ \vdots \\ \chi^{(A_{c})}}\,.
\end{equation}
What we mean by this is that the diagram
\begin{center}
\begin{tikzcd}
C^\infty(\mathbb{R}^d) \arrow[d, equal] \arrow[rr, "\Psi"] &  & C^\infty(\tilde{E}(d)) \arrow[d, "\Theta"] \\
C^\infty(\mathbb{R}^d) \arrow[rr, "D(H_\Psi)"]                              &  & {C^\infty(\mathbb{R}^d, \mathbb{R}^c)}
\end{tikzcd}
\end{center}
commutes. Concretely, we have
\begin{equation}
  \Psi(f)(x, A_{i}) = (\chi^{(A_{i})}f)(x) = (D(H_{\Psi})f)(x)_{i}\,.
\end{equation}
So we need to check whether \(H_{\Psi}\) satisfies the PDO steerability constraint
for trivial to regular PDOs:
\begin{equation}\label{eq:psi_equivariance_constraint}
  H_{\Psi}(Ax) = \rho_{\text{regular}}(A)H_{\Psi}(x)\,.
\end{equation}
Using the definition of \(\rho_{\text{regular}}(A)\), we can rewrite the RHS as
\begin{equation}\label{eq:psi_rhs}
  \begin{split}
    \rho_{\text{regular}}(A)H_{\Psi}(x) &= \sum_{k = 1}^{c} \rho_{\text{regular}}(A) H_{\psi}(x)_{k}e_{A_{k}} \\
    &= \sum_{k = 1}^{c} (A_{k} \cdot H)(x) e_{AA_{k}}\\
    &= \sum_{k = 1}^{c} H(A_{k}^{-1}x) e_{AA_{k}}\\
    &= \sum_{l = 1}^{c} H(A_{l}^{-1}Ax) e_{A_{l}}\,.
  \end{split}
\end{equation}
Here, we use basis vectors \(e_{A_{1}}, \ldots, e_{A_{c}}\) for \(\R^{c}\), with the
same enumeration \(A_{1}, \ldots, A_{c}\) of \(S\) used to define \(\Theta\).
For the final step, we reparameterized the sum with \(A_{l} := AA_{k}\).

The LHS of \cref{eq:psi_equivariance_constraint} can be written as
\begin{equation}
  \begin{split}
    H_{\Psi}(Ax) &= \mat{A_{1} \cdot H(Ax), & \ldots, & A_{c} \cdot H(Ax)}^{T} \\
                 &= \mat{H(A_{1}^{-1}Ax), & \ldots, & H(A_{c}^{-1}Ax)}^{T}\,.
  \end{split}
\end{equation}
This is the same as the final term in \cref{eq:psi_rhs}, which proves that
the PDO steerability constraint is satisfied.

The second PDO-eConv layer type, mapping between regular representations, is
defined as
\begin{equation}
  \Phi: C^{\infty}(\tilde{E}(d)) \to C^{\infty}(\tilde{E}(d)), \qquad \Phi(e)(x, A) := \sum_{j = 1}^{c} \left(\chi_{A_{j}}^{(A)}e\right)(x, AA_{j})\,.
\end{equation}
\(\chi_{A_{j}}^{(A)}\) are \(c\) different PDOs and they act on
\(e \in C^{\infty}(\tilde{E}(d))\) by acting on each of the \(c\) components
separately. Under the \(\Theta\) bijection, \(\Phi\) becomes
\begin{equation}
  \begin{split}
    \Theta(\Phi(e))(x)_{i} &= \sum_{j = 1}^{c} \left(\chi_{A_{j}}^{(A_{i})}e\right)(x, A_{i}A_{j})\\
    &= \sum_{j = 1}^{c}\left(\chi_{A_{j}}^{(A_{i})}\Theta(e)_{A_{i}A_{j}}\right)(x) \\
    &= \sum_{j = 1}^{c}\left(\chi_{A_{i}^{-1}A_{j}}^{(A_{i})}\Theta(e)_{A_{j}}\right)(x)\,.
  \end{split}
\end{equation}
We can therefore represent it as a \(c \times c\) PDO \(D(H_{\Phi})\) with
\begin{equation}
  \left(H_{\Phi}\right)_{ij}(x) = H_{A_{i}^{-1}A_{j}}(A_{i}^{-1}x) =: H_{ij}(x)\,,
\end{equation}
where \(H_{A_{i}^{-1}A_{j}}\) is the polynomial that induces \(\chi_{A_{i}^{-1}A_{j}}^{(I)}\).
This makes the diagram
\begin{center}
\begin{tikzcd}
C^\infty(\tilde{E}(d)) \arrow[rr, "\Phi"] \arrow[d, "\Theta"]  &  & C^\infty(\tilde{E}(d)) \arrow[d, "\Theta"] \\
{C^\infty(\mathbb{R}^d, \mathbb{R}^c)} \arrow[rr, "D(H_\Phi)"] &  & {C^\infty(\mathbb{R}^d, \mathbb{R}^c)}
\end{tikzcd}
\end{center}
commute, similar to the case discussed above. So again, we need to check that
\(H_{\Phi}\) satisfies the PDO steerability constraint, this time for
\(\rho_{\text{in}}\) and \(\rho_{\text{out}}\) both regular:
\begin{equation}
  H_{\Phi}(Ax) = \rho_{\text{regular}}(A)H_{\Phi}(x)\rho_{\text{regular}}(A^{-1})\,.
\end{equation}
Writing out \(H_{\Psi}\) in its components, this becomes
\begin{equation}
  \begin{split}
    H_{\Psi}(Ax) &\overset{!}{=} \sum_{i, j}\rho_{\text{regular}}(A)H_{ij}(x)e_{A_{i}}e_{A_{j}}^{T}\rho_{\text{regular}}(A)^{-1} \\
    &\overset{(1)}{=} \sum_{i, j}H_{ij}(x)e_{AA_{i}}e_{AA_{j}}^{T} \\
    &= \sum_{i, j}H_{A_{i}^{-1}A_{j}}(A_{i}^{-1}x)e_{AA_{i}}e_{AA_{j}}^{T} \\
    &\overset{(2)}{=} \sum_{i, j}H_{A_{i}^{-1}AA^{-1}A_{j}}(A_{i}^{-1}Ax)e_{A_{i}}e_{A_{j}}^{T} \\
    &= \sum_{i, j}H_{A_{i}^{-1}A_{j}}(A_{i}^{-1}Ax)e_{A_{i}}e_{A_{j}}^{T} \\
    &= \sum_{i, j}H_{ij}(Ax)e_{A_{i}}e_{A_{j}}^{T}\,.
  \end{split}
\end{equation}
The first and the last term are the same, just written out in components on the
RHS, so the steerability constraint is again satisfied.
For (1), we used that \(\rho_{\text{regular}}(A)\) is orthogonal, and thus
\begin{equation}
  e_{j}^{T}\rho_{\text{regular}}(A)^{-1} = e_{j}^{T}\rho_{\text{regular}}(A)^{T} =
  \left(\rho_{\text{regular}}(A)e_{j}\right)^{T}\,.
\end{equation}
(2) was again a reparameterization of the sum, with
\(AA_{i} \mapsto A_{i}\) and \(AA_{j} \mapsto A_{j}\). The other steps are only
simplifications and plugging in definitions.

In conclusion, we have shown that there is a simple bijection between the feature
spaces used for PDO-eConv hidden layers and the feature fields we use, and that
under this bijection, PDO-eConvs correspond to steerable PDOs with regular
representations (and trivial representations for the input).

It is relatively easy to adapt the argument we present for the converse direction:
every equivariant PDO between two regular feature fields (or from a scalar to a
regular one) can be interpreted as a PDO-eConv layer.

\section{Additional experimental results}\label{sec:more_experiments}

\subsection{Equivariance errors}
To check the equivariance error---and indirectly the discretization error, since at least equivariant
layers have zero equivariance error in the continuous setting---, we checked how much rotating
an input image changes the output of a layer, compared to what the output should be under perfect
equivariance. The challenge here is that rotating a discrete image itself introduces some errors.
To minimize those, we used a large high-dimensional image, rotated it, and then scaled it down before
passing it into the layer, and scaled down again after that. We compared the result of this procedure
to what we get by first downscaling, then applying a convolutional or PDO layer, then rotating, and
then downscaling again. Effectively, all rotations thus happen at large resolutions, which should
minimize artifacts.

\Cref{tab:equivariance_errors} shows the relative equivariance errors (as multiples of 1e-6).
These errors are for randomly initialized layers (averaged over 10 initializations).
As discussed in the main text, the asymptotic error bound
for finite difference discretization does not lead to a particularly low error in practice.

\begin{table}
  \small
  \centering
  \caption{\small Relative equivariance errors for $C_{16}$ on a test image, averaged over 10
  random initializations of the layer. As orientation, we also include non-equivariant (vanilla)
  convolutions.}\label{tab:equivariance_errors}
  \begin{tabular}{lll}
    \toprule
    Method & Stencil &  Error [1e-6] \\
    \midrule
    \multirow{2}{1.3cm}{Vanilla convolution} & \(3 \times 3\) & \(32716 \pm 5484\) \\
                                     & \(5 \times 5\) & \(32785 \pm 5620\) \\
    \midrule
    \multirow{2}{1.3cm}{Kernels} & \(3 \times 3\) & \(5.0 \pm 1.0\) \\
                                 & \(5 \times 5\) & \(5.0 \pm 1.3\) \\
    \midrule
    \multirow{2}{1.3cm}{FD}     & \(3 \times 3\) & \(4.8 \pm 1.1\) \\
                                & \(5 \times 5\) & \(6.6 \pm 1.2\) \\
    \midrule
    \multirow{2}{1.3cm}{RBF-FD} & \(3 \times 3\) & \(5.9 \pm 1.2\) \\
                                & \(5 \times 5\) & \(6.1 \pm 1.1\) \\
    \midrule
    \multirow{2}{1.3cm}{Gauss}  & \(3 \times 3\) & \(4.9 \pm 1.7\)\\
                                & \(5 \times 5\) & \(6.4 \pm 0.8\) \\
    \bottomrule
  \end{tabular}
\end{table}

\subsection{Restriction experiments}
\Cref{tab:mnist_rot_restriction} shows additional results on MNIST-rot. The
general architecture and hyperparameters are the same as in the experiments in
\cref{sec:experiments} with regular representations, using our basis. However,
in the experiments in this section, the first five layers are
\(D_{16}\)-equivariant, while the final PDO/convolutional layer is
\(C_{16}\)-equivariant. The motivation for this is that while the input images
do not have \emph{global} reflectional symmetry, such symmetry occurs on smaller
scales, so that stronger equivariance in earlier layers might be helpful.

However, we don't observe clear improvements over pure \(C_{16}\)-equivariance.
A reason could be that even the \(C_{16}\)-equivariant networks are already very
parameter efficient compared to classical CNNs, so that parameter efficiency and
equivariance are not a bottleneck anymore. It is also possible that the
architecture would need to be adapted slightly to profit from the
\(D_{16}\)-equivariant layers.
\begin{table}
  \small
  \centering
  \caption{\small MNIST-rot results with restriction from \(D_{16}\) to
    \(C_{16}\) equivariance. Test errors \(\pm\) standard deviations are
    averaged over six runs. See main text for details on the
    models.}\label{tab:mnist_rot_restriction}
  \begin{tabular}{llll}
    \toprule
    Method & Stencil &  Error [\%] & Params \\
    \midrule
    \multirow{2}{1.3cm}{Kernels} & \(3 \times 3\) & \(0.717 \pm 0.022\) & 709K \\
           & \(5 \times 5\) & \(0.710 \pm 0.020\) & 1.1M \\
    \cmidrule{1-4}
    \multirow{2}{1.3cm}{FD}      & \(3 \times 3\) & \(1.248 \pm 0.060\) & 709K \\
           & \(5 \times 5\) & \(1.436 \pm 0.063\) & 947K \\
    \cmidrule{1-4}
    \multirow{2}{1.3cm}{RBF-FD}  & \(3 \times 3\) & \(1.396 \pm 0.059\) & 709K \\
           & \(5 \times 5\) & \(1.565 \pm 0.048\) & 947K \\
    \cmidrule{1-4}
    \multirow{2}{1.3cm}{Gauss}   & \(3 \times 3\) & \(0.806 \pm 0.047\) & 709K \\
           & \(5 \times 5\) & \(0.778 \pm 0.051\) & 947K \\
    \bottomrule
  \end{tabular}
\end{table}
\subsection{Stencil images}
\Cref{fig:stencils} contains examples of stencils used during the MNIST-rot experiments.
For the \(3 \times 3\) stencil, the different methods yield qualitatively similar results (though FD and
RBF-FD have fewer stencils that make use of the four corners). But for \(5 \times 5\)
stencils, kernels and Gauss discretization make significantly more use of the outer
stencil points than FD and RBF-FD.

\begin{figure}
  \centering
  \begin{subfigure}{.7\linewidth}
    \includegraphics[width=\linewidth]{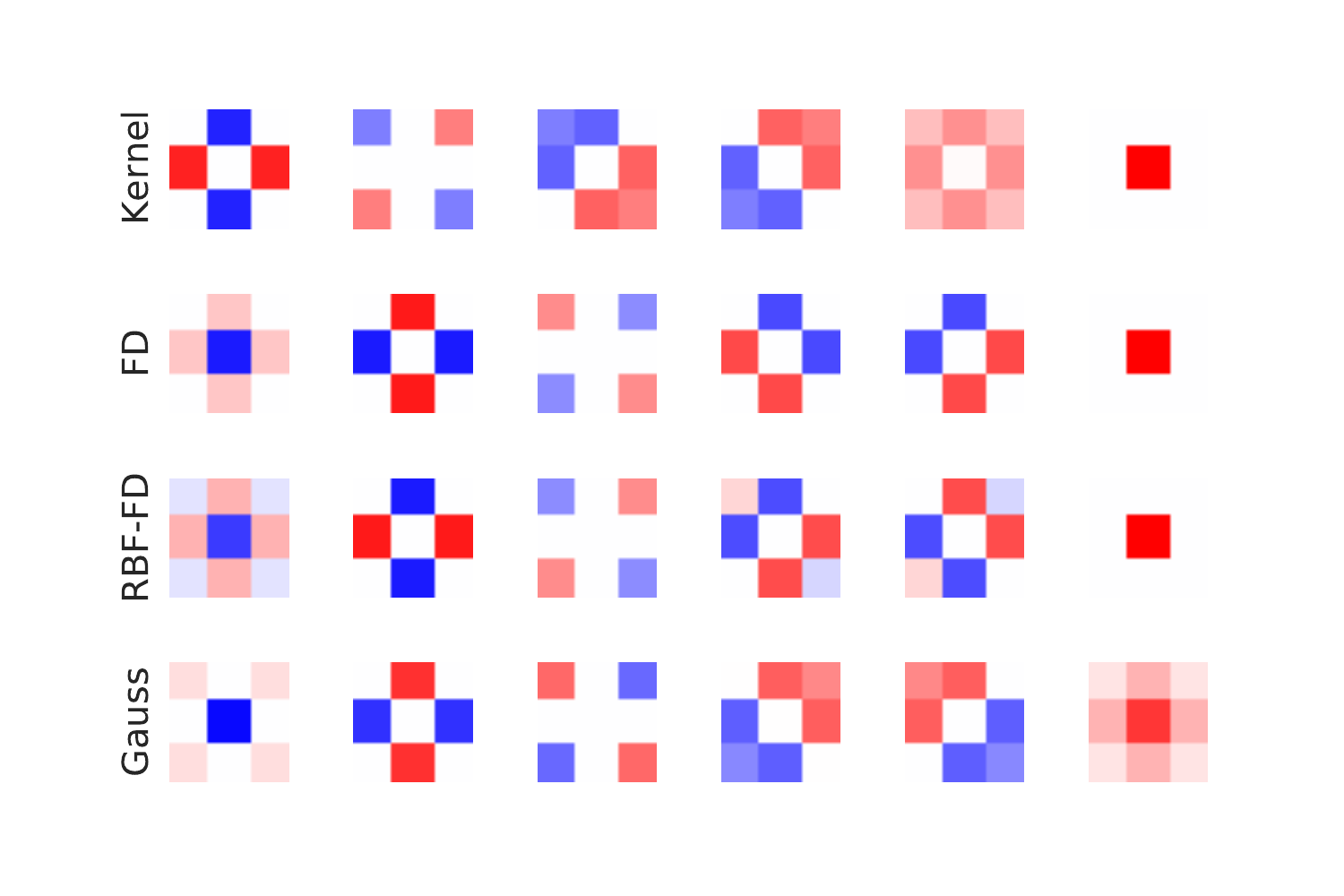}
    \caption{\(3 \times 3\) stencils}
  \end{subfigure}
  \begin{subfigure}{.7\linewidth}
    \includegraphics[width=\linewidth]{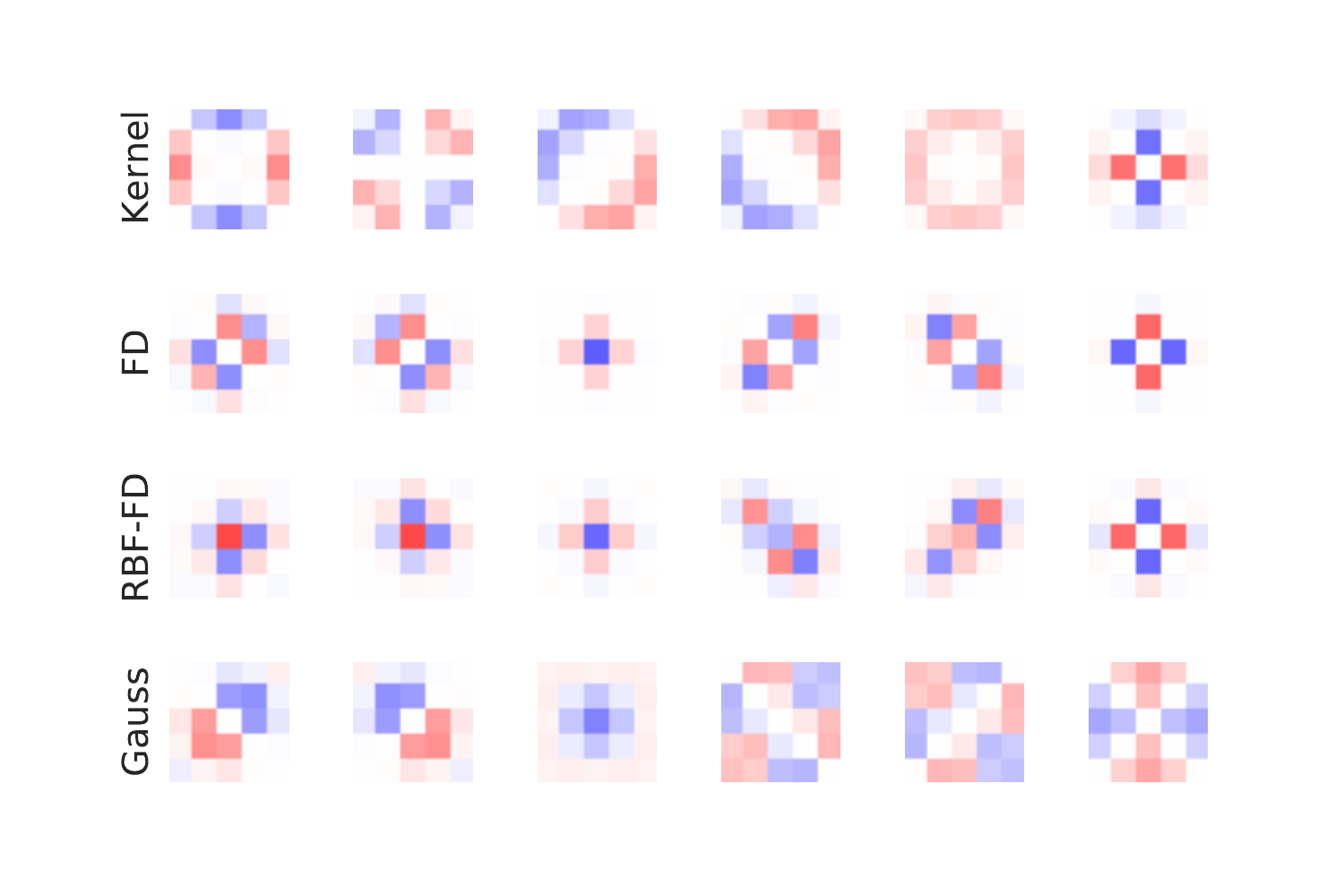}
    \caption{\(5 \times 5\) stencils}
  \end{subfigure}
  \caption{Stencils of basis filters for a trivial to regular layer. Each row is
    a different method and contains six arbitrarily selected filters from the
    basis (which is not the entire basis). All the settings are those that were
    actually used for the MNIST-rot experiments.}\label{fig:stencils}
\end{figure}

\newpage
\section{Details on experiments}\label{sec:experiment_details}
\subsection{MNIST-rot experiments}
For the MNIST-rot experiments, we use an architecture similar to one from~\citep{weiler2019}.
\Cref{tab:mnist_rot_architecture} contains a listing of all the layers. Each
conv block consists of a \(3 \times 3\) or \(5 \times 5\) steerable layer, either
convolutional or a PDO, followed by batch-normalization and an ELU nonlinearity.
The output fields are the number of \(C_{16}\)-regular feature fields that is
used; in the case of the Vanilla CNN and for quotient representations, the
number of fields is adjusted so that the parameter count is approximately preserved.
\begin{table}
  \centering
  \caption{Architecture for MNIST-rot experiments}\label{tab:mnist_rot_architecture}
  \begin{tabular}{lr}
    \toprule
    Layer & Output fields \\
    \midrule
    Conv block & 16 \\
    Conv block & 24 \\
    Max pooling & \\
    Conv block & 32 \\
    Conv block & 32 \\
    Max pooling & \\
    Conv block & 48 \\
    Conv block & 64 \\
    Group pooling & \\
    Global average pooling & \\
    Fully connected & 64 \\
    Fully connected + Softmax & 10 \\
    \bottomrule
  \end{tabular}
\end{table}

For the quotient experiments, we use
\(5\rho_{\text{regular}} \oplus 2\rho_{\text{quot}}^{C_{16}/C_{2}} \oplus 2 \rho_{\text{quot}}^{C_{16}/C_{4}} \oplus 4 \rho_{\text{trivial}}\)
as the representation, where the numbers are scaled to reach the same parameter
count as the model with only regular representations. This combination of
representations is the same one used by \citet{weiler2019} and we refer to their
appendix for motivation on why we need to combine different representations when
using quotients.

We trained all MNIST-rot models for 30 epochs with Adam~\citep{kingma2015} and a
batch size of 64. The training data was normalized and augmented using random
rotations. For the final training runs, we used the entire set of 12k training
plus validation images, as is common practice on MNIST-rot. The initial learning
rate was 0.05, which was decayed exponentially after a burn-in of 5 epochs at a
rate of 0.7 per epoch. We used a dropout of 0.5 after the fully connected layer,
and a weight decay of 1e-7.

These hyperparameters are based on those used in~\citep{weiler2019}; the main
difference is that we use another learning rate schedule, which works better
than the original one for all models.

For the Gaussian discretization models, we use a standard deviation of \(\sigma = 1\)
for \(3 \times 3\) stencils and \(\sigma = 1.3\) for \(5 \times 5\) stencils; we chose these
values by visual inspection of the stencils, with the aim that full use is made
of the stencil (see \cref{sec:more_experiments}). The RBF-FD discretization uses
third-order polyharmonic basis functions, i.e.\ \(\phi(r) = r^{3}\).

\subsection{STL-10 experiments}
For the STL-10 experiments, we used exactly the same architecture and
hyperparameters as \citet{weiler2019}, which in turn are essentially those of
\citet{devries2017}. This means we train a Wide-ResNet-16-8~\citep{zagoruyko2017} for 1000 epochs,
with SGD and Nesterov momentum of 0.9, a batch size of 128 and weight decay of
5e-4. We begin with a learning rate of 0.1 and divide it by 5 after 300, 400,
600 and 800 epochs. For data augmentation, we pad the image by 12 pixels, then
randomly crop a \(96 \times 96\) pixel patch, randomly flip horizontally, and apply
Cutout~\citep{devries2017} with a cutout size of \(60 \times 60\) pixels.

\subsection{Fluid flow prediction}
In the original dataset~\citep{ribeiro2020},
the fluid always flows into a tube from the same direction. To better demonstrate
the effects of equivariance, we randomly rotate each sample by a multiple of $\frac{\pi}{2}$,
making the task somewhat more challenging. During training, the rotations are different
in each epoch to avoid disadvantaging the non-equivariant networks (i.e. we effectively
use data augmentation). To ensure all inputs have the same shape,
we pad the original samples with zeros to make them square before rotating them.

Our architecture and hyperparameters follow those of \citet{ribeiro2020}. The only exception
is that we use a single decoder for the vector field (since we treat it as a single vector-valued output
for the purposes of equivariance), whereas \citet{ribeiro2020} used separate decoders for the
two vector components in some of their experiments. Just as in the STL-10 experiments,
we perform no hyperparameter tuning and use the hyperparameters that \citet{ribeiro2020}
optimized for the non-equivariant network. To make the network equivariant, we replace the usual
convolutions with $C_8$-equivariant steerable kernels or PDOs. We chose the channel sizes such
that all networks had approximately the same number of parameters (slightly over 800k). 

\subsection{Computational requirements}
We performed our experiments on an internal cluster with a GeForce RTX 2080 Ti
and 6 CPU cores. A single run of an MNIST-rot model took about 12 minutes and a
run of the STL-10 model about 5.5h. Training the fluid flow prediction model took
about 45 minutes. Multiplying this by the number of
experiments we did and by six runs with different seeds, the MNIST-rot
results took about 34 hours to produce, the STL-10 results about 264 hours,
and the fluid flow results about 18 hours.
The initial tests we did to debug and find a good learning rate schedule (on MNIST-rot)
took much less time than that. So we estimate that producing this paper took around
350 GPU-hours on the GeForce RTX 2080 Ti.

\end{document}